\documentclass{article}

\usepackage{microtype}
\usepackage{graphicx}
\usepackage{subcaption}
\usepackage{booktabs} 

\usepackage{hyperref}

\usepackage[preprint]{icml2026}

\usepackage{amsmath}
\usepackage{amssymb}
\usepackage{mathtools}
\usepackage{amsthm}
\usepackage{bm}
\usepackage{dsfont}

\usepackage[capitalize,noabbrev]{cleveref}

\theoremstyle{plain}
\newtheorem{theorem}{Theorem}[section]
\newtheorem{proposition}[theorem]{Proposition}
\newtheorem{lemma}[theorem]{Lemma}
\newtheorem{corollary}[theorem]{Corollary}
\theoremstyle{definition}
\newtheorem{definition}[theorem]{Definition}

\theoremstyle{remark}

\newcommand{\bbR}{\mathbb{R}}
\newcommand{\bbE}{\mathbb{E}}
\newcommand{\bbN}{\mathbb{N}}
\newcommand{\bbP}{\mathbb{P}}

\newcommand{\indic}{\mathds{1}}
\newcommand{\indep}{\perp \!\!\! \perp}

\usepackage[textsize=tiny]{todonotes}

\icmltitlerunning{Graph Attention Network for Node Regression on Random Geometric Graphs with ER contamination}

\begin{document}

\twocolumn[
  \icmltitle{Graph Attention Network for Node Regression on Random Geometric Graphs with Erd\H{o}s--R\'enyi contamination}



  \icmlsetsymbol{equal}{*}

  \begin{icmlauthorlist}
    \icmlauthor{Somak Laha}{yyy}
    \icmlauthor{Suqi Liu}{comp}
    \icmlauthor{Morgane Austern}{yyy}
  \end{icmlauthorlist}

  \icmlaffiliation{yyy}{Department of Statistics, Harvard University, \texttt{\{somaklaha, maustern\}@fas.harvard.edu}}
  \icmlaffiliation{comp}{Department of Statistics, University of California, Riverside, \texttt{suqi.liu@ucr.edu}}

  \icmlcorrespondingauthor{Somak Laha}{\texttt{somaklaha@fas.harvard.edu}}

  \icmlkeywords{Machine Learning, ICML}

  \vskip 0.3in
]



\printAffiliationsAndNotice{}  

\begin{abstract}
Graph attention networks (GATs) are widely used and often appear robust to noise in node covariates and edges, yet rigorous statistical guarantees demonstrating a provable advantage of GATs over non-attention graph neural networks~(GNNs) are scarce. We partially address this gap for node regression with graph-based errors-in-variables models under simultaneous covariate and edge corruption: responses are generated from latent node-level covariates, but only noise-perturbed versions of the latent covariates are observed; and the sample graph is a random geometric graph created from the node covariates but contaminated by independent Erd\H{o}s--R\'enyi edges. We propose and analyze a carefully designed, task-specific GAT that constructs denoised proxy features for regression. We prove that regressing the response variables on the proxies achieves lower error asymptotically in (a) estimating the regression coefficient compared to the ordinary least squares (OLS) estimator on the noisy node covariates, and (b) predicting the response for an unlabelled node compared to a vanilla graph convolutional network~(GCN)---under mild growth conditions. Our analysis leverages high-dimensional geometric tail bounds and concentration for neighbourhood counts and sample covariances. We verify our theoretical findings through experiments on synthetically generated data. We also perform experiments on real-world graphs and demonstrate the effectiveness of the attention mechanism in several node regression tasks.

\end{abstract}

\section{Introduction}

Graph attention networks (GATs)~\cite{veličković2018graphattentionnetworks} have been proposed to adapt to informative neighbours in message passing, thereby alleviating the oversmoothing phenomenon suffered by graph convolutional networks (GCNs). In practice, GAT-style architectures are a common choice in graph neural network~(GNN) modeling, and have been shown to be effective across many node-level prediction benchmarks. Yet it remains unclear when attention provably improves over non-attention message passing. Heuristically, attention should help when covariates are noisy or useful signal is unevenly distributed across neighbours, since the adaptive weighting can better separate signal from the background. In this paper, we formalize this idea with a simple node-regression task, showing that a specifically-designed discrete-attention scheme yields consistent estimation and lower prediction error than non-attention aggregation.

To this end, we study a regression model with noisy graph side information. Specifically, we observe a network with $n$ nodes, where each node $i$ has an unobserved latent covariate $\bm x_i\in\mathbb R^d$ and a scalar response
$y_i= \bm x_i^\top \bm \beta + \varepsilon_i$. We do not observe $\bm x_i$ directly; instead, we are given a noise-amplified version $\bm z_i=\bm x_i+\bm\eta_i$. Conditioned on the latent covariates $\{\bm x_i\}_{i=1}^n$, the underlying graph is a random geometric (dot-product) graph: we connect $i$ and $j$ whenever $\bm x_i^\top \bm x_j$ exceeds a threshold $t_n$, hence linked nodes have strongly aligned latent covariates. The observed graph is then formed by corrupting this geometric graph with independent Erd\H{o}s--R\'enyi~(ER) noise. That is, independently of ${\bm x_i}$ and the geometric graph, we add independent edges from an ER graph, and the observed edge set $E$ is the union of the geometric and ER edges. Both the dimension of the features and the number of observations diverge $n,d\to\infty$ with signal and noise variances that do not scale with $n$ or $d$. Our goal is node regression: using both the features $ \{\bm {z}\}^n_{i=1} $ and the noisy observed graph to construct accurate predictors of the responses and, in particular, to recover the regression coefficient $\bm {\beta} $. Feature noise creates an errors-in-variables problem (attenuation if regressing $y_i$ on $\bm z_i$), while ER edges introduce spurious neighbours that can dominate naive neighbourhood averaging.
The model and the relevant notations are formalized in Section~\ref{model overview}.

A straightforward method to estimate the coefficients $\bm \beta$ is by regressing $\bm Y_n$ directly on $\bm Z_n$. This naive approach inherits the classical attenuation bias of errors-in-variables: ordinary least squares (OLS) estimate of the regression coefficients obtained by regressing on $\bm Z_n$ is inconsistent even when $d \ll n$.
To overcome this, we design a discretized attention mechanism that computes node-level denoised proxies $\bm \lambda_i$ for the latent variables $\bm x_i$. To motivate our architecture, heuristically, geometric neighbours of node $i$ tend to have latent covariates $\bm x_j$ similar to $\bm x_i$, so message passing along geometric edges can in principle denoise the observed features $\bm z_i$.
However, because the observed graph is contaminated by many spurious ER edges, standard message passing that averages over all neighbours yields suboptimal estimation and high prediction error. We therefore implement our discretized attention mechanism as a two-layer attention-based graph neural network that constructs the denoised proxies $\bm \lambda_i$ 
by selectively averaging neighbour covariates.
We prove that regressing on these denoised proxies not only yields consistent estimation for $\bm\beta$ (Theorem~\ref{main text theorem proxy ols}),
but also achieves strictly smaller asymptotic risk compared to any finite-depth GNN that aggregates raw neighbourhoods when ER edges dominate geometric ones (Theorem~\ref{main text theorem non-att gnn prediction}).

One interesting feature is that designing the attention weights is delicate: if we were to let the weights depend directly on a single dot-product $\bm z_i^\top \bm z_j$ in the same coordinates that are later averaged, the selection would be strongly correlated with the measurement noise (and hence with the regression residuals), leading to biased estimation of $\bm \beta$. To avoid this, we split each $\bm z_i$ into two disjoint coordinate blocks and use a cross-fitting attention rule: screening uses dot-products within one block to decide which neighbours to keep, while averaging uses the other block. This decouples the selection event from the coordinates being averaged and yields low-variance, approximately unbiased proxies for the latent covariates. The resulting proxies $\bm \lambda_i$ are collected in the matrix $\bm \Lambda_n$, and we finally regress $\bm Y_n$ on $\bm \Lambda_n$. Algorithm~\ref{alg:attn-proxy-regression} details this procedure.

Our analysis shows that, under mild growth conditions relating $n,d$, the average degree of both the ER and geometric graph, as well as covariate and noise variances, the attention-based proxies $\bm \lambda_i$ are close to the latent covariates $\bm x_i$ in $L_2$ sense, and the resulting OLS estimate for the regression coefficient is consistent. The asymptotic mean squared error (MSE) for predicting the response variable of an unlabelled node achieved by our proposed method (Algorithm~\ref{alg:attn-proxy-prediction}) based on proxies $\bm \lambda_i$ approaches the observation-noise variance even when ER contamination is substantial. At the same time, any finite layer non-attention message passing network that aggregates raw neighbourhoods suffers a higher asymptotic MSE whenever the ER degree dominates the geometric degree. 

In summary, by leveraging a noisy random geometric graph setting,
we rigorously demonstrate that (1) a carefully-designed attention-based graph neural network overcomes the attenuation bias faced by OLS, obtaining consistent estimation of the coefficients for a node regression task;
(2) the GAT-style model provably outperforms all non-attention graph neural networks under the same data generative assumptions.
We further validate our theoretical results using experiments on synthetic data as well as real-world networks.

\paragraph{Notations.}
Suppose the observed graph is $G = (\bm Z_n, E)$ where $\bm Z_n$ is the matrix of node-level covariates and $E$ is the set of edges. For each $i \in [n]$, we use the notation $N_i \coloneqq \{j \in [n]: (i,j)\in E\}$ to denote the neighbourhood of $i$ in the observed graph $G$.
For any vector $\bm v \in \mathbb{R}^d$, we denote by $\bm v^{(1)}$ the top $\lceil d/2 \rceil$ coordinates of $\bm v$, and by $\bm v^{(2)}$ the bottom $\lfloor d/2 \rfloor$ coordinates of $\bm v$. In the rest of this article, we assume that $d$ is even, but the analysis for odd $d$ remains identical as $d$ grows to infinity.
The Euclidean norm of a vector $\bm v$ is defined by $\|\bm v\| = \sqrt{\bm v^\top \bm v}$ and the $L_r$ norms of a random variable $X$ are denoted by
$\|X\|_{L_r} = \left(\mathbb{E}[|X|^r]\right)^{1/r}$.
For any matrix $\bm M \in \mathbb{R}^{m\times m}$, we denote its operator norm by
$\|\bm M\|_{\mathrm{op}} = \sup\left\{ \frac{\|\bm M \bm v\|}{\|\bm v\|}: \bm v \in \mathbb{R}^m \right\}$.
For any $k\in\bbN$, we use the notation $[k] \coloneqq \{1,\ldots, k\}$.
We shall use the standard order notations $o(\cdot), O(\cdot), \Theta(\cdot)$ to mean asymptotic orders as $n$ tends to infinity.

\section{Related work}\label{related work}

\textbf{Message passing and attention.} Modern GNNs formalize representation learning on graphs as local message passing. The spectral simplification in \cite{kipf2017semisupervisedclassificationgraphconvolutional} crystallized neighbourhood averaging as a degree‑normalized linear operator and popularized two‑layer architectures for node tasks. The inductive perspective in \cite{10.5555/3294771.3294869} introduced sampling‑and‑aggregation operators (GraphSAGE) that scale to unseen nodes and graphs. A concurrent unifying view appears in \cite{10.5555/3305381.3305512}, which formalizes message passing with learnable update and aggregation maps. 

On top of these, an attention mechanism allows a node to weight its neighbours. Graph Attention Networks (GATs) \cite{veličković2018graphattentionnetworks}, together with the contemporary attention-based GNN of \cite{thekumparampil2018attention}, introduced masked self‑attention over 1‑hop neighbourhoods with softmax‑normalized coefficients and quickly became standard baselines for node prediction tasks. 

Subsequent work questioned the expressivity of the original mechanism: \cite{brody2022attentivegraphattentionnetworks} showed that the classic single-head GAT layer implements a limited, essentially “static” attention where the coefficients do not sufficiently depend on the representation of the center node, and proposed a more expressive dynamic variant, GATv2, whose attention scores can approximate any permutation‑invariant neighbourhood weighting. Overall, most of this line of work is concerned with expressivity: (i) which neighbourhood functions attention layers can represent, or (ii) with optimization phenomena.

Closer to our statistical perspective, \cite{ma2024graph} analyzed graph attention in contextual stochastic block models (CSBMs), with SBM edges and Gaussian-mixture node features. Their framework separates “structure noise” and “feature noise” and characterizes regimes where an explicit (sign-based) attention rule can increase an effective feature signal-to-noise ratio, improving classification when structural noise dominates. However, their strongest guarantees are proved in the “easy” high signal-to-noise ratio scaling: for single-layer perfect node classification, they assume an SNR scaling $\mathrm{SNR}=\omega(\sqrt{\log n})$, which under their parametrization implies that the feature-noise parameter tends to zero as $n$ increases. They also emphasize that the attention mechanism analyzed is a simplified, non-learnable version chosen for tractability. In contrast, we work with latent-position graphs and additive covariate measurement error at fixed, non-vanishing variance, and we study linear regression under errors-in-variables rather than community detection. Because our attention mechanism uses noisy covariates both to screen neighbours and to form averaged proxies, we adopt a cross-part screening/averaging design that decouples these steps. This enables us to prove a GAT-style advantage over GCN-style averaging even when covariate noise is constant, and under a broader stochastic notion of structural noise via ER contamination. This significantly generalizes the previous work’s attention-gain phenomenon beyond asymptotically clean features, to robustness at fixed feature-noise levels and under ER structural contamination. As a result, their theory speaks to attention advantages primarily in regimes with asymptotically clean features, whereas our focus is robustness when covariate noise remains substantial.

\textbf{Robustness to perturbations.} A large body of work studies worst‑case or adaptive perturbations of graph structure and covariates. \cite{10.1145/3219819.3220078} introduced Nettack, a targeted attack that greedily edits a small number of edges or features yet can drastically degrade the accuracy of GCN‑style models. \cite{zügner2024adversarialattacksgraphneural} formulated training‑time attacks as a bilevel optimization problem and used meta‑gradients to learn discrete edge perturbations, showing that small, carefully chosen changes to the graph can make GNNs perform worse than a classifier that ignores the graph altogether. Our focus is different: we analyze random measurement error in covariates and probabilistic Erd\H{o}s--R\'enyi contamination in edges and prove that a specific, non‑adaptive attention rule improves estimation and prediction in that stochastic setting. The two perspectives of adversarial robustness and stochastic noise are complementary.

\textbf{Latent‑position random graphs.} The geometric component of our model is a dot‑product mechanism, and the overall graph is its union with an independent Erd\H{o}s--R\'enyi graph. The theory of random geometric graphs is treated in the monograph \cite{penrose2003random}. Latent‑space models for networks were formalized by \cite{Hoff01122002}, while the dot‑product specialization appears in \cite{10.1007/978-3-540-77004-6_11}. For statistical methodology and asymptotics on random dot-product graphs, see \cite{JMLR:v18:17-448}. 

\textbf{Errors‑in‑variables and attenuation.} The inconsistency of OLS under covariate noise is classical; see \cite{10.5555/19255} and \cite{Carroll2006Measurement}. Our Theorem~\ref{main text theorem naive ols} instantiates attenuation in a networked setting while the proxy‑based OLS recovers consistency under geometric‑dominance conditions (Theorem~\ref{main text theorem proxy ols}).

\section{Node regression on random graphs}\label{model overview}

\subsection{Graph distribution}
Our graph generating distribution is a generalization of the random dot-product graphs, with additional Erd\H{o}s--R\'enyi contamination. Random dot-product graphs are a special form of random geometric graphs. In order to precisely define the model, we shall use the notation $G = (\bm X_n, E)$ to denote a graph $G$ on $n$ nodes, equipped with node covariates $\bm X_n =  \begin{bmatrix}
    \bm x_1,\ldots,\bm x_n
\end{bmatrix}^\top$ where $\bm x_i \in \bbR^d$ and a set of undirected edges $E$.

\begin{definition}[Erd\H{o}s--R\'enyi contaminated random dot-product graph]
    Consider a graph on $n$ nodes equipped with node covariates $\bm X_n = \begin{bmatrix}
        \bm x_1, \ldots, \bm x_n
    \end{bmatrix}^\top \in \bbR^{n\times d}$. We define the geometric edge set $E_1 \subset [n]^2$ as
    $$(i,j) \in E_1 \text{ if and only if } \bm x_i^\top \bm x_j \geq t\sqrt{d}$$
    for all $i,j\in [n], \, i\neq j$, and some parameter $t\geq 0$; and the Erd\H{o}s--R\'enyi edge set $E_2\subset [n]^2$ as
    $$\bbP\left((i,j) \in E_2 \right) = p,$$
    independently for all $i,j\in [n], \, i\neq j$ and some parameter $0\leq p<1$. The Erd\H{o}s--R\'enyi edges are also independent of the node covariates $\bm X_n$. The resulting graph $G = (\bm X_n, E_1\cup E_2)$ is an Erd\H{o}s--R\'enyi contaminated random dot-product graph, denoted by $\mathcal{G}(n,p,t,d)$.

\end{definition}

Having $p = 0$ corresponds to no Erd\H{o}s--R\'enyi edges in the graph, which is equivalent to a pure random dot-product graph. Increasing $p$ heuristically corresponds to increasing Erd\H{o}s--R\'enyi edge density, hence the noise in the graph.

\subsection{Covariates and response distribution}

Suppose that $G_0 = (\bm X_n, E) \sim \mathcal{G}(n,p_n,t_n,d)$ is an Erd\H{o}s--R\'enyi contaminated random dot-product graph. We observe $E$, but not the covariate matrix $\bm{X}_n$. Instead, we observe $\bm z_i$'s for each node $i\in [n]$, where
$$\bm z_i = \bm x_i + \bm \eta_i,$$
and assume $\bm x_i \sim N(\bm 0, \sigma_x^2\bm I_d), \bm \eta_i \sim N(\bm 0, \sigma_\eta^2\bm I_d)$, independently for all $i\in[n]$. For each node $i\in [n]$, we also observe a response variable $y_i$, and assume a true model
$$y_i = \bm x_i^\top \bm \beta + \varepsilon_i,$$
where $\varepsilon_i \sim N(0, \sigma_\varepsilon^2)$, for all $i\in [n]$, mutually independently and independently of $\bm x_i$ and $\bm \eta_i$'s. We additionally assume
$$\sup_d \bm \beta^\top \bm \beta<\infty.$$

\subsection{Asymptotic regime and parameters}

For the geometric edge threshold parameter, we choose  $t=\sigma_x^2t_n$. We further parametrize $t_n$ and the Erd\H{o}s--R\'enyi edge probability as
$$t_n = \sqrt{2(1-\alpha)\log n}, \quad p_n = n^{\gamma-1}, \quad 0<\alpha,\gamma<1.$$
The variance parameters $\sigma_x^2, \sigma_\eta^2$ and $\sigma_\varepsilon^2$, and the parameters $\alpha$ and $\gamma$ are assumed to be free of $n,d$. This reparametrization of $t_n$ and $p_n$ is useful because as we shall see in subsequent sections, the average degrees of the random dot-product graph and the Erd\H{o}s--R\'enyi graph are of the orders $n^\alpha$ and $n^\gamma$ respectively, up to log factors.

\paragraph{Problem statement.}
Node regression is a classical problem in statistics and machine learning \cite{hamilton2020graph} that focuses on modeling relationships between node-level covariates and responses within a network. In our setting, because the observed covariates satisfy $\bm z_i=\bm x_i+\bm \eta_i$, regressing $y_i$ on $\bm z_i$ is an {errors-in-variables} problem; the observed graph $G=(\bm Z_n,E)$ provides {side information} about relationships among the latent covariates $\{\bm x_i\}$ that we will exploit in our procedures. In the setup as described above, our objective is to
\begin{itemize}
    \item[(i)] Estimate $\bm \beta$ based on the observed graph and response variables $(G = (\bm Z_n, E), \{y_i\}_{i=1}^n)$,
    \item[(ii)] Predict the response $y_{n+1}$ for an unlabelled $(n+1)-$th node with observed covariate vector $\bm z_{n+1}$, given the updated observation $(G = (\bm Z_{n+1}, E), \{y_i\}_{i=1}^n)$, where $E$ is the updated set of edges between all $n+1$ nodes.
\end{itemize}

\section{Main results}\label{main results}

\subsection{Estimation algorithm through discretized attention-based proxies}
The baseline approach to estimate $\bm\beta$ is to plug in the observed $\bm Z_n$ in place of the unobserved $\bm X_n$ in the formula of the OLS estimate of $\bm \beta$, to get
$$\hat{\bm \beta}_z = (\bm Z_n^\top \bm Z_n)^{-1}\bm Z_n^\top \bm Y_n,$$
but as will be made precise in subsequent sections, this approach leads to attenuation bias. Instead, we employ a specifically-designed graph attention network to compute a proxy $\bm \lambda_i$ for each of the $\bm x_i$'s, and then regress the responses $y_i$ on the proxies to estimate $\bm \beta$. Algorithm~\ref{alg:attn-proxy-regression} describes the procedure to obtain the estimate $\hat{\bm \beta}_\lambda$. 

Algorithm~\ref{alg:attn-proxy-regression} builds a denoised covariate proxy for each node by {binary screening} of neighbours followed by {cross-part averaging}. 
Initially splitting each covariate vector into two disjoint blocks, screening uses dot-products of the {same} blocks, while averaging uses the {other} block; this split is essential to decouple selection noise from the averaging step. The discrete weights $w_{ij,k}=\mathbf{1}\{ \bm z_i^{(k)\top} \bm z_j^{(k)} \ge \sigma_x^2t_n\sqrt{d}/2\}$ implement the screening, and the resulting averages produce proxies close to the latent covariate $\bm x_i$ under the regimes stated later.

 In subsequent sections, we formalize how $\hat{\bm \beta}_\lambda$ consistently estimates $\bm \beta$, and how the main motivator is that the proxies $\{\bm \lambda_i\}$ are close to the unobserved covariates $\{\bm x_i\}$ in $L_2$ sense in a moderate high-dimensional regime of $n,d$.

\begin{algorithm}[ht]
   \caption{Regression coefficient ($\bm \beta$) estimation}
\label{alg:attn-proxy-regression}
\begin{algorithmic}
\STATE {\bfseries Input:} Observed graph $G=(\bm Z_n,E)$ with nodes $[n]$, response vector $\bm Y_n\in\mathbb{R}^{n}$, geometric threshold $t_n>0$.
\vspace{2mm}

\STATE {\bfseries Binary attention weights:} For every edge $(i,j)\in E$ and for $k\in\{1,2\}$, set
\[
w_{ij,k}\gets \mathds{1}\!\left\{\, \bm z_i^{(k)\top} \bm z_j^{(k)} \;\ge\; \sigma_x^2 t_n\sqrt{d}/2\; \right\}.
\]

\STATE {\bfseries Cross-part neighbour averaging:} For each $i\in[n]$, compute
$$\bm \lambda_i^{(1)} \!\leftarrow\! \frac{\sqrt{d}}{t_n}
  \frac{\sum_{j\in N_i} w_{ij,2}\bm z_j^{(1)}}{\sum_{j\in N_i} w_{ij,2}},$$ $$\bm \lambda_i^{(2)} \!\leftarrow\! \frac{\sqrt{d}}{t_n}
  \frac{\sum_{j\in N_i} w_{ij,1}\bm z_j^{(2)}}{\sum_{j\in N_i} w_{ij,1}}.$$
 If a denominator is zero, set the corresponding $$\bm \lambda_i^{(k)}\leftarrow \bm 0\in\mathbb{R}^{d/2}.$$

\STATE {\bfseries Concatenate proxies:} Set $\bm \lambda_i\gets \begin{bmatrix}\bm \lambda_i^{(1)} \\ \bm \lambda_i^{(2)}\end{bmatrix}\in\mathbb{R}^{d}$ and stack $\bm \Lambda_n\gets[\bm \lambda_1, \ldots, \bm \lambda_n]^\top\in\mathbb{R}^{n\times d}$.
\vspace{2mm}
\STATE {\bfseries OLS with proxies:} Return
\[
\hat{\bm\beta}_\lambda \;\gets\; (\bm \Lambda_n^\top \bm \Lambda_n)^{-1}\bm \Lambda_n^\top \bm Y_n.
\]
\end{algorithmic}
\end{algorithm}

\subsection{Efficient estimation guarantee of \texorpdfstring{$\beta$}{}}

The baseline approach that plugs in $\bm Z_n$ in place of $\bm X_n$ in the formula of the OLS estimate to obtain the estimate $\hat{\bm{\beta}}_z$, results in attenuation bias. This is a classically known result~\cite{10.5555/19255, Carroll2006Measurement}.
\begin{proposition}\label{main text theorem naive ols}
    Assume that $d\ll n$. Then the OLS estimator $\hat{\bm \beta}_z = (\bm Z_n^\top \bm Z_n)^{-1}\bm Z_n^\top \bm Y_n$ is not consistent for the parameter $\bm \beta$, in the sense that, as $n \longrightarrow \infty$,
    $$\|\hat{\bm \beta}_z - \bm \beta\| \stackrel{\mathbb{P}}{\longrightarrow} \frac{\sigma_\eta^2}{\sigma_x^2+\sigma_\eta^2}\|\bm \beta\|.$$
\end{proposition}

In fact, it can be noted from Proposition~\ref{main theorem beta y} (in Appendix~\ref{key results appdx}) that under the same conditions as in Proposition~\ref{main text theorem naive ols},
$$\left\|\hat{\bm \beta}_z - \frac{\sigma_x^2}{\sigma_x^2 + \sigma_\eta^2}\bm \beta\right\| \stackrel{\bbP}{\longrightarrow} 0.$$
This reflects a fundamental non-identifiability phenomenon when the graph is ignored: from the i.i.d. pairs $(\bm z_i, y_i)$ alone, one can at best identify the attenuated coefficient $\frac{\sigma_x^2}{\sigma_x^2+\sigma_\eta^2}\bm \beta$, not $\bm \beta$ itself. Although the marginal law of $\bm z_i$ allows consistent estimation of $\sigma_x^2+\sigma_\eta^2$, the latent variance $\sigma_x^2$, and hence the attenuation factor, cannot be recovered without additional structure, leading to $\bm \beta$ being non-identifiable.

On the other hand, in a certain regime of the graph size $n$, dimension $d$ and the parameters $\alpha, \gamma, \sigma_x^2,\sigma_\eta^2$, the $\bm \lambda_i$-based estimator, computed using Algorithm~\ref{alg:attn-proxy-regression}, turns out to be consistent for $\bm \beta$.
\begin{theorem}\label{main text theorem proxy ols}
    Assume that $\gamma < \alpha + \frac{\sigma_x^4(1-\alpha)}{2(\sigma_x^2+\sigma_\eta^2)^2}$ and $n^{\min\left\{\alpha,\frac{1}{3}\right\}}\gg d \gg (\log n)^3$. Then the estimator $\hat{\bm \beta}_\lambda = (\bm \Lambda_n^\top \bm \Lambda_n)^{-1}\bm \Lambda_n^\top \bm Y_n$ is consistent for the parameter $\bm \beta$, in the sense that, as $n \longrightarrow \infty$,
    $$\|\hat{\bm \beta}_\lambda - \bm \beta\| \stackrel{\mathbb{P}}{\longrightarrow} 0.$$
\end{theorem}

The signal-to-noise condition $\gamma < \alpha + \frac{\sigma_x^4(1-\alpha)}{2(\sigma_x^2+\sigma_\eta^2)^2}$ is required so that the signal from the geometric neighbours is not drowned by the noise from Erd\H{o}s--R\'enyi neighbours. Importantly, this does not require the set of edges in the observed graph to predominantly be geometric ones.
In our parametrization, the expected geometric degree scales like $n^\alpha$ while the expected ER degree scales like $n^\gamma$, so we explicitly allow regimes with $\gamma>\alpha$, where the number of geometric edges $n^\alpha$ can be negligible compared to the number of ER edges $n^\gamma$; the condition only rules out ER contamination so strong that it overwhelms the geometric signal after screening. 
On the other hand, heuristically, $d\ll n^\alpha$ ensures each node has “enough” informative geometric neighbours relative to the ambient dimension so that the screened averages can denoise and produce proxies $\bm \lambda_i$ that track $\bm x_i$ (formalized in Lemma~\ref{lambda close to x main text}), while the additional restriction $d\ll n^{1/3}$ ensures the empirical covariance matrices built from these proxies behave stably (so regression on $\bm \Lambda_n$ behaves like its population analogue; see Appendix~\ref{lemmas for thm b.2 appdx}). Finally, the lower bound $d\gg (\log n)^3$ is a convenient sufficient regime in which norms and dot-products are sharply concentrated uniformly over all $n$ nodes at the threshold scale $t_n \asymp \sqrt{\log n}$, making the screening step and neighbourhood sizes well-behaved.

The detailed proofs of Proposition~\ref{main text theorem naive ols} and Theorem~\ref{main text theorem proxy ols} are in Appendix~\ref{key results appdx}.

\subsection{Prediction algorithm for unlabelled node}

We now turn to the prediction problem: given an unlabelled $(n+1)$-th node with observed noise-amplified covariate vector $\bm z_{n+1}$ and edges to the existing graph, we wish to predict its unobserved response $y_{n+1}$. This places us in a transductive semi-supervised setting in the spirit of graph-based methods such as graph convolutional networks (GCNs) \cite{kipf2017semisupervisedclassificationgraphconvolutional}, where labels are available only on a subset of nodes and the goal is to exploit both node-level covariates and graph structure to infer the label of an unlabelled node.

One possible approach is to train a non-attention-based $L$-layer neural network on the graph on $[n]$ nodes, and use the fitted network to predict $y_{n+1}$. This approach leads to higher MSE (mean squared error) for the prediction when the average Erd\H{o}s--R\'enyi degree of a node is much larger compared to the average geometric degree, i.e. $\gamma>\alpha$ (Theorem~\ref{main text theorem non-att gnn prediction}). On the other hand, in a subregime where $\gamma>\alpha$, yet we can estimate $\bm \beta$ consistently, we can leverage Algorithm~\ref{alg:attn-proxy-regression} to calculate the denoised covariate proxy $\bm \lambda_{n+1}$ and specially estimate a consistent estimator of $\beta$ to predict $y_{n+1}$. This attention-based graph neural network approach leads to a lower MSE with high probability (Theorem~\ref{main text theorem gann prediction}). This approach is detailed in Algorithm~\ref{alg:attn-proxy-prediction}.

\begin{algorithm}[ht]
   \caption{Response variable $(y_{n+1})$ prediction for unlabelled node}
\label{alg:attn-proxy-prediction}
\begin{algorithmic}
\STATE {\bfseries Input:} Observed graph $G=(\bm Z_{n+1},E)$ with nodes $[n+1]$, response vector $\bm Y_n\in\mathbb{R}^{n}$, geometric threshold $t_n>0$.
\vspace{2mm}
\STATE {\bfseries Attention-based proxy:} Run Algorithm~\ref{alg:attn-proxy-regression} on $G=(\bm Z_{n+1},E)$ to find attention-based covariate proxy $\bm \lambda_{n+1}$ for the unlabelled node.
\vspace{2mm}
\STATE {\bfseries Algorithm~\ref{alg:attn-proxy-regression} on subgraph:} Use Algorithm~\ref{alg:attn-proxy-regression} on the subgraph of $G$ induced by set of nodes $[n]\setminus N_{n+1}$, where $N_{n+1}$ is the neighbourhood of the unlabelled node in $G$, to calculate $\hat{\bm \beta}_{\lambda,-(n+1)}$.
\vspace{2mm}
\STATE {\bfseries Prediction:} Return
\[
\hat{y}_{n+1} \gets \bm \lambda_{n+1}^\top \hat{\bm \beta}_{\lambda,-(n+1)}.
\]
\end{algorithmic}
\end{algorithm}

\subsection{Prediction MSE guarantee}

Theorem~\ref{main text theorem proxy ols} states a regime of parameters where Algorithm~\ref{alg:attn-proxy-regression} produces a consistent estimate of $\bm \beta$. In the same regime of the parameters, the prediction algorithm Algorithm~\ref{alg:attn-proxy-prediction} can predict $y_{n+1}$ with prediction MSE converging in probability to the oracle prediction MSE $\sigma_\varepsilon^2$.

\begin{theorem}\label{main text theorem gann prediction}
Suppose that $n^{\min\left\{\alpha, \frac{1}{3}\right\}} \gg d \gg (\log n)^3$ and $\alpha>\gamma + \frac{\sigma_x^4(\alpha-1)}{2(\sigma_x^2+\sigma_\eta^2)^2}$. Let $(\bm z_{n+1},y_{n+1})$ be the $(n+1)-$th node in the graph with neighbourhood $N_{n+1}$. Then, for the prediction $\hat{y}_{n+1}$ obtained through Algorithm~\ref{alg:attn-proxy-prediction}, we have
$$\bbE\left[\left(y_{n+1} - \hat{y}_{n+1}\right)^2\Bigr| \hat{\bm \beta}_{\lambda,-(n+1)}\right] \stackrel{\bbP}{\longrightarrow} \sigma_\varepsilon^2 ,$$    
as $n\longrightarrow\infty$.
\end{theorem}
The assumptions of this theorem are the same as the ones in Theorem~\ref{main text theorem proxy ols} because, heuristically, as soon as the regression coefficient $\bm \beta$ can be estimated consistently, one can use the consistent estimator and the attention-based proxy of the latent node covariate, which is also efficient under the same conditions, to produce a low MSE prediction of the response variable. In practice, to circumvent issues around dependence of the estimate of $\bm \beta$ obtained from the entire graph and the attention-based proxy $\bm \lambda_{n+1}$, we estimate $\bm \beta$ by executing Algorithm~\ref{alg:attn-proxy-regression} on a subgraph obtained by deleting the entire neighbourhood of the unlabelled $(n+1)$-th node.

The detailed proof of Theorem~\ref{main text theorem gann prediction} is in Appendix~\ref{key results appdx}.

\subsection{Prediction lower bound for graph convolutional network}

In addition, we show that the prediction approach based on a non-attention-based graph convolutional network results in worse performance
by establishing a lower bound on the MSE in the following theorem.
\begin{theorem}\label{main text theorem non-att gnn prediction}
    Assume that $(\psi_\ell)$ is a sequence of odd, $C_\psi$-Lipschitz $\bbR^d\mapsto \bbR^d$ activation functions and $(M_0^{(\ell)}, M_1^{(\ell)})$ is a sequence of $\bbR^{d\times d}$ matrices with their operator norm bounded above by $C_M$; and $(\psi_\ell, M_0^{(\ell)}, M_1^{(\ell)})$ are $\sigma\left(\left\{y_i\right\}_{i\in [n]},\left\{\bm z_i\right\}_{i\in [n+1]}\right)$ measurable. Let a $L$-layer network be defined as
$$\bm \xi_i^{(\ell+1)} = \psi_{\ell}\left(M_0^{(\ell)}\bm \xi_i^{(\ell)}+M_1^{(\ell)}\frac{1}{|N_i|}\sum_{j\in N_i}\bm \xi_j^{(\ell)}\right),$$
for $\ell=0,1,\ldots,L-1$,
and $\bm \xi_i^{(0)} = \bm z_i$ for all $i\in [n]$. 
If $\log n \ll d \ll n^\alpha$ and $\alpha<\gamma$, then
for any $\sigma\left(\left\{y_i\right\}_{i\in [n]},\left\{\bm z_i\right\}_{i\in [n+1]}\right)$ measurable $C$-Lipschitz function $g:\bbR^{d}\mapsto \bbR$,
$$\bbE\left(y_{n+1} - g\left(\bm \xi_{n+1}^{(L)}\right)\right)^2\geq \sigma_\varepsilon^2 + \frac{\sigma_x^2\sigma_\eta^2}{\sigma_x^2+\sigma_\eta^2}\|\bm \beta\|^2+o(1),$$
as $n,d\longrightarrow \infty$.
\end{theorem}

\begin{figure*}[tb]
\vskip 0.2in
\begin{center}
\centerline{\includegraphics[width=1.5\columnwidth]{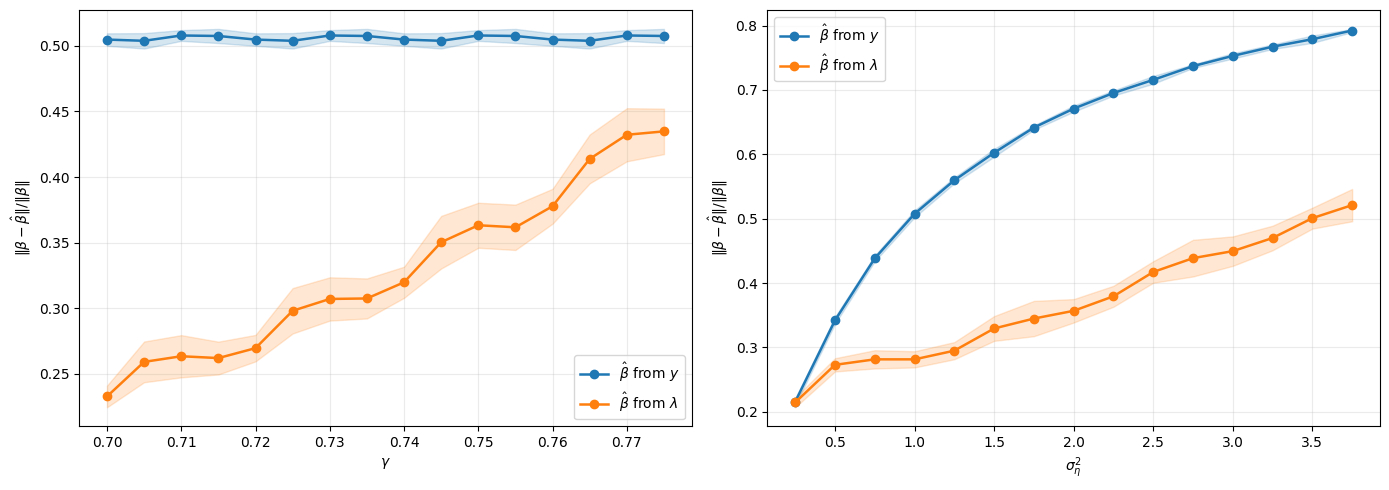}}
\caption{Comparison of relative $L_2$ errors of the estimates $\hat{\bm \beta}_\lambda$ and $\hat{\bm \beta}_z$ by varying $\gamma$ and $\sigma_\eta^2$ respectively. We fix the parameters $n=30000, d=250, \sigma_x^2=1,\sigma_\varepsilon^2=1,$ and $\alpha=0.72$. In the first plot, we vary $\gamma$ from $0.70$ to $0.75$ in $0.05$ increments and fix $\sigma_\eta^2=1$. In the second plot, we vary $\sigma_\eta^2$ from $0.25$ to $3.00$ in $0.25$ increments and fix $\gamma=0.725$. Curves show seed-averaged relative errors (mean $\pm \mathrm{SD}$ ) across $10$ seeds.}
\label{L2 error beta}
\end{center}
\vskip -0.2in
\end{figure*}

Theorem~\ref{main text theorem non-att gnn prediction} shows that increasing depth does not circumvent the fundamental error lower bound. For every finite-depth neighbourhood-aggregation GNN with Lipschitz updates/readout, the stated MSE lower bound still holds, so simply stacking more aggregation layers cannot improve performance in the ER-dominant regime $(\gamma>\alpha)$.
Thus, in the ER-dominant regime $(\gamma>\alpha)$, adding more neighbourhood-aggregation layers cannot recover the signal, while the proposed attention-based screening can (Theorem~\ref{main text theorem gann prediction}).

\subsection{Auxiliary result on denoising guarantee}\label{auxiliary results section}

At the center of the prediction guarantee is the observation that geometric edges preferentially connect nodes with aligned latent covariates, so averaging over geometric neighbours would denoise $\bm z_i$. The challenge is that observed neighbourhoods mix geometric and ER neighbours and the attention scores depend on noisy features. Our discretized attention therefore (i) screens neighbours to retain mostly geometric ones and (ii) uses cross-part screening/averaging to mitigate selection bias. Lemma~\ref{lambda close to x main text} formalizes that the resulting proxy $\bm \lambda_i$ is $L_2$-close to $\bm x_i$ under suitable scaling. Beyond the regression problem studied here, this suggests that similar discrete-attention, cross-part denoising layers could serve as a generic preprocessing step for other downstream tasks on contaminated graphs (e.g., classification or clustering) by producing denoised node features.

\begin{lemma}\label{lambda close to x main text}
    Assume that $\gamma < \alpha + \frac{\sigma_x^4(1-\alpha)}{2(\sigma_x^2+\sigma_\eta^2)^2}$ and $n^\alpha\gg d \gg (\log n)^3$. Then for each $i\in [n]$,
    $$\| \|\bm \lambda_i - \bm x_i\| \|_{L_2} = o(\sqrt{d}).$$
\end{lemma}
The proof of this lemma is in Appendix~\ref{lambda close to x appendix}. This lemma highlights how $\bm \lambda_i$ is a more efficient proxy for $\bm x_i$, as compared to $\bm z_i$, in the $L_2$ sense. This result is the main motivator to show how the regression estimator that uses $\bm \lambda_i$'s is more efficient than the one that uses $\bm z_i$'s. 

\section{Experiments}\label{numerical experiments}

We evaluate the proposed discretized-attention proxies in two complementary regimes. On synthetic graphs generated from the model in Section~\ref{model overview}, we isolate the estimation effect of the proxies by comparing the regression-coefficient error of proxy regression (Algorithm~\ref{alg:attn-proxy-regression}) to naive errors-in-variables OLS on the noisy covariates. On real graphs (OGBN-Products, OGBN-MAG (paper), and PyG-Reddit), we evaluate prediction MSE at scale. An important consideration in the real-data experiments is that a substantial fraction of nodes are low-degree. Since the proxy construction relies on averaging over screened neighbourhoods, its benefits are most reliable for nodes with sufficiently large neighbourhoods. We therefore report performance both overall and on a high-degree subset, and we make explicit when predictions use a low-degree fallback.

\subsection{Synthetic data}

\begin{figure*}[tb]
\vskip 0.2in
\begin{center}
\centerline{\includegraphics[width=1.9\columnwidth]{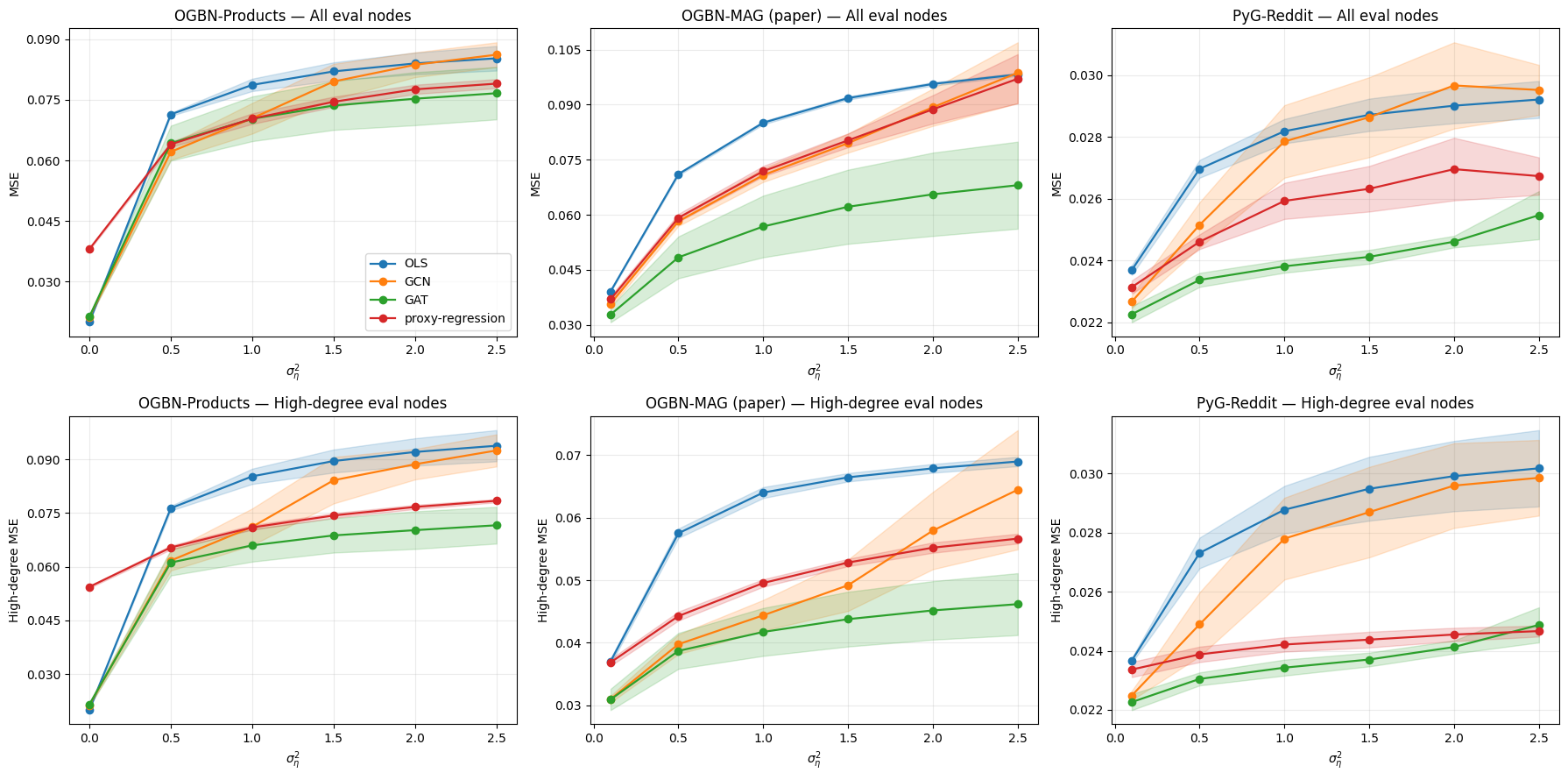}}
\caption{MSE comparison for graph-based regression prediction on OGBN-Products, OGBN-MAG (paper), and PyG-Reddit under ER edge contamination and covariate noise ($\sigma_\eta^2$).
Methods: OLS on noisy covariates, GCN, GAT, and proxy-regression (ours).
Curves show seed-averaged MSE (mean $\pm \,\mathrm{SD}$) across $10$ seeds.
Top row evaluates all nodes; bottom row evaluates high-degree nodes only. }
\label{pred error real}
\end{center}
\vskip -0.2in
\end{figure*}

\begin{table*}[tb]
\caption{Dataset statistics for the real-graph experiments.}
\begin{center}
\begin{small}
\begin{sc}
\begin{tabular}{lccccc}
\toprule
Dataset & \#nodes & \#edges & \#feature dim & \#held-out nodes & \#high-degree nodes \\
\midrule
ogbn-products    & 2{,}449{,}029 & 61{,}859{,}140 & 100 & 200{,}000 & $\sim$ 51\% \\
ogbn-mag (paper) & 736{,}839 & 5{,}396{,}336 & 128 & 80{,}000 & $\sim$ 21\% \\
pyg-reddit       &   232{,}965   & 57{,}307{,}946 & 602 & 50{,}000  &  $\sim$ 54\%\\
\bottomrule
\end{tabular}
\end{sc}
\end{small}
\end{center}
\vskip -0.1in
\label{dataset table}
\end{table*}

We simulate graphs exactly from the data-generating process in Section~\ref{model overview}. We compare the two estimators $\hat{\bm \beta}_z$ (Proposition~\ref{main text theorem naive ols}) and $\hat{\bm \beta}_\lambda$ (Algorithm~\ref{alg:attn-proxy-regression}). Performance is measured by the relative coefficient estimation error $\frac{\|\hat{\bm \beta} - \bm \beta\|}{\|\bm \beta\|}$, averaged across independent runs (with variability bands).

Across parameter sweeps, proxy regression consistently improves over naive OLS. As ER contamination increases, proxy quality degrades as expected, but proxy regression remains markedly more accurate than regressing on $\bm Z_n$. The benefits of proxy-regression are most pronounced at higher covariate-level noise levels $\sigma_\eta^2$, aligning with the goal of undoing attenuation and consistent with our theoretical predictions (Figure~\ref{L2 error beta}).
\subsection{Real data}

We next test whether the same denoising mechanism improves prediction on real-world graph datasets:
OGBN-Products, OGBN-MAG (paper) \cite{10.5555/3495724.3497579}, and PyG-Reddit \cite{10.5555/3294771.3294869} (Table~\ref{dataset table}). 
Since these benchmarks do not come with regression responses, we construct a controlled node-regression task via standard preprocessing. We preprocess the raw node features $\bm v_i$ into covariates $\bm x_i = f(\bm v_i)$, where $f$ is trained on the observed graph using an edge-prediction objective. We then
treat the $\{\bm x_i\}$ as latent covariates and generate noisy covariates $\bm z_i = \bm x_i + \bm \eta_i$ and
responses $y_i = \bm x_i^\top \bm \beta + \varepsilon_i$, sweeping the covariate-noise level $\sigma_\eta^2$. To probe
robustness to additional non-geometric structure and to vary effective neighbourhood sizes, we also augment the
observed graph by adding Erd\H{o}s--R\'enyi edges at controlled rates. 

We compare four predictors: OLS on $\bm z$
(ignoring the graph), a trained GCN regressor, a trained GAT regressor, and our proxy-regression predictor, which is a deliberately simplified GAT-style discrete-attention architecture designed to match the mechanism we can analyze
theoretically. Note that the GAT regressor we train uses a scaled dot-product attention rule (query–key dot-products with softmax normalization) in the style of \cite{NIPS2017_3f5ee243}, rather than the more commonly used additive MLP-based attention scoring used in the original GAT formulation \cite{veličković2018graphattentionnetworks}. Our goal is not to be state of the art (indeed, we expect a flexible learned GAT to be strong), but
to test whether our simple, analyzable attention-denoising rule yields systematic gains over non-attention message
passing and naive OLS under controlled covariate noise and edge contamination. 

Evaluation is conducted on held-out nodes, chosen by randomly subsampling nodes over multiple seeds, and then averaging errors. As a measure of performance, we report test MSE as a function of the covariate-noise level $\sigma_\eta^2$ under ER edge contamination (Figure~\ref{pred error real}). Since the proxy construction in our method relies on screened neighbourhood averages, its gains are concentrated on nodes with sufficiently large neighbourhoods. Accordingly, we apply Algorithm~\ref{alg:attn-proxy-prediction} only on a dataset-specific high-degree subset (while retaining low-degree neighbours), and use a GCN fallback for the remaining low-degree nodes (Appendix~\ref{experiment appdx}). We therefore report both overall MSE (top row of Figure~\ref{pred error real}), where ``proxy-regression'' should be read as ``proxy-regression on high-degree nodes + GCN elsewhere''; and MSE restricted to high-degree evaluation nodes (bottom row of Figure~\ref{pred error real}), where the curves isolate proxy regression without fallback and make the denoising effect most visible. In these high-degree plots, proxy-regression (our proposed method) shows more pronounced gains over OLS and the trained GCN baseline, while a fully trained GAT performs even better, as expected. Finally, at this high edge-contamination level, the relative advantage of our attention-based method over GCN becomes more pronounced as $\sigma_\eta^2$ increases, consistent with our theory’s prediction that attention helps more when feature-noise is larger. 

All experiments were run on a computing cluster using NVIDIA H200 GPUs (one GPU per job), with SLURM-managed resources (64 CPU cores and 50 GB RAM).

\section{Discussion}
In this paper, we study a node regression task on a random geometric graph model
and show that a specifically-designed attention-based graph neural network achieves consistent parameter estimation and provably outperforms traditional graph convolutional networks in terms of prediction error.
A natural extension is to move beyond the random dot-product graph model and develop analogous guarantees for other graph generative mechanisms.
While we focus on regression tasks in this paper, the same denoising proxies could potentially serve as a preprocessing step for other downstream node-level or edge-level tasks such as node classification and link prediction, and it would be interesting to understand when a similar ``attention vs. averaging'' separation persists.

\bibliography{gatreg}
\bibliographystyle{icml2026}

\newpage
\appendix
\onecolumn
\section{Model description and notations}

\subsection{Observation and regression model}

$G_0 = (\bm X_n, E)$ is a random dot-product graph with contaminating edges from an independent Erd\H{o}s--R\'enyi model. We observe $E$, but not the covariate matrix $\bm{X}_n$. Instead, we observe $\bm z_i$'s for each node $i\in [n]$, where
$$\bm z_i = \bm x_i + \bm \eta_i,$$
and assume $\bm x_i \sim N(\bm 0, \sigma_x^2\bm I_d), \bm \eta_i \sim N(\bm 0, \bm \sigma_\eta^2 I_d)$, independently for all $i\in[n]$. For each node $i\in [n]$, we also observe a response variable $y_i$, and assume a true model
$$y_i = \bm x_i^\top \bm \beta + \varepsilon_i,$$
where $\varepsilon_i \sim N(0, \sigma_\varepsilon^2)$, for all $i\in [n]$, mutually independently and independently of $\bm x_i$ and $\bm \eta_i$'s. We additionally assume
$$\sup_d \bm \beta^\top \bm \beta<\infty.$$

\subsection{Notation}

Recall that there is a geometric edge between nodes $i$ and $j$ if and only if
$$\bm x_i^\top \bm x_j \geq \sigma_x^2t_n\sqrt{d},$$
while independent Erd\H{o}s--Renyi edges occur independently, each with probability $p_n$.
We consider 
$$t_n = \sqrt{2(1-\alpha)\log n}, \quad p_n=n^{\gamma-1}, \qquad 0<\alpha,\gamma<1.$$
The variance parameters $\sigma_x^2, \sigma_\eta^2, \sigma_\varepsilon^2$, and the parameters $\alpha, \gamma$ are free of $n,d$.

For each $i \in [n]$, we use the notation $N_i$ to denote the neighbourhood of $i$ in the observed $G = (\bm Z_n, E)$, 
$$N_i := \{j \in [n]: (i,j)\in E\}.$$
We assume that $d$ is even, and for any vector $\bm v \in \mathbb{R}^d$, we denote by $\bm v^{(1)}$ the top $d/2$ coordinates of $\bm v$, and by $\bm v^{(2)}$ the bottom $d/2 $ coordinates of $\bm v$. To distinguish between notations for Euclidean length, and $L_r$ norms of random variables, we use the explicit notations
$$\|\bm v\| = \sqrt{\bm v^\top \bm v}, \qquad \|X\|_{L_r} = \left(\mathbb{E}[|X|^r]\right)^{1/r}.$$
We shall denote $\bm v_i = \bm \lambda_i - \bm x_i$. Define 
$$\bm X_n = \begin{bmatrix}
    \bm x_1, \ldots, \bm x_n
\end{bmatrix}^\top.$$ 
Similar to $\bm X_n$, construct matrices $\bm Z_n, \bm \Lambda_n, \bm N_n$ and $\bm V_n$ by stacking the vectors $\bm z_i, \bm \lambda_i, \bm \eta_i$ and $\bm v_i$ respectively. Finally, define $\bm Y_n = [y_1, \ldots, y_n]^T \in \bbR^n$ and $\bm E_n = [\varepsilon_1, \ldots, \varepsilon_n]^T \in \bbR^n$.
For any matrix $\bm M \in \mathbb{R}^{m\times m}$, we denote its operator norm by
$$\|\bm M\|_{\mathrm{op}} = \sup\left\{ \frac{\|\bm M \bm v\|}{\|\bm v\|}: \bm v \in \mathbb{R}^m \right\}.$$
We also use the standard order notations $o(\cdot), O(\cdot), \Theta(\cdot)$.

\section{Key Results}\label{key results appdx}

In this section, we begin with the main motivating result - that the attention-based message passing proxies $\bm \lambda_i$ are close to the latent covariates $\bm x_i$ in $L_2$ sense. It drives our main theorems which state that the OLS estimator of $\bm \beta$ based on $\bm \lambda_i$'s is consistent while the one based on the noisy observations $\bm z_i$'s is not, and also that the prediction of the response variable for an unlabelled node has lower MSE when one uses the attention-based network instead of a non-attention-based one. We prove the main theorems here following their corresponding statement.

\begin{lemma}\label{lambda close to x}
    Assume that $n^\alpha\gg d \gg (\log n)^3$ and $\alpha>\gamma+\frac{\sigma_x^4(\alpha-1)}{2(\sigma_x^2+\sigma_\eta^2)^2}$. Then for any fixed $r\geq 1$,
    $$\left\|\bm \lambda_i - \bm x_i\right\|_{L_r} = O\left(\sqrt{d}\left(n^{\gamma+\frac{\sigma_x^4(\alpha-1)}{2(\sigma_x^2+\sigma_\eta^2)^2}-\alpha}+\frac{1}{t_n}+\frac{t_n^3}{\sqrt{d}}+\frac{\sqrt{d}}{\sqrt{t_nn^\alpha}}\right)\right)=o(\sqrt{d}).$$
    \begin{proof}
        The proof is deferred to Appendix~\ref{lambda close to x appendix}.
    \end{proof}
\end{lemma}

\begin{theorem}\label{main theorem beta lambda}
    Suppose that $n^{\min\{\alpha,\frac{1}{3}\}}\gg d \gg (\log n)^3$ and $\alpha > \gamma + \frac{\sigma_x^4(\alpha-1)}{2(\sigma_x^2+\sigma_\eta^2)^2}$. Then
        $$\left\|\hat{\bm \beta}_\lambda - \bm \beta\right\|= O_\bbP\left(\sqrt{\frac{d}{n}} + \frac{d^{3/2}\sqrt{\log d}}{t_n\sqrt{n}} + n^{\gamma+\frac{\sigma_x^4(\alpha-1)}{2(\sigma_x^2+\sigma_\eta^2)^2}-\alpha}+\frac{1}{t_n}+\frac{t_n^3}{\sqrt{d}}+\frac{\sqrt{d}}{\sqrt{t_nn^\alpha}}  \right).$$
        \begin{proof}
            We shall use the abbreviation $$\text{err}_{n,d} = n^{\gamma+\frac{\sigma_x^4(\alpha-1)}{2(\sigma_x^2+\sigma_\eta^2)^2}-\alpha}+\frac{1}{t_n}+\frac{t_n^3}{\sqrt{d}}+\frac{\sqrt{d}}{\sqrt{t_nn^\alpha}}.$$ Also, let us use the notation
        $$\Tilde{A}_n = \frac{\bm \Lambda_n^\top \bm \Lambda_n}{n}, \qquad \Tilde{A}_0 = \sigma_x^2\bm I_{d}, \qquad \Tilde{B}_n = \frac{\bm \Lambda_n^\top \bm X_n}{n}, \qquad \Tilde{B}_0 = \sigma_x^2\bm I_d.$$
        Then, from Lemma~\ref{Conv of XtX and EtE}, Lemma~\ref{R diagonal}, Lemma~\ref{the freedman lemma} and Lemma~\ref{the freedman lemma 2},
        \begin{align*}
            \|\Tilde{A}_n - \Tilde{A}_0\|_{\mathrm{op}} &\leq \left\| \frac{\bm X_n^\top \bm X_n}{n} - \sigma_x^2\bm I_d\right\|_{\mathrm{op}} + \left\| \frac{\bm X_n^\top \bm V_n}{n}\right\|_{\mathrm{op}}+\left\| \frac{\bm V_n^\top \bm X_n}{n}\right\|_{\mathrm{op}}+\left\| \frac{\bm V_n^\top \bm V_n}{n}\right\|_{\mathrm{op}}\\
            &= O_\bbP\left(\sqrt{\frac{d}{n}} + \frac{d^{3/2}\sqrt{\log d}}{t_n\sqrt{n}} + \text{err}_{n,d}  \right),
        \end{align*}
        and
        \begin{align*}
            \|\Tilde{B}_n - \Tilde{B}_0\|_{\mathrm{op}} \leq \left\| \frac{\bm X_n^\top \bm X_n}{n} - \sigma_x^2 \bm I_d\right\|_{\mathrm{op}} + \left\| \frac{\bm V_n^\top \bm X_n}{n}\right\|_{\mathrm{op}} = O_\bbP\left(\sqrt{\frac{d}{n}} + \frac{d^{3/2}\sqrt{\log d}}{t_n\sqrt{n}} + \text{err}_{n,d} \right).
        \end{align*}
        
        Because $\|\Tilde{A}_0^{-1}\|_{\mathrm{op}} = \sigma_x^{-2}$ is constant and $\|\Tilde{A}_n- \Tilde{A}_0\|_{\mathrm{op}}=o_\bbP(1)$, Lemma~\ref{Neumann series} gives
        $$\|\Tilde{A}_n^{-1} - \Tilde{A}_0^{-1}\|_{\mathrm{op}} = O_\bbP\left(\sqrt{\frac{d}{n}} + \frac{d^{3/2}\sqrt{\log d}}{t_n\sqrt{n}} + \text{err}_{n,d}  \right), \qquad\text{and}\qquad \|\Tilde{A}_n^{-1}\|_{\mathrm{op}} = O_\bbP(1).$$
        We also have
        $$\|\Tilde{B}_n\|_{\mathrm{op}} \leq \|\Tilde{B}_0\|_{\mathrm{op}} + \|\Tilde{B}_n-\Tilde{B}_0\|_{\mathrm{op}} =\sigma_x^2+O_\bbP\left(\sqrt{\frac{d}{n}} + \frac{d^{3/2}\sqrt{log d}}{t_n\sqrt{n}} + \text{err}_{n,d} \right)=\Theta_\bbP(1). $$

        Now, note that
        \begin{equation}\label{attenuation decomp 2}
            \hat{\bm\beta}_\lambda - \bm\beta = \Tilde{A}_n^{-1}\left(\Tilde{B}_n\bm \beta + \frac{1}{n}\bm \Lambda_n^\top \bm E_n \right)-\bm \beta
            = \left(\Tilde{A}_n^{-1}\Tilde{B}_n - \Tilde{A}_0^{-1}\Tilde{B}_0\right)\bm \beta + \frac{1}{n}\Tilde{A}_n^{-1}\bm \Lambda_n^\top\bm E_n.
        \end{equation}
        
        Using $\sup_n \|\bm \beta\|<\infty$, we get
        \begin{align*}
            \left\|\left(\Tilde{A}_n^{-1}\Tilde{B}_n - \Tilde{A}_0^{-1}\Tilde{B}_0\right)\bm \beta\right\| &\leq \left( \left\|\Tilde{A}_n^{-1}-\Tilde{A}_0^{-1}\right\|_{\mathrm{op}}\|\Tilde{B}_n\|_{\mathrm{op}} + \|\Tilde{A}_0^{-1}\|_{\mathrm{op}}\|\Tilde{B}_n-\Tilde{B}_0\|_{\mathrm{op}} \right)\|\bm \beta\|\\
            &= O_\bbP\left(\sqrt{\frac{d}{n}} + \frac{d^{3/2}\sqrt{\log d}}{t_n\sqrt{n}} + \text{err}_{n,d}  \right).
        \end{align*}
        We also get using Lemma~\ref{conv of xtu} and Lemma~\ref{lambda t u},
        \begin{align*}
            &\left\|\frac{\bm \Lambda_n^\top \bm E_n}{n}\right\| \leq \left\|\frac{\bm X_n^\top \bm E_n}{n}\right\|+\left\|\frac{\bm V_n^\top \bm E_n}{n}\right\| = O_\bbP\left(\sqrt{\frac{d}{n}}\right) + O_\bbP\left( \sqrt{\frac{d}{n}}\sqrt{{\sqrt{\frac{d^3\log d}{nt_n^2}}}+\text{err}_{n,d}^2} \right)=O_\bbP\left(\sqrt{\frac{d}{n}}\right)\\
            \implies& \left\|\frac{1}{n}\Tilde{A}_n^{-1}\bm \Lambda_n^\top\bm E_n\right\|\leq \|\Tilde{A}_n^{-1}\|_{\mathrm{op}}\left\|\frac{1}{n}\bm \Lambda_n^T\bm E_n\right\|_{\mathrm{op}} = O_\bbP\left(\sqrt{\frac{d}{n}}\right).
        \end{align*}
        Plugging these orders into Eq.~\ref{attenuation decomp 2}, we get
        \begin{align*}
            \left\|\hat{\bm\beta}_\lambda - \bm\beta\right\| = O_\bbP\left(\sqrt{\frac{d}{n}} + \frac{d^{3/2}\sqrt{\log d}}{t_n\sqrt{n}} + \text{err}_{n,d}  \right).
        \end{align*}
        \end{proof}
\end{theorem}

    \begin{theorem}\label{main theorem beta y}
        Suppose that $d\ll n$. Then
        $$\left\|\hat{\bm \beta}_z - \bm \beta\right\|= \frac{\sigma_\eta^2}{\sigma_x^2+\sigma_\eta^2}\|\bm \beta\|+O_\bbP\left(\sqrt{\frac{d}{n}}\right).$$
        \begin{proof}
            Let us use the notation
        $$A_n = \frac{\bm Z_n^\top \bm Z_n}{n}, \qquad A_0 = (\sigma_x^2+\sigma_\eta^2)\bm I_{d}, \qquad B_n = \frac{\bm Z_n^\top \bm X_n}{n}, \qquad B_0 = \sigma_x^2\bm I_d.$$
        Then, from Lemma~\ref{Conv of XtX and EtE},
        $$\|A_n - A_0\|_{\mathrm{op}} \leq \left\| \frac{\bm X_n^\top \bm X_n}{n} - \sigma_x^2\bm I_d\right\|_{\mathrm{op}} + \left\| \frac{\bm X_n^\top \bm N_n}{n}\right\|_{\mathrm{op}}+\left\| \frac{\bm N_n^\top \bm X_n}{n}\right\|_{\mathrm{op}}+\left\| \frac{\bm N_n^T \bm N_n}{n} - \sigma_\eta^2\bm I_d\right\|_{\mathrm{op}}=O_\bbP\left(\sqrt{\frac{d}{n}}\right),$$
        $$\|B_n - B_0\|_{\mathrm{op}} \leq \left\| \frac{\bm X_n^\top \bm X_n}{n} - \sigma_x^2\bm I_d\right\|_{\mathrm{op}} + \left\| \frac{\bm N_n^\top \bm X_n}{n}\right\|_{\mathrm{op}} = O_\bbP\left(\sqrt{\frac{d}{n}}\right).$$
        Because $\|A_0^{-1}\|_{\mathrm{op}} = (\sigma_x^2+\sigma_\eta^2)^{-1}$ is constant and $\|A_n- A_0\|_{\mathrm{op}}=o_\bbP(1)$, Proposition~\ref{Neumann series} gives
        $$\|A_n^{-1} - A_0^{-1}\|_{\mathrm{op}} = O_\bbP\left( \sqrt{\frac{d}{n}} \right), \qquad\text{and}\qquad \|A_n^{-1}\|_{\mathrm{op}} = O_\bbP(1).$$
        We also have
        $$\|B_n\|_{\mathrm{op}} \leq \|B_0\|_{\mathrm{op}} + \|B_n-B_0\|_{\mathrm{op}} =\sigma_x^2 + O_\bbP\left(\sqrt{\frac{d}{n}}\right)=\Theta_\bbP(1). $$

        Now, note that
        \begin{equation}\label{attenuation decomp}
            \hat{\bm\beta}_z - \frac{\sigma_x^2}{\sigma_x^2+\sigma_\eta^2}\bm\beta = A_n^{-1}\left(B_n\bm \beta + \frac{1}{n}\bm Z_n^\top \bm E_n \right) - \frac{\sigma_x^2}{\sigma_x^2+\sigma_\eta^2}\bm \beta
            = \left(A_n^{-1}B_n - A_0^{-1}B_0\right)\bm \beta + \frac{1}{n}A_n^{-1}\bm Z_n^\top\bm E_n.
        \end{equation}
        
        Using $\sup_n \|\bm \beta\|<\infty$, we get
        \begin{align*}
            \left\|\left(A_n^{-1}B_n - A_0^{-1}B_0\right)\bm \beta\right\| &\leq \left( \left\|A_n^{-1}-A_0^{-1}\right\|_{\mathrm{op}}\|B_n\|_{\mathrm{op}} + \|A_0^{-1}\|_{\mathrm{op}}\|B_n-B_0\|_{\mathrm{op}} \right)\|\bm \beta\|\\
            &= O_\bbP\left(\sqrt{\frac{d}{n}}\right)\Theta_\bbP(1) + \frac{1}{\sigma_x^2+\sigma_\eta^2}O_\bbP\left(\sqrt{\frac{d}{n}}\right)=O_\bbP\left(\sqrt{\frac{d}{n}}\right).
        \end{align*}
        We also get using Lemma~\ref{conv of xtu},
        \begin{align*}
            &\left\|\frac{\bm Z_n^\top \bm E_n}{n}\right\| \leq \left\|\frac{\bm X_n^\top \bm E_n}{n}\right\|+\left\|\frac{\bm N_n^\top \bm E_n}{n}\right\| = O_\bbP\left(\sqrt{\frac{d}{n}}\right)\\
            \implies& \left\|\frac{1}{n}A_n^{-1}\bm Z_n^T\bm E_n\right\| \leq \|A_n^{-1}\|_{\mathrm{op}}\left\|\frac{1}{n}\bm Z_n^\top\bm E_n\right\| = O_\bbP\left(\sqrt{\frac{d}{n}}\right).
        \end{align*}
        Plugging these orders into Eq.~\ref{attenuation decomp}, we get
        \begin{align*}
            \left\|\hat{\bm\beta}_z - \frac{\sigma_x^2}{\sigma_x^2+\sigma_\eta^2}\bm\beta\right\| = O_\bbP\left(\sqrt{\frac{d}{n}}\right).
        \end{align*}
        Since 
        $$\hat{\bm \beta}_z - \bm \beta = -\frac{\sigma_\eta^2}{\sigma_x^2+\sigma_\eta^2}\bm \beta + \left(\hat{\bm\beta}_z - \frac{\sigma_x^2}{\sigma_x^2+\sigma_\eta^2}\bm\beta\right),$$
        the reverse triangle inequality gives
        $$\left|\left\|\hat{\bm \beta}_z - \bm \beta\right\|-\frac{\sigma_\eta^2}{\sigma_x^2+\sigma_\eta^2}\|\bm \beta\|\right|\leq \left\|\hat{\bm\beta}_z - \frac{\sigma_x^2}{\sigma_x^2+\sigma_\eta^2}\bm\beta\right\| = O_\bbP\left(\sqrt{\frac{d}{n}}\right).$$
        \end{proof}
    \end{theorem}

\begin{theorem}\label{attn lower MSE}
Suppose that $n^{\min\left\{\alpha, \frac{1}{3}\right\}} \gg d \gg (\log n)^3$ and $\alpha>\gamma + \frac{\sigma_x^4(\alpha-1)}{2(\sigma_x^2+\sigma_\eta^2)^2}$. Let $(\bm z_{n+1},y_{n+1})$ be the $(n+1)-$th node in the graph with neighbourhood $N_{n+1}$. Let $\bm \lambda_{n+1}$ be calculated as in Algorithm~\ref{alg:attn-proxy-regression} based on the full graph on $[n+1]$, and let $\hat{\bm \beta}_{\lambda,-(n+1)}$ be calculated according to Algorithm~\ref{alg:attn-proxy-regression} but based on the subgraph induced by nodes $[n]\setminus N_{n+1}$ instead of the full graph on $[n]$. Then,
$$\bbE\left(y_{n+1} - \bm \lambda_{n+1}^\top \hat{\bm \beta}_{\lambda,-(n+1)}\Bigr| \hat{\bm \beta}_{\lambda,-(n+1)}\right)^2= \sigma_\varepsilon^2 + o_\bbP(1).$$
    \begin{proof}
        Note that $\hat{\bm \beta}_{\lambda, -(n+1)}$ is independent of both $\bm x_{n+1}$ and $\bm \lambda_{n+1}$. Since, by Lemma~\ref{er ngbhd size} and Lemma~\ref{pure geo ngbhd size}, $|N_{n+1}|=O_\bbP(n^\alpha/t_n) = o(n)$, we have $|[n] \setminus N_{n+1}|=n(1-o(1))$ and we can still apply Theorem~\ref{main theorem beta lambda}, and get
        $$\|\hat{\bm \beta}_{\lambda,-(n+1)} - \bm \beta\| = o_\bbP(1).$$
        
        We start by decomposing
        \begin{equation}\label{decompose gann}
            \begin{split}
                &\bbE\left(y_{n+1} - \bm \lambda_{n+1}^\top \hat{\bm \beta}_{\lambda,-(n+1)}\Bigr|\hat{\bm \beta}_{\lambda,-(n+1)}\right)^2\\
                =& \bbE\left(\bm x_{n+1}^\top\bm\beta + \varepsilon_{n+1} - \bm x_{n+1}^\top\hat{\bm \beta}_{\lambda,-(n+1)} + \bm x_{n+1}^\top\hat{\bm \beta}_{\lambda,-(n+1)} - \bm \lambda_{n+1}^\top \hat{\bm \beta}_{\lambda,-(n+1)}\Bigr|\hat{\bm \beta}_{\lambda,-(n+1)}\right)^2\\
            =& \sigma_\varepsilon^2 + \bbE\left(\bm x_{n+1}^\top\bm\beta  - \bm x_{n+1}^\top\hat{\bm \beta}_{\lambda,-(n+1)} + \bm x_{n+1}^\top\hat{\bm \beta}_{\lambda,-(n+1)} - \bm \lambda_{n+1}^\top \hat{\bm \beta}_{\lambda,-(n+1)}\Bigr| \hat{\bm \beta}_{\lambda,-(n+1)}\right)^2.
            \end{split}
        \end{equation}
        Now, using independence of $\hat{\bm \beta}_{\lambda,-(n+1)}$ and $\bm x_{n+1}$, and Theorem~\ref{main theorem beta lambda}, we have
        \begin{align*}
            &\bm x_{n+1}^\top\bm\beta  - \bm x_{n+1}^\top\hat{\bm \beta}_{\lambda,-(n+1)}\Bigr| \hat{\bm \beta}_{\lambda,-(n+1)} \sim N\left(0, \sigma_x^2\left\|\bm \beta - \hat{\bm \beta}_{\lambda,-(n+1)}\right\|^2\right)\\
            \implies& \bbE\left[\left(\bm x_{n+1}^\top\bm\beta  - \bm x_{n+1}^\top\hat{\bm \beta}_{\lambda,-(n+1)}\right)^2\Bigr| \hat{\bm \beta}_{\lambda,-(n+1)}\right] = \sigma_x^2\left\|\bm \beta - \hat{\bm \beta}_{\lambda,-(n+1)}\right\|^2 = o_\bbP(1).
        \end{align*}
        On the other hand, using the independence  $(\bm x_{n+1},\bm \lambda_{n+1}) \indep \hat{\bm \beta}_{\lambda,-(n+1)}$, and Lemma~\ref{R diagonal},
        \begin{align*}
            \bbE\left[\left(\bm x_{n+1}^\top \hat{\bm \beta}_{\lambda,-(n+1)} - \bm \lambda_{n+1}^\top \hat{\bm \beta}_{\lambda,-(n+1)}\right)^2\Bigr|\hat{\bm \beta}_{\lambda,-(n+1)}\right] &\leq \|\hat{\bm \beta}_{\lambda,-(n+1)}\|^2\left\|\bbE(\bm v_{n+1}\bm v_{n+1}^\top)\right\|_{\mathrm{op}}\\
            &= \left(\|\bm \beta\|^2 + o_\bbP(1)\right) o(1) = o_\bbP(1).
        \end{align*}
        Thus it follows that
        \begin{align*}
             &\bbE\left(\bm x_{n+1}^\top\bm\beta  - \bm x_{n+1}^\top\hat{\bm \beta}_{\lambda,-(n+1)} + \bm x_{n+1}^\top\hat{\bm \beta}_{\lambda,-(n+1)} - \bm \lambda_{n+1}^\top \hat{\bm \beta}_{\lambda,-(n+1)}\Bigr| \hat{\bm \beta}_{\lambda,-(n+1)}\right)^2\\
             \leq & 2\left[\bbE\left(\bm x_{n+1}^\top\bm\beta  - \bm x_{n+1}^\top\hat{\bm \beta}_{\lambda,-(n+1)}\Bigr|\hat{\bm \beta}_{\lambda,-(n+1)}\right)^2 + \bbE\left(\bm x_{n+1}^\top\hat{\bm \beta}_{\lambda,-(n+1)} - \bm \lambda_{n+1}^\top \hat{\bm \beta}_{\lambda,-(n+1)}\Bigr|\hat{\bm \beta}_{\lambda,-(n+1)}\right)^2\right] = o_\bbP(1).
        \end{align*}
        Hence, combining the orders, we get from Eq.~\ref{decompose gann},
        $$\bbE\left(y_{n+1} - \bm \lambda_{n+1}^\top \hat{\bm \beta}_{\lambda,-(n+1)}\Bigr| \hat{\bm \beta}_{\lambda,-(n+1)}\right)^2= \sigma_\varepsilon^2 + o_\bbP(1).$$
    \end{proof}
\end{theorem}    

\begin{theorem}\label{non-attn higher MSE}
Assume that $(\psi_\ell)$ is a sequence of odd, $C_\psi$-Lipschitz $\bbR^d\mapsto \bbR^d$ activation functions and $(M_0^{(\ell)}, M_1^{(\ell)})$ is a sequence of $\bbR^{d\times d}$ matrices with their operator norm bounded above by $C_M$; and $(\psi_\ell, M_0^{(\ell)}, M_1^{(\ell)})$ are $\sigma\left(\left\{y_i\right\}_{i\in [n]},\left\{\bm z_i\right\}_{i\in [n+1]}\right)$ measurable. Let a $L$-layer network be defined as
$$\bm \xi_i^{(\ell+1)} = \psi_{\ell}\left(M_0^{(\ell)}\bm \xi_i^{(\ell)}+M_1^{(\ell)}\frac{1}{|N_i|}\sum_{j\in N_i}\bm \xi_j^{(\ell)}\right), \qquad \ell=0,1,\ldots,L-1,$$
and $\bm \xi_i^{(0)} = \bm z_i$ for all $i\in [n]$. 
If $\log n \ll d \ll n^\alpha$ and $\alpha<\gamma$, then
for any $\sigma\left(\left\{y_i\right\}_{i\in [n]},\left\{\bm z_i\right\}_{i\in [n+1]}\right)$ measurable $C$-Lipschitz function $g:\bbR^{d}\mapsto \bbR$,
$$\bbE\left(y_{n+1} - g\left(\bm \xi_{n+1}^{(L)}\right)\right)^2\geq \sigma_\varepsilon^2 + \frac{\sigma_x^2\sigma_\eta^2}{\sigma_x^2+\sigma_\eta^2}\|\bm \beta\|^2+o(1).$$
    \begin{proof}
        Consider any $\sigma(\left\{y_i,\bm z_i\right\}:1\leq i\leq n)$ measurable function $g:\bbR^d \mapsto \bbR$ such that $g$ is $C$-Lipschitz. First note that, we need to consider only those functions $g$ such that $\|g(\phi_L(\bm z_{n+1}))\|_{L_2} = O(1)$, because if instead $\|g(\phi_L(\bm z_{n+1}))\|_{L_2}\gg1$, then
        \begin{align*}
            &y_{n+1} - g\left(\bm \xi_{n+1}^{(L)}\right) = \left(y_{n+1} - g(\phi_L(\bm z_{n+1}))\right) + \left(g(\phi_L(\bm z_{n+1})) - g\left(\bm \xi_{n+1}^{(L)}\right)\right)\\
             \implies& \bbE\left|y_{n+1} - g\left(\bm \xi_{n+1}^{(L)}\right)\right|^2 \geq \left(\left\|g(\phi_L(\bm z_{n+1}))\right\|_{L_2} - \left\|y_{n+1}\right\|_{L_2} - \left\|g(\phi_L(\bm z_{n+1})) - g(\bm \xi_{n+1}^{(L)})\right\|_{L_2}\right)^2. 
        \end{align*}
        By assumption, $y_{n+1} \sim N(0, \sigma_x^2\|\bm \beta\|^2+\sigma_\varepsilon^2)$, hence $\|y_{n+1}\|_{L_2}=O(1)$ and using the $C$-Lipschitz property of $g$ and Lemma~\ref{GNN output close to non-GNN}, 
        $$\left\|g(\phi_L(\bm z_{n+1})) - g\left(\bm \xi_{n+1}^{(L)}\right)\right\|_{L_2}\leq C\left\|\left\| \phi_L(\bm z_{n+1}) - \bm \xi_{n+1}^{(L)}\right\|\right\|_{L_2}= O\left(\sqrt{d}n^{-\gamma/2}+n^{\alpha-\gamma}\right)=o(1).$$
        Plugging the orders in and recalling that for those $g$, we have $\|g(\phi_L(\bm z_{n+1}))\|_{L_2}\gg1$,
        $$\bbE \left|y_{n+1} - g\left(\bm \xi_{n+1}^{(L)}\right)\right|^2 \gg 1 .$$
        Because $\sigma_\varepsilon^2 + \frac{\sigma_x^2\sigma_\eta^2}{\sigma_x^2+\sigma_\eta^2}\|\bm \beta\|^2 = O(1)$, the claim of the lemma holds for all such $g$ for which we have $\|g(\phi_L(\bm z_{n+1}))\|_{L_2}\gg1$. So, for the rest of the proof we shall restrict ourselves to only such $g$ that $\|g(\phi_L(\bm z_{n+1}))\|_{L_2} = O(1)$.

        We write
        \begin{equation}\label{decompose gnn}
            \begin{split}
                \bbE\left(y_{n+1} - g\left(\bm \xi_{n+1}^{(L)}\right)\right)^2 &= \bbE\left(y_{n+1} - g\left(\phi_L(\bm z_{n+1})\right) + g\left(\phi_L(\bm z_{n+1})\right) - g\left(\bm \xi_{n+1}^{(L)}\right)\right)^2\\
            &= \bbE\left(y_{n+1} - g\left(\phi_L(\bm z_{n+1})\right)\right)^2 + \bbE\left(g\left(\phi_L(\bm z_{n+1})\right) - g\left(\bm \xi_{n+1}^{(L)}\right)\right)^2\\&\qquad\qquad + 2\bbE\left[\left(y_{n+1} - g\left(\phi_L(\bm z_{n+1})\right)\right)\left(g\left(\phi_L(\bm z_{n+1})\right) - g\left(\bm \xi_{n+1}^{(L)}\right)\right)\right].
            \end{split}
        \end{equation}
        Note that, for the first summand above,
        \begin{align*}
            \bbE\left(y_{n+1} - g\left(\phi_L(\bm z_{n+1})\right)\right)^2&= \bbE\left(\bm x_{n+1}^\top\beta +\varepsilon_{n+1} - g\left(\phi_L(\bm z_{n+1})\right)\right)^2\\
            &= \sigma_\varepsilon^2 + \bbE\left(\bm x_{n+1}^\top\beta  - g\left(\phi_L(\bm z_{n+1})\right)\right)^2\\
            &\geq \sigma_\varepsilon^2 +  \bbE\left(\bm x_{n+1}^\top\beta  - \bbE[\bm x_{n+1}^\top\beta | \bm z_{n+1}]\right)^2\\
            &= \sigma_\varepsilon^2 + \bbE\left[\operatorname{Var}\left(\bm x_{n+1}^\top\beta|\bm z_{n+1}\right)\right] = \sigma_\varepsilon^2 + \frac{\sigma_x^2\sigma_\eta^2}{\sigma_x^2+\sigma_\eta^2}\|\bm \beta\|^2.
        \end{align*}
        Using the $C$-Lipschitz property of $g$ and Lemma~\ref{GNN output close to non-GNN}, for the second summand,
        \begin{align*}
            \bbE\left(g\left(\phi_L(\bm z_{n+1})\right) - g\left(\bm \xi_{n+1}^{(L)}\right)\right)^2 \leq C^2\bbE\left\|\phi_L(\bm z_{n+1}) - \bm \xi_{n+1}^{(L)}\right\|^2 = o(1).
        \end{align*}
        And finally for the third summand,
        \begin{align*}
            &2\bbE\left[\left(y_{n+1} - g\left(\phi_L(\bm z_{n+1})\right)\right)\left(g\left(\phi_L(\bm z_{n+1})\right) - g\left(\bm \xi_{n+1}^{(L)}\right)\right)\right]\\
            \leq& 2\sqrt{\bbE\left[y_{n+1} - g\left(\phi_L(\bm z_{n+1})\right)\right]^2\bbE\left[g\left(\phi_L(\bm z_{n+1})\right) - g\left(\bm \xi_{n+1}^{(L)}\right)\right]^2} = o(1),
        \end{align*}
        which follows using $\bbE\left[g\left(\phi_L(\bm z_{n+1})\right) - g\left(\bm \xi_{n+1}^{(L)}\right)\right]^2=o(1)$ and $\|g(\phi_L(\bm z_{n+1}))\|_{L_2} = O(1)$, we have
        $$\sqrt{\bbE\left[y_{n+1} - g\left(\phi_L(\bm z_{n+1})\right)\right]^2}\leq \|y_{n+1}\|_2 + \left\|g\left(\phi_L(\bm z_{n+1})\right)\right\|_2 = O(1).$$
        Combining the above into Eq.~\ref{decompose gnn},
        $$\bbE\left(y_{n+1} - g\left(\bm \xi_{n+1}^{(L)}\right)\right)^2\geq \sigma_\varepsilon^2 + \frac{\sigma_x^2\sigma_\eta^2}{\sigma_x^2+\sigma_\eta^2}\|\bm \beta\|^2+o(1).$$
    \end{proof}
\end{theorem}

\section{Lemmas required for proof of Theorem~\ref{main theorem beta lambda}}\label{lemmas for thm b.2 appdx}

\subsection{Rotational invariance and population covariance structure}

In order to prove the concentration of the random matrices $\frac{\bm \Lambda_n^\top\bm \Lambda_n}{n}$ and $\frac{\bm \Lambda_n^\top\bm X_n}{n}$, the first step is to find the non-random matrices that they would concentrate around. For that, we identify that the joint distributions of $(\bm x_i, \bm \eta_i)$ are rotationally invariant and leverage that to prove that the above-mentioned covariance matrices concentrate around diagonal matrices.

For this section and the next, recall that we use the notation $\bm v_i = \bm \lambda_i - \bm x_i$ and stack them in the matrix $\bm V_n = \begin{bmatrix}
    \bm v_1, \ldots, \bm v_n
\end{bmatrix}^\top$.
\begin{lemma}\label{orthogonal matrix invariance}
    For any fixed orthogonal matrices $\bm Q_1, \bm Q_2 \in \bbR^{\frac{d}{2}\times \frac{d}{2}}$, define $\bm Q = \begin{bmatrix}
        \bm Q_1 & \bm 0 \\ \bm 0 & \bm Q_2
    \end{bmatrix}$. Then, one has
    $$(\bm x_i, \bm v_i) \stackrel{d}{=} (\bm Q\bm x_i, \bm Q \bm v_i).$$
    \begin{proof}
        Define the transformation
        $$T_{\bm Q}(\{\bm x_i\}_{i\in [n]}, \{\bm \eta_i\}_{i\in [n]}) = (\{\bm Q\bm x_i\}_{i\in [n]}, \{\bm Q\bm \eta_i\}_{i\in [n]}).$$
        Then it follows that $T_{\bm Q}(\{\bm z_i\}_{i\in [n]}) = \{\bm Q \bm z_i\}_{i\in [n]}$. 
By spherical symmetry of the distributions of $\bm x_i, \bm \eta_i$ and independence of $N_i^{(ER)}$, we get
        $$\left(\{\bm x_i\}_{i\in [n]}, \{\bm \eta_i\}_{i\in [n]}, N_i^{(ER)}\right) \stackrel{d}{=} T_{\bm Q}\left(\{\bm x_i\}_{i\in [n]}, \{\bm \eta_i\}_{i\in [n]}, N_i^{(ER)}\right)$$
        $$\implies (\bm x_i,\bm \lambda_i) \stackrel{d}{=} T_{\bm Q}(\bm x_i, \bm \lambda_i ) .$$
        
        Consider the notations $$w_{ij,2}^{(1)} = \indic\{\bm z_i^{(2)\top}\bm z_j^{(2)} \geq \sigma_x^2t_n\sqrt{d}/2\}, \qquad w_{ij,2}^{(2)} = \indic\{\bm x_i^\top \bm x_j\geq \sigma_x^2t_n\sqrt{d}, \bm z_i^{(2)\top}\bm z_j^{(2)} \geq \sigma_x^2t_n\sqrt{d}/2\},$$
        $$w_{ij,1}^{(1)} = \indic\{\bm z_i^{(1)\top}\bm z_j^{(1)} \geq \sigma_x^2t_n\sqrt{d}/2\}, \qquad w_{ij,1}^{(2)} = \indic\{\bm x_i^\top\bm x_j\geq \sigma_x^2t_n\sqrt{d}, \bm z_i^{(1)\top}\bm z_j^{(1)} \geq \sigma_x^2t_n\sqrt{d}/2\}.$$ 
        We immediately get
        $$\bm z_i^{(1)\top}\bm z_j^{(1)} = (\bm Q_1\bm z_i^{(1)\top})(\bm Q_1\bm z_j^{(1)}) = T_{\bm Q}(\bm z_i^{(1)\top}\bm z_j^{(1)}), \quad \bm z_i^{(2)\top}\bm z_j^{(2)} = (\bm Q_2\bm z_i^{(2)\top})(\bm Q_2\bm z_j^{(2)}) = T_{\bm Q}(\bm z_i^{(2)\top}\bm z_j^{(2)}),$$
        $$\bm x_i^{\top}\bm x_j = (\bm Q\bm x_i^{\top})(\bm Q\bm x_j)=T_{\bm Q}(\bm x_i^{\top}\bm x_j),$$
        implying $$(w_{ij,2}^{(1)}, w_{ij,2}^{(2)}, w_{ij,1}^{(1)}, w_{ij,1}^{(2)}) = T_{\bm Q}(w_{ij,2}^{(1)}, w_{ij,2}^{(2)}, w_{ij,1}^{(1)}, w_{ij,1}^{(2)}).$$ Consequently, it follows that
        \begin{align*}
            T_{\bm Q}\left(\bm \lambda_i^{(1)}\right) &= \frac{\sqrt{d}}{t_n} \cdot \frac{\sum_{j\in T_{\bm Q}\left(N_i^{(ER)}\right)} T_{\bm Q}\left(w_{ij,2}^{(1)}\right) T_{\bm Q}\left(\bm z_j^{(1)}\right) + \sum_{j\in [n]\setminus T_{\bm Q}\left(N_i^{(ER)}\right)} T_{\bm Q}\left(w_{ij,2}^{(2)}\right) T_{\bm Q}\left(\bm z_j^{(1)}\right)}{\sum_{j\in T_{\bm Q}\left(N_i^{(ER)}\right)} T_{\bm Q}\left(w_{ij,2}^{(1)}\right) + \sum_{j\in [n]\setminus T_{\bm Q}\left(N_i^{(ER)}\right)} T_{\bm Q}\left(w_{ij,2}^{(2)}\right)}\\[4mm]
            &\stackrel{d}{=} \frac{\sqrt{d}}{t_n} \cdot \frac{\sum_{j\in N_i^{(ER)}} w_{ij,2}^{(1)} \bm Q_1\bm z_j^{(1)} + \sum_{j\in [n]\setminus N_i^{(ER)}} w_{ij,2}^{(2)} \bm Q_1\bm z_j^{(1)}}{\sum_{j\in N_i^{(ER)}} w_{ij,2}^{(1)} + \sum_{j\in [n]\setminus N_i^{(ER)}} w_{ij,2}^{(2)}}\\[4mm]
            &= \bm Q_1 \bm \lambda_i^{(1)}.
        \end{align*}
        Similarly, we also get $T_{\bm Q}(\bm \lambda_i^{(2)}) \stackrel{d}{=} \bm Q_2 \bm \lambda_i^{(2)}$, which we combine with the above and obtain
        $$T_{\bm Q}\left(\bm \lambda_i\right) = T_{\bm Q}{\begin{pmatrix}
                \bm \lambda_i^{(1)} \\[3mm] \bm \lambda_i^{(2)}
            \end{pmatrix}} \stackrel{d}{=} \begin{pmatrix}
                \bm Q_1 \bm \lambda_i^{(1)} \\[3mm] \bm Q_2 \bm \lambda_i^{(2)}
            \end{pmatrix}= \bm Q \bm \lambda_i.$$
            
        Thus we have obtained
        $$(\bm x_i, \bm \lambda_i) \stackrel{d}{=} T_{\bm Q}(\bm x_i, \bm \lambda_i) = (\bm Q \bm x_i, \bm Q \bm \lambda_i) \implies (\bm x_i, \bm v_i) = (\bm x_i, \bm \lambda_i-\bm x_i) \stackrel{d}{=} (\bm Q\bm x_i, \bm Q(\bm \lambda_i-\bm x_i)) = (\bm Q\bm x_i, \bm Q\bm v_i).$$
        This concludes the proof.
    \end{proof}
\end{lemma}

\begin{lemma}\label{R diagonal}
    Suppose $n \gg d \gg (\log n)^3$. Then
    $$\left\| \bbE\left(\frac{\bm V_n^\top \bm V_n}{n}\right) \right\|_{\mathrm{op}}, \left\| \bbE\left(\frac{\bm X_n^\top \bm V_n}{n}\right) \right\|_{\mathrm{op}}^2, \left\| \bbE\left(\frac{\bm V_n^\top \bm X_n}{n}\right) \right\|_{\mathrm{op}}^2 = O\left(\text{err}_{n,d}^2\right),$$
    where
    $$\text{err}_{n,d} = n^{\gamma+\frac{\sigma_x^4(\alpha-1)}{2(\sigma_x^2+\sigma_\eta^2)^2}-\alpha}+\frac{1}{t_n}+\frac{t_n^3}{\sqrt{d}}+\frac{\sqrt{d}}{\sqrt{t_nn^\alpha}}.$$
    \begin{proof}
    We have by Lemma~\ref{orthogonal matrix invariance}, for any orthogonal matrix $\bm Q \in \bbR^{d\times d}$, 
    $$\bm Q \bbE\left[ \bm v_i \bm v_i^\top \right] \bm Q^\top = \bbE\left[ \bm Q \bm v_i \bm v_i^\top \bm Q^\top \right]
            = \bbE\left[ (\bm Q \bm v_i) (\bm Q \bm v_i)^\top \right]
            = \bbE[\bm v_i\bm v_i^\top].$$
            Then, by Proposition~\ref{orthogonal matrix linalg fact} and Lemma~\ref{lambda close to x},
        $$\left\|\bbE\left(\frac{\bm V_n^\top \bm V_n}{n}\right)\right\|_{\mathrm{op}} = \|\bbE(\bm v_i\bm v_i^\top)\|_{\mathrm{op}} \leq \frac{2\bbE\left(\sum_{j=1}^d \bm v_i(j)^2\right)}{d} = \frac{2\|\|\bm v_i\|\|_{L_2}^2}{d} = O\left(\text{err}_{n,d}^2\right).$$

        Similarly, we again have from Lemma~\ref{orthogonal matrix invariance}, for any orthogonal matrix $\bm Q \in \bbR^{d\times d}$, 
    $$\bm Q \bbE\left[ \bm x_i \bm v_i^\top \right] \bm Q^\top = \bbE\left[ \bm Q \bm x_i \bm v_i^\top \bm Q^\top \right]
            = \bbE\left[ (\bm Q \bm x_i) (\bm Q \bm v_i)^\top \right]
            = \bbE[\bm x_i\bm v_i^\top].$$
            Then, again, by Proposition~\ref{orthogonal matrix linalg fact} and Lemma~\ref{lambda close to x},
        $$\left\|\bbE\left(\frac{\bm X_n^\top \bm V_n}{n}\right)\right\|_{\mathrm{op}} = \|\bbE(\bm x_i\bm v_i^\top)\|_{\mathrm{op}} \leq \frac{2\bbE\left(\sum_{j=1}^d \bm |\bm x_i(j)\bm v_i(j)|\right)}{d} \leq \frac{2}{d}\sqrt{\|\|\bm x_i\|\|_{L_2}^2\|\|\bm v_i\|\|_{L_2}^2} = O\left(\text{err}_{n,d}\right).$$
        
        An identical argument works for $\left\| \bbE\left(\frac{\bm V_n^\top \bm X_n}{n}\right) \right\|_{op}$.
    \end{proof}
\end{lemma}

\subsection{Concentration of random covariance matrices}
While the previous section finds the expectation of certain covariance matrices, this section proves that the concentration indeed occurs for large values of $n,d$. There is dependence among the random vectors $\bm v_i$, which makes the proofs of concentrations less simple and we use Theorem 1.2 in \cite{10.1214/ECP.v16-1624} to work around it. As apparent from the statement of the lemmas that follow in this section, the concentrations require the additional condition $d\ll n^{1/3}$ that appear in Theorem~\ref{main theorem beta lambda}.

\begin{lemma}\label{the freedman lemma}
Assume that $n^\alpha\gg d\gg (\log n)^3$ and $\alpha>\gamma+\frac{\sigma_x^4(\alpha-1)}{2(\sigma_x^2+\sigma_\eta^2)^2}$. Then
$$\left\|\frac{1}{n}\bm V_n^\top\bm V_n - \bbE\left[\frac{1}{n}\bm V_n^\top\bm V_n\right]\right\|_{op} = O_\bbP\left(\frac{d^{3/2}\sqrt{\log d}}{t_n\sqrt{n}}\right).$$
    \begin{proof}
    
        Define the event
        $$\mathcal E_n = \left\{\max_{i\in[n]}\|\bm x_i\| = O(\sqrt{d}), \max_{i\in [n]}\|\bm \eta_i\| = O(\sqrt{d}), \max_{i\in[n],k=1,2}|\Tilde{N}_i^{(k)}| = \Theta\left(\frac{n^\alpha}{t_n}\right), \min_{i\in[n],k=1,2}|\Tilde{N}_i^{(k)}| = \Theta\left(\frac{n^\alpha}{t_n}\right)\right\}.$$
        By Lemma~\ref{norm x_i lemma} and Lemma~\ref{wij count}, 
        $\bbP(\mathcal E_n^c)\longrightarrow 0$. It follows that inside event $\mathcal E_n$, $\max_{i\in[n]}\|\bm z_i\| = O(\sqrt{d})$. Also, inside event $\mathcal E_n$, for all $i\in [n]$,
        $$\|\bm \lambda_i^{(1)}\| = \frac{\sqrt{d}}{t_n|\Tilde{N}_i^{(1)}|}\left\|\sum_{j\in \Tilde{N}_i^{(1)}} \bm z_j^{(1)}\right\| \leq \frac{\sqrt{d}}{t_n}\max_{j}\|\bm z_j^{(1)}\| = O\left(\frac{d}{t_n}\right).$$
        Similarly, $\|\bm \lambda_i^{(2)}\| = O\left(\frac{d}{t_n}\right)$ for all $i\in [n]$, and thus $\max_{i\in[n]}\|\bm \lambda_i\| = O\left(\frac{d}{t_n}\right)$  inside event $\mathcal E_n$. Recalling that $\|\bm v_i\| \leq \|\bm x_i\| + \|\bm \lambda_i\|$, it also follows that $\max_{i\in[n]}\|\bm v_i\| = O\left(\frac{d}{t_n}\right)$ inside event $\mathcal E_n$.

        Let $e_{ij}$ be the Erd\H{o}s--R\'enyi edge indicator between $i$ and $j$. Define $M_i := \{\bm x_i, \bm \eta_i, \{e_{ij}:j>i\}\}$, and the natural filtration $\mathcal F_i := \sigma(M_1,M_2,\ldots,M_i)$ for all $i\in [n]$. Write $Z = (M_1,\ldots,M_n)$, and for all $m\in [n]$,
        $$Z^{(m)} := (M_1,\ldots,M_{m-1},M_m',M_{m+1},\ldots,M_n),$$
        where $M_m' = \{\bm x_m', \bm \eta_m', \{e'_{mj}:j>m\}\}$ is an i.i.d. copy of $M_m$, i.e. $\bm x_m'$ is an i.i.d. copy of $\bm x_m$, and $\bm \eta_m'$ is an i.i.d. copy of $\eta_m$, and $\{e'_{mj}:j>m\}\}$ are i.i.d. copies of $\{e_{mj}:j>m\}\}$. We are allowed to resample a subset of Erd\H{o}s--R\'enyi edges since all such edges are mutually independent and also independent from everything else. Let
        $$F_n(Z) = \frac{1}{n}\bm V_n^\top\bm V_n.$$
        Then,
        $$\Delta_m^F = \bbE[F_n(Z) - F_n(Z^{(m)}) | \mathcal F_m]$$
        are Doob's martingale increments. 
        
        Fix $m\in [n]$. After resampling $M_m' = \{\bm x_m', \bm \eta_m', \{e'_{mj}:j>m\}\}$, let us denote 
        the new versions of $\bm \lambda_i$ as $\bm \lambda_i'$, of $\bm v_i$ as $\bm v_i'$, of $\Tilde{N}_i^{(k)}$ as $\Tilde{N}_i^{(k)'}$, for all $i\in [n]$ and $k=1,2$. Observe that
        $$F_n(Z) - F_n(Z^{(m)}) = \frac{1}{n}\sum_{i=1}^n (\bm v_i \bm v_i^\top - \bm v_i'\bm v_i'^{\top}).$$
        So, inside $\mathcal E_n$,
        $$\|F_n(Z) - F_n(Z^{(m)})\|_{\mathrm{op}} \leq \frac{1}{n}\sum_{i=1}^n \left\|\bm v_i \bm v_i^\top - \bm v_i'\bm v_i'^{\top}\right\|_{\mathrm{op}} \leq \frac{1}{n}\sum_{i=1}^n (\|\bm v_i\|+\|\bm v_i'\|)\left\|\bm v_i - \bm v_i'\right\|=O\left(\frac{d}{nt_n}\right)\sum_{i=1}^n \left\|\bm v_i - \bm v_i'\right\|.$$
        Note that $\left\|\bm v_i - \bm v_i'\right\|$
        is nonzero if either $i=m$ or $m$ is present in at least one of the old or new screened neighbourhoods of $i$. The second case can be written more precisely as $i\neq m$ but $m\in \bigcup_{j=1}^4 A_{ji}$, where
        $$A_{1i}=(\Tilde{N}_i^{(1)}\cap \Tilde{N}_i^{(1)'})\cup (\Tilde{N}_i^{(2)}\cap \Tilde{N}_i^{(2)'}), \quad A_{2i}=(\Tilde{N}_i^{(1)}\cap \Tilde{N}_i^{(1)'})\cup (\Tilde{N}_i^{(2)}\triangle \Tilde{N}_i^{(2)'}),$$
        $$A_{3i}=(\Tilde{N}_i^{(1)}\triangle \Tilde{N}_i^{(1)'})\cup (\Tilde{N}_i^{(2)}\cap \Tilde{N}_i^{(2)'}), \quad A_{4i}=(\Tilde{N}_i^{(1)}\triangle \Tilde{N}_i^{(1)'})\cup (\Tilde{N}_i^{(2)}\triangle \Tilde{N}_i^{(2)'}).$$
         
        \begin{itemize}
        \item[(i)] $i=m$\\[3mm]
        We use triangle inequality to get, inside event $\mathcal E_n$,
        $$\|\bm v_m - \bm v_m'\| \leq \|\bm v_m\| +\| \bm v_m'\| = O\left(\frac{d}{t_n}\right).$$
        
        Also, using Lemma~\ref{lambda close to x},
        \begin{align*}
            \bbE\left(\|\bm v_m \bm v_m^\top - \bm v_m'\bm v_m'^\top\|_{\mathrm{op}}^2\right) &\leq \bbE\left(\left(\|\bm v_m\|+\|\bm v_m'\|\right)^2\left(\|\bm v_m-\bm v_m'\|^2\right)\right)\\
            &\leq 2\left\|\|\bm v_m\|+\|\bm v_m'\|\right\|_{L_4}^4=O(d^2).
        \end{align*}
        
        \item[(ii)] $i\neq m$ but $m\in \Tilde{N}_i^{(1)}\cap \Tilde{N}_i^{(1)'}$\\[3mm]
        
        In this case, $\bm x_i = \bm x_i'$ and $\Tilde{N}_i^{(1)} = \Tilde{N}_i^{(1)'}$, so
        \begin{align*}
            \|\bm v_i^{(1)} - \bm v_i'^{(1)}\| = \|\bm \lambda_i^{(1)} - \bm \lambda_i'^{(1)}\| &= \frac{\sqrt{d}}{t_n}\left\|\frac{\sum_{j\in \Tilde{N}_i^{(1)}}\bm z_j^{(1)}}{|\Tilde{N}_i^{(1)}|} - \frac{\sum_{j\in \Tilde{N}_i^{(1)'}}\bm z_j'^{(1)}}{|\Tilde{N}_i^{(1)'}|}\right\|= \frac{\sqrt{d}}{t_n}\left\|\frac{\bm z_m^{(1)} - \bm z_m'^{(1)}}{|\Tilde{N}_i^{(1)}|}\right\|. 
        \end{align*}
        So, inside event $\mathcal E_n$,
        $$\|\bm v_i^{(1)} - \bm v_i'^{(1)}\| =\frac{\sqrt{d}}{t_n} \cdot \frac{O(\sqrt{d})}{\Theta\left(\frac{n^\alpha}{t_n}\right)} = O\left( \frac{d}{n^\alpha} \right).$$

        Also, for any fixed $r\geq 1$,
        $$\left\|\|\bm v_i^{(1)} - \bm v_i'^{(1)}\|\cdot \indic\{m\in \Tilde{N}_i^{(1)}\cap \Tilde{N}_i^{(1)'}\}\right\|_{L_r} \leq \frac{\sqrt{d}}{t_n}\left\|\bm z_m^{(1)} - \bm z_m'^{(1)}\right\|_{L_{2r}}\left\|\frac{1}{|\Tilde{N}_i^{(1)}|}\right\|_{L_{2r}} = O\left(\frac{\sqrt{d}}{t_n}\cdot\sqrt{d}\cdot \frac{t_n}{n^\alpha}\right)=O\left( \frac{d}{n^\alpha} \right),$$
        where the bounds follow from using Holder's inequality with Lemma~\ref{inverse wij count moments} and Lemma~\ref{wij count}.

        \item[(iii)] $i\neq m$ but $m \in \Tilde{N}_i^{(1)}\triangle \Tilde{N}_i^{(1)'}$\\[3mm]
        In this case $\bm x_i = \bm x_i'$. Also since only $M_m$ has been replaced by its independent copy $M_m'$, we must have $\Tilde{N}_i^{(1)}\triangle \Tilde{N}_i^{(1)'}=\{m\}$. Let us assume without loss of generality, $ \Tilde{N}_i^{(1)} \setminus \Tilde{N}_i'^{(1)} = \{m\}$. 
        Then,
        \begin{align*}
            \|\bm v_i^{(1)} - \bm v_i'^{(1)}\|=\|\bm \lambda_i^{(1)} - \bm \lambda_i'^{(1)}\| &= \frac{\sqrt{d}}{t_n}\left\| \frac{\sum_{j\in \Tilde{N}_i^{(1)}}\bm z_j^{(1)}}{|\Tilde{N}_i^{(1)}|} - \frac{\sum_{j\in \Tilde{N}_i^{(1)'}}\bm z_j'^{(1)}}{|\Tilde{N}_i^{(1)'}|} \right\|\\
                &= \frac{\sqrt{d}}{t_n}\cdot\frac{\left\| |\Tilde{N}_i^{(1)'}|\sum_{j\in \Tilde{N}_i^{(1)}}\bm z_j^{(1)} - |\Tilde{N}_i^{(1)}|\sum_{j\in \Tilde{N}_i^{(1)'}}\bm z_j'^{(1)} \right\|}{|\Tilde{N}_i^{(1)}||\Tilde{N}_i^{(1)'}|}\\
                &\leq \frac{\sqrt{d}}{t_n}\cdot\frac{\left(\sum_{j\in \Tilde{N}_i^{(1)}\setminus \{m\}}\left\| \bm z_j^{(1)}\right\|\right) + |\Tilde{N}_i^{(1)'}|\|\bm z_m^{(1)} \|}{|\Tilde{N}_i^{(1)}||\Tilde{N}_i^{(1)'}|}.,
        \end{align*}
        Then, inside event $\mathcal E_n$,
        $$\|\bm v_i^{(1)} - \bm v_i'^{(1)}\|\leq \frac{\sqrt{d}}{t_n}\left( \frac{O\left(\frac{n^\alpha}{t_n}\sqrt{d}\right)}{\Theta\left( \frac{n^\alpha}{t_n}\right)^2} \right) = O\left(\frac{d}{n^\alpha}\right).$$
        Also, for any fixed $r\geq 1$,
        \begin{align*}
            \left\|\|\bm v_i^{(1)} - \bm v_i'^{(1)}\|\cdot \indic\{m\in \Tilde{N}_i^{(1)}\triangle\Tilde{N}_i^{(1)'}\}\right\|_{L_r} &\leq \frac{\sqrt{d}}{t_n}\left\|\left(\sum_{j\in \Tilde{N}_i^{(1)}\setminus \{m\}}\left\| \bm z_j^{(1)}\right\|\right) + |\Tilde{N}_i^{(1)'}|\|\bm z_m^{(1)} \|\right\|_{L_{2r}}\left\|\frac{1}{|\Tilde{N}_i^{(1)}|}\right\|_{L_{4r}}\left\|\frac{1}{|\Tilde{N}_i^{(1)'}|}\right\|_{L_{4r}} \\
            &= \frac{\sqrt{d}}{t_n}O\left( \frac{n^\alpha}{t_n}\cdot \sqrt{d}\right)O\left(\left(\frac{t_n}{n^\alpha}\right)^2\right) = O\left(\frac{d}{n^\alpha}\right),
        \end{align*}
        where the bounds follow by repeated application of Holder's inequality and Lemma~\ref{inverse wij count moments} and Lemma~\ref{wij count}.
        \end{itemize}

        Observe that same bounds hold for $\left\|\|\bm v_i^{(2)} - \bm v_i'^{(2)}\|\right\|$ as well when $i\neq m$ but $m\in \Tilde{N}_i^{(2)}\cap \Tilde{N}_i^{(2)'}$ or $m\in \Tilde{N}_i^{(2)}\triangle \Tilde{N}_i^{(2)'}$. So, we get that, if $i\neq m$ but $m\in \bigcup_{j=1}^4 A_{ji}$, then inside event $\mathcal E_n$,
        $$\max_{i\in [n]}\left\|\|\bm v_i - \bm v_i'\|\right\|=O\left(\frac{d}{n^\alpha}\right).$$
        Also, by similar calculation, $\left\|\|\bm v_i^{(k)} - \bm v_i'^{(k)}\|\cdot\indic\{m\in A_{ji}\}\right\|_{L_r}=O(d/n^\alpha)$ holds for $k=1,2$ and $j=1,\ldots,4$. So, using Lemma~\ref{lambda close to x},
        \begin{align*}
            \bbE\left[\left\| \bm v_i\bm v_i^\top - \bm v_i'\bm v_i'^{\top}\right\|_{\mathrm{op}}^2\cdot\indic\left\{m\in \bigcup_{j=1}^4 A_{ij}\right\}\right]&\leq \left\|\|\bm v_i\|+\|\bm v_i'\|\right\|_{L_4}^2 \left\|\|\bm v_i - \bm v_i'\|\cdot\indic\left\{m\in \bigcup_{j=1}^4 A_{ij}\right\}\right\|_{L_4}^2\\
            &= O\left(d\cdot \frac{d^2}{n^{2\alpha}}\right)=O\left(\frac{d^3}{n^{2\alpha}}\right).
        \end{align*}

        Let $S_m$ denote the indices in $[n]$ such that $\|\bm v_i-\bm v_i'\|$ is nonzero. Combine all the  cases, and observe that $$S_m\setminus\{m\}\subseteq \bigcup_{k=1,2}\left(\Tilde{N}_i^{(k)}\cup \Tilde{N}_i^{(k)'}\right).$$
        So $|S_m\setminus \{m\}|$ is at most $\Theta(n^\alpha/t_n)$, thus we get inside event $\mathcal E_n$,
        \begin{align*}
            &\sum_{i=1}^n \|\bm v_i - \bm v_i'\| \leq O\left( \frac{d}{t_n} + \frac{n^\alpha}{t_n} \cdot \frac{d}{n^\alpha}\right)=O\left(\frac{d}{t_n}\right)\\
            \implies & \|F_n(Z) - F_n(Z^{(m)})\|_{op} = O\left(\frac{d^2}{nt_n^2}\right).
        \end{align*}
        Hence it follows that inside event $\mathcal E_n$, for all $m\in [n]$,
        \begin{equation}\label{term R}
            \|\Delta_m^F\|_{\mathrm{op}} = O\left(\frac{d^2}{nt_n^2}\right).
        \end{equation}

        On the other hand,
        \begin{align*}
            \sum_{m=1}^n \bbE\left[(\Delta_m^F)^2 | \mathcal{F}_{m-1}\right]  \preceq \left( \sum_{m=1}^n \bbE\left[\|\Delta_m^F\|_{\mathrm{op}}^2 \Bigr| \mathcal{F}_{m-1}\right] \right) \bm I_d.
        \end{align*}
        Let us bound an individual summand $\bbE\left[\|\Delta_m^F\|_{\mathrm{op}}^2 \Bigr| \mathcal{F}_{m-1}\right]$. Note that, using Cauchy-Schwarz inequality,
        \begin{align*}
            \bbE\left[\|\Delta_m^F\|_{\mathrm{op}}^2 \right] &= \bbE\left[\left\|\frac{1}{n}\sum_{i=1}^n \left(\bm v_i\bm v_i^\top - \bm v_i' \bm v_i'^\top\right)\right\|_{\mathrm{op}}^2 \right]\\
            &\leq \frac{1}{n^2}\left(2\bbE\left[\left\| \left(\bm v_m\bm v_m^\top - \bm v_m' \bm v_m'^\top\right)\right\|_{\mathrm{op}}^2 \right]+2\bbE\left[\left\|\sum_{i\in S_m\setminus\{m\}} \left(\bm v_i\bm v_i^\top - \bm v_i' \bm v_i'^\top\right)\right\|_{\mathrm{op}}^2 \right]\right).
        \end{align*}
        Now, again using Cauchy-Schwarz inequality,
        \begin{align*}
            \bbE\left[\left\|\sum_{i\in S_m\setminus\{m\}} \left(\bm v_i\bm v_i^\top - \bm v_i' \bm v_i'^\top\right)\right\|_{\mathrm{op}}^2 \Bigr| S_m\setminus \{m\} \right] &\leq \bbE\left[|S_m\setminus\{m\}|\sum_{i\in S_m\setminus\{m\}}\left\| \left(\bm v_i\bm v_i^\top - \bm v_i' \bm v_i'^\top\right)\right\|_{\mathrm{op}}^2 \Bigr| S_m\setminus \{m\} \right]\\
            &\leq |S_m\setminus \{m\}|^2\max_{i\in S_m\setminus\{m\}} \bbE\left\| \left(\bm v_i\bm v_i^\top - \bm v_i' \bm v_i'^\top\right)\right\|_{\mathrm{op}}^2\\
            &= O\left(|S_m\setminus \{m\}|^2\cdot \frac{d^3}{n^{2\alpha}}\right).
        \end{align*}
        Taking a further expectation recalling that $S_m\setminus\{m\}\subseteq \bigcup_{k=1,2}\left(\Tilde{N}_i^{(k)}\cup \Tilde{N}_i^{(k)'}\right)$, it follows using Lemma~\ref{inverse wij count moments},
        $$\bbE\left[\left\|\sum_{i\in S_m\setminus\{m\}} \left(\bm v_i\bm v_i^\top - \bm v_i' \bm v_i'^\top\right)\right\|_{\mathrm{op}}^2 \right] \leq O\left(\bbE\left(|\Tilde{N}_i\cup \Tilde{N}_i'|^2\right)\cdot\frac{d^3}{n^{2\alpha}}\right) = O\left(\frac{d^3}{t_n^2}\right).$$
        Plugging this in, we get
        $$\bbE\left[\|\Delta_m^F\|_{\mathrm{op}}^2 \right]\leq \frac{1}{n^2}\left(2\bbE\left[\left\| \left(\bm v_m\bm v_m^\top - \bm v_m' \bm v_m'^\top\right)\right\|_{\mathrm{op}}^2 \right]+2\bbE\left[\left\|\sum_{i\in S_m\setminus\{m\}} \left(\bm v_i\bm v_i^\top - \bm v_i' \bm v_i'^\top\right)\right\|_{\mathrm{op}}^2 \right]\right)=O\left(\frac{d^3}{n^2t_n^2}\right).$$
        Summing over $m$ yields
        $$\sum_{m=1}^n \bbE\left[\|\Delta_m^F\|_{\mathrm{op}}^2 \right]=O\left(\frac{d^3}{nt_n^2}\right).$$
        More specifically, let $C>0$ be a constant such that, using the above and Eq.~\ref{term R},
        $$\|\Delta_m^F\|_{\mathrm{op}} \leq C\frac{d^{2}}{nt_n^2} \quad \forall m\in[n] \text{ inside event $\mathcal E_n$}, \qquad \sum_{m=1}^n\bbE\left[\|\Delta_m^F\|^2 \right]\leq C\frac{d^3}{nt_n^2}.$$
        Then, using Markov's inequality we get
        $$\bbP\left( \sum_{m=1}^n \bbE\left[\|\Delta_m^F\|_{\mathrm{op}}^2 \Bigr| \mathcal{F}_{m-1}\right] \geq CK\cdot \frac{d^3}{nt_n^2}\right) \leq  \frac{\sum_{m=1}^n \bbE\left[\|\Delta_m^F\|_{\mathrm{op}}^2 \right]}{CK\cdot \frac{d^3}{nt_n^2}} \leq \frac{1}{K}.$$

        Then, by Theorem 1.2 in \cite{10.1214/ECP.v16-1624}, 
        \begin{align*}
            \bbP\left(\left\|\frac{1}{n}\bm V_n^\top\bm V_n - \bbE\left[\frac{1}{n}\bm V_n^\top\bm V_n\right]\right\|_{op} \geq a\right)&= \bbP\left( \left\{\left\|\frac{1}{n}\bm V_n^\top\bm V_n - \bbE\left[\frac{1}{n}\bm V_n^\top\bm V_n\right]\right\|_{op} \geq a\right\} \cap \mathcal E_n \right)+\bbP(\mathcal E_n^c) \\
            &\leq 2d\exp\left(\frac{-a^2/2}{CK\frac{d^3}{nt_n^2} + C\frac{ad^{2}}{3nt_n^2}}\right) + \bbP\left( \sum_{m=1}^n \bbE\left[\|\Delta_m^F\|_{\mathrm{op}}^2 \Bigr| \mathcal{F}_{m-1}\right] \geq CK\cdot \frac{d^3}{nt_n^2}\right) + o(1).
        \end{align*}

        Choose $a = c\,\frac{d^{3/2}\sqrt{\log d}}{t_n\sqrt{n}}$ for some $c>0$. Then,
        $$CK\frac{d^3}{nt_n^2} + C\frac{ad^{2}}{3nt_n^2} = CK\frac{d^3}{nt_n^2} + cC\cdot \frac{d^{3/2}\sqrt{\log d}}{t_n\sqrt{n}}\cdot \frac{d^{2}}{3nt_n^2} \leq C'\cdot \frac{d^3}{nt_n^2},$$
        for a constant $C'>0$ independent of $c$, since the first summand clearly dominates the second eventually. Hence,
        \begin{align*}
            2d\exp\left(\frac{-a^2/2}{CK\frac{d^3}{nt_n^2} + C\frac{ad^{2}}{3nt_n^2}}\right) + \bbP\left( \sum_{m=1}^n \bbE\left[\|\Delta_m^F\|_{\mathrm{op}}^2 \Bigr| \mathcal{F}_{m-1}\right] \geq CK\cdot \frac{d^3}{nt_n^2}\right) &\leq  2\exp\left( \log d - \frac{c^2d^3\log d/nt_n^2}{C'd^3/nt_n^2} \right) + \frac{1}{K}\\
            &\leq 2\exp\left(\log d - \frac{c^2\log d}{C'}\right) + \frac{1}{K}.
        \end{align*}
        Because the above can be made arbitrarily small by choosing $K, c$ appropriately, we get that
        $$\left\|\frac{1}{n}\bm V_n^\top\bm V_n - \bbE\left[\frac{1}{n}\bm V_n^\top\bm V_n\right]\right\|_{\mathrm{op}} = O_\bbP\left(\frac{d^{3/2}\sqrt{\log d}}{t_n\sqrt{n}}\right).$$
    \end{proof}
\end{lemma}

\begin{lemma}\label{the freedman lemma 2}
Assume that $n^\alpha\gg d\gg (\log n)^3$ and $\alpha>\gamma+\frac{\sigma_x^4(\alpha-1)}{2(\sigma_x^2+\sigma_\eta^2)^2}$. Then
$$\left\|\frac{1}{n}\bm X_n^\top\bm V_n - \bbE\left[\frac{1}{n}\bm X_n^\top\bm V_n\right]\right\|_{op}, \left\|\frac{1}{n}\bm V_n^\top\bm X_n - \bbE\left[\frac{1}{n}\bm V_n^\top\bm X_n\right]\right\|_{op} = O_\bbP\left(\frac{d^{3/2}\sqrt{\log d}}{t_n\sqrt{n}}\right).$$
    \begin{proof}
    
        Define the event
        $$\mathcal E_n = \left\{\max_{i\in[n]}\|\bm x_i\| = O(\sqrt{d}), \max_{i\in[n]}\|\bm \eta_i\| = O(\sqrt{d}), \max_{i\in[n],k=1,2}|\Tilde{N}_i^{(k)}| = \Theta(n^\alpha/t_n), \min_{i\in[n],k=1,2}|\Tilde{N}_i^{(k)}| = \Theta(n^\alpha/t_n)\right\}.$$
        By Lemma~\ref{norm x_i lemma} and Lemma~\ref{wij count}, 
        $\bbP(\mathcal E_n^c)\longrightarrow 0$. It follows that inside event $\mathcal E_n$,  $\max_{i\in [n]}\|\bm z_i\| = O(\sqrt{d})$. Also, inside event $\mathcal E_n$, for all $i\in [n]$,
        $$\|\bm \lambda_i^{(1)}\| = \frac{\sqrt{d}}{t_n|\Tilde{N}_i^{(1)}|}\left\|\sum_{j\in \Tilde{N}_i^{(1)}} \bm z_j^{(1)}\right\| \leq \frac{\sqrt{d}}{t_n}\max_{j}\|\bm z_j^{(1)}\| = O\left(\frac{d}{t_n}\right).$$
        Similarly, $\|\bm \lambda_i^{(2)}\| = O\left(\frac{d}{t_n}\right)$ for all $i\in [n]$, and thus $\max_{i\in[n]}\|\bm \lambda_i\| = O\left(\frac{d}{t_n}\right)$ inside event $\mathcal E_n$. Recalling that $\|\bm v_i\| \leq \|\bm x_i\| + \|\bm \lambda_i\|$, it also follows that $\max_{i\in [n]}\|\bm v_i\| = O\left(\frac{d}{t_n}\right)$ inside event $\mathcal E_n$.
        
        Let $e_{ij}$ be the Erd\H{o}s--R\'enyi edge indicator between $i$ and $j$. Define $M_i := \{\bm x_i, \bm \eta_i, \{e_{ij}:j>i\}\}$, and the natural filtration $\mathcal F_i := \sigma(M_1,M_2,\ldots,M_i)$ for all $i\in [n]$. Write $Z = (M_1,\ldots,M_n)$, and for all $m\in [n]$,
        $$Z^{(m)} := (M_1,\ldots,M_{m-1},M_m',M_{m+1},\ldots,M_n),$$
        where $M_m' = \{\bm x_m', \bm \eta_m', \{e'_{mj}:j>m\}\}$ is an i.i.d. copy of $M_m$, i.e. $\bm x_m'$ is an i.i.d. copy of $\bm x_m$, and $\bm \eta_m'$ is an i.i.d. copy of $\eta_m$, and $\{e'_{mj}:j>m\}\}$ are i.i.d. copies of $\{e_{mj}:j>m\}\}$. We are allowed to resample a subset of Erd\H{o}s--R\'enyi edges since all such edges are mutually independent and also independent from everything else. Let
        $$G_n(Z) = \frac{1}{n}\bm X_n^\top \bm V_n.$$
        Then,
        $$\Delta_m^G = \bbE[G_n(Z) - G_n(Z^{(m)}) | \mathcal F_m]$$
        are Doob's martingale increments. 
        
        Fix $m\in [n]$. After resampling $M_m' = \{\bm x_m', \bm \eta_m', \{e'_{mj}:j>m\}\}$, let us denote new copies of $\bm \lambda_i$ as $\bm \lambda_i'$, of $\bm v_i$ as $\bm v_i'$, of $\Tilde{N}_i^{(k)}$ as $\Tilde{N}_i^{(k)'}$, for all $i\in [n]$ and $k=1,2$. Observe that
        $$\|G_n(Z) - G_n(Z^{(m)})\|_{\mathrm{op}} = \left\|\frac{1}{n}\sum_{i=1}^n (\bm x_i \bm v_i^\top - \bm x_i'\bm v_i'^{\top})\right\|_{\mathrm{op}} \leq \frac{1}{n}\sum_{i=1}^n \left\|\bm x_i \bm v_i^\top - \bm x_i'\bm v_i'^{\top}\right\|_{\mathrm{op}}.$$
        
        Note that $\left\|\bm x_i \bm v_i^T - \bm x_i'\bm v_i'^{T}\right\|_{op}$
         is nonzero if either $i=m$ or $m$ is present in at least one of the old or new screened neighbourhoods of $i$. The second case can be written more precisely as $i\neq m$ but $m\in \bigcup_{j=1}^4 A_{ji}$, where
         $$A_{1i}=(\Tilde{N}_i^{(1)}\cap \Tilde{N}_i^{(1)'})\cup (\Tilde{N}_i^{(2)}\cap \Tilde{N}_i^{(2)'}), \quad A_{2i}=(\Tilde{N}_i^{(1)}\cap \Tilde{N}_i^{(1)'})\cup (\Tilde{N}_i^{(2)}\triangle \Tilde{N}_i^{(2)'}),$$
         $$A_{3i}=(\Tilde{N}_i^{(1)}\triangle \Tilde{N}_i^{(1)'})\cup (\Tilde{N}_i^{(2)}\cap \Tilde{N}_i^{(2)'}), \quad A_{4i}=(\Tilde{N}_i^{(1)}\triangle \Tilde{N}_i^{(1)'})\cup (\Tilde{N}_i^{(2)}\triangle \Tilde{N}_i^{(2)'}).$$
         
        \begin{itemize}
        \item[(i)] $i=m$\\[3mm]
        We have inside $\mathcal E_n$,
        $$\|\bm v_m - \bm v_m'\| \leq \|\bm v_m\| + \|\bm v_m'\| = O\left( \frac{d}{t_n} \right),$$
        and
        $$\|\bm x_m \bm v_m^\top - \bm x_m' \bm v_m'^{\top}\|_{op} \leq \|\bm x_m\|\|\bm v_m - \bm v_m'\| +\| \bm x_m - \bm x_m'\|\|\bm v_m'\| = O\left(\sqrt{d}\cdot\frac{d}{t_n}\right) = O\left(\frac{d^{3/2}}{t_n}\right).$$
        
        Then, we also have, from Lemma~\ref{lambda close to x},
        \begin{align*}
            \bbE\left(\|\bm x_m \bm v_m^\top - \bm x_m'\bm v_m'^\top\|_{\mathrm{op}}^2\right) &\leq 2\bbE\left(\|\bm x_m\|^2\|\bm v_m - \bm v_m'\|^2 +\| \bm x_m - \bm x_m'\|^2\|\bm v_m'\|^2\right)\\
            &\leq 2\left\|\|\bm x_m\|\right\|_{L_4}^2\left\|\|\bm v_m-\bm v_m'\|\right\|_{L_4}^2 + 2\left\|\|\bm v_m'\|\right\|_{L_4}^2 \left\|\|\bm x_m - \bm x_m'\|\right\|_{L_4}^2\\
            &\leq 2\left\|\|\bm x_m\|\right\|_{L_4}^2\left\|\|\bm v_m-\bm v_m'\|\right\|_{L_4}^2 + 4\left\|\|\bm v_m'\|\right\|_{L_4}^2 \left(\left\|\|\bm x_m\|\right\|_{L_4}^2 + \left\|\|\bm x_m'\|\right\|_{L_4}^2\right)\\
            & = O\left( d^2\right).
        \end{align*}
        
        \item[(ii)] $i\neq m$ but $m\in \Tilde{N}_i^{(1)}\cap \Tilde{N}_i^{(1)'}$\\[3mm]
        In this case, $\bm x_i = \bm x_i'$ and $\Tilde{N}_i^{(1)} = \Tilde{N}_i^{(1)'}$, so
        \begin{align*}
            \|\bm v_i^{(1)} - \bm v_i'^{(1)}\| = \|\bm \lambda_i^{(1)} - \bm \lambda_i'^{(1)}\| &= \frac{\sqrt{d}}{t_n}\left\|\frac{\sum_{j\in \Tilde{N}_i^{(1)}}\bm z_j^{(1)}}{|\Tilde{N}_i^{(1)}|} - \frac{\sum_{j\in \Tilde{N}_i^{(1)'}}\bm z_j'^{(1)}}{|\Tilde{N}_i^{(1)'}|}\right\|= \frac{\sqrt{d}}{t_n}\left\|\frac{\bm z_m^{(1)} - \bm z_m'^{(1)}}{|\Tilde{N}_i^{(1)}|}\right\|. 
        \end{align*}
        So, inside event $\mathcal E_n$,
        $$\|\bm v_i^{(1)} - \bm v_i'^{(1)}\| =\frac{\sqrt{d}}{t_n} \cdot \frac{O(\sqrt{d})}{\Theta\left(\frac{n^\alpha}{t_n}\right)} = O\left( \frac{d}{n^\alpha} \right).$$

        Also, for any fixed $r\geq 1$,
        $$\left\|\|\bm v_i^{(1)} - \bm v_i'^{(1)}\|\cdot\indic\{m\in \Tilde{N}_i^{(1)}\cap\Tilde{N}_i^{(1)'}\}\right\|_{L_r} \leq \frac{\sqrt{d}}{t_n}\left\|\bm z_m^{(1)} - \bm z_m'^{(1)}\right\|_{L_{2r}}\left\|\frac{1}{|\Tilde{N}_i^{(1)}|}\right\|_{L_{2r}} = O\left(\frac{\sqrt{d}}{t_n}\cdot\sqrt{d}\cdot \frac{t_n}{n^\alpha}\right)=O\left( \frac{d}{n^\alpha} \right),$$
        where the bounds follow from using Holder's inequality with Lemma~\ref{inverse wij count moments} and Lemma~\ref{wij count}.

        \item[(iii)] $i\neq m$ but $m \in \Tilde{N}_i^{(1)}\triangle \Tilde{N}_i^{(1)'}$\\[3mm]
        In this case $\bm x_i = \bm x_i'$. Also since only $M_m$ has been replaced by its independent copy $M_m'$, we must have $\Tilde{N}_i^{(1)}\triangle \Tilde{N}_i^{(1)'}=\{m\}$. Let us assume without loss of generality, $ \Tilde{N}_i^{(1)} \setminus \Tilde{N}_i'^{(1)} = \{m\}$. 
        Then,
        \begin{align*}
            \|\bm v_i^{(1)} - \bm v_i'^{(1)}\|=\|\bm \lambda_i^{(1)} - \bm \lambda_i'^{(1)}\| &= \frac{\sqrt{d}}{t_n}\left\| \frac{\sum_{j\in \Tilde{N}_i^{(1)}}\bm z_j^{(1)}}{|\Tilde{N}_i^{(1)}|} - \frac{\sum_{j\in \Tilde{N}_i^{(1)'}}\bm z_j'^{(1)}}{|\Tilde{N}_i^{(1)'}|} \right\|\\
                &= \frac{\sqrt{d}}{t_n}\cdot\frac{\left\| |\Tilde{N}_i^{(1)'}|\sum_{j\in \Tilde{N}_i^{(1)}}\bm z_j^{(1)} - |\Tilde{N}_i^{(1)}|\sum_{j\in \Tilde{N}_i^{(1)'}}\bm z_j'^{(1)} \right\|}{|\Tilde{N}_i^{(1)}||\Tilde{N}_i^{(1)'}|}\\
                &\leq \frac{\sqrt{d}}{t_n}\cdot\frac{\left(\sum_{j\in \Tilde{N}_i^{(1)}\setminus \{m\}}\left\| \bm z_j^{(1)}\right\|\right) + |\Tilde{N}_i^{(1)'}|\|\bm z_m^{(1)} \|}{|\Tilde{N}_i^{(1)}||\Tilde{N}_i^{(1)'}|}.,
        \end{align*}
        Then, inside event $\mathcal E_n$,
        $$\|\bm v_i^{(1)} - \bm v_i'^{(1)}\|\leq \frac{\sqrt{d}}{t_n}\left( \frac{O\left(\frac{n^\alpha}{t_n}\sqrt{d}\right)}{\Theta\left( \frac{n^\alpha}{t_n}\right)^2} \right) = O\left(\frac{d}{n^\alpha}\right).$$
        Also, for any fixed $r\geq 1$,
        \begin{align*}
            \left\|\|\bm v_i^{(1)} - \bm v_i'^{(1)}\|\cdot\indic\{m\in \Tilde{N}_i^{(1)}\triangle\Tilde{N}_i^{(1)'}\}\right\|_{L_r} &\leq \frac{\sqrt{d}}{t_n}\left\|\left(\sum_{j\in \Tilde{N}_i^{(1)}\setminus \{m\}}\left\| \bm z_j^{(1)}\right\|\right) + |\Tilde{N}_i^{(1)'}|\|\bm z_m^{(1)} \|\right\|_{L_{2r}}\left\|\frac{1}{|\Tilde{N}_i^{(1)}|}\right\|_{L_{4r}}\left\|\frac{1}{|\Tilde{N}_i^{(1)'}|}\right\|_{L_{4r}} \\
            &= \frac{\sqrt{d}}{t_n}O\left( \frac{n^\alpha}{t_n}\cdot \sqrt{d}\right)O\left(\left(\frac{t_n}{n^\alpha}\right)^2\right) = O\left(\frac{d}{n^\alpha}\right),
        \end{align*}
        where the bounds follow by repeated application of Holder's inequality and Lemma~\ref{inverse wij count moments} and Lemma~\ref{wij count}.
        \end{itemize}

        Observe that same bounds hold for $\|\bm v_i^{(2)} - \bm v_i'^{(2)}\|$ as well when  $i\neq m$ but $m\in \Tilde{N}_i^{(2)}\cap \Tilde{N}_i^{(2)'}$ or $m\in \Tilde{N}_i^{(2)}\triangle \Tilde{N}_i^{(2)'}$. So, we get that, if $i\neq m$ but $i\in \bigcup_{j=1}^4 A_{ji}$, then inside event $\mathcal E_n$,
        $$\max_{i\in [n]}\|\bm x_i\bm v_i^\top - \bm x_i'\bm v_i'^\top\|_{op} \leq \|\bm x_i\|\|\bm v_i - \bm v_i'\|=O\left(\frac{d^{3/2}}{n^\alpha}\right).$$
        Also, by similar calculation,  $\left\|\|\bm v_i^{(k)} - \bm v_i'^{(k)}\|\cdot\indic\{m\in A_{ji}\}\right\|_{L_r} = O(d/n^\alpha)$ holds for $k=1,2$ and $j=1,\ldots,4$. So,
        $$\bbE\left[\left\| \bm x_i\bm v_i^\top - \bm x_i'\bm v_i'^{\top}\right\|_{\mathrm{op}}^2\cdot\indic\left\{m\in \bigcup_{j=1}^4 A_{ji}\right\}\right]\leq \left\|\|\bm x_i\|\right\|_{L_4}^2 \left\|\|\bm v_i - \bm v_i'\|\cdot\indic\left\{m\in \bigcup_{j=1}^4 A_{ji}\right\}\right\|_{L_4}^2 = O\left(d\cdot \frac{d^2}{n^{2\alpha}}\right)=O\left(\frac{d^3}{n^{2\alpha}}\right).$$

        Let $S_m$ denote the indices in $[n]$ such that $\|\bm x_i \bm v_i^\top- \bm x_i' \bm v_i'^\top\|_{\mathrm{op}}$ is nonzero. Combine all the  cases, and observe that $$S_m\setminus\{m\}\subseteq \bigcup_{k=1,2}\left(\Tilde{N}_i^{(k)}\cup \Tilde{N}_i^{(k)'}\right).$$
        So $|S_m\setminus \{m\}|$ is at most $\Theta(n^\alpha/t_n)$, thus we get inside event $\mathcal E_n$,
        \begin{align*}
             \|G_n(Z) - G_n(Z^{(m)})\|_{\mathrm{op}} = O\left(\frac{d^{3/2}}{nt_n} + \frac{n^\alpha}{nt_n}\cdot\frac{d^{3/2}}{n^\alpha}\right) = O\left(\frac{d^{3/2}}{nt_n}\right).
        \end{align*}
        Hence it follows that inside event $\mathcal E_n$, for all $m\in [n]$,
        \begin{equation}\label{term R 2}
            \|\Delta_m^G\|_{\mathrm{op}} = O\left(\frac{d^{3/2}}{nt_n}\right).
        \end{equation}

        On the other hand,
        \begin{align*}
            \sum_{m=1}^n \bbE\left[(\Delta_m^G)^2 | \mathcal{F}_{m-1}\right]  \preceq \left( \sum_{m=1}^n \bbE\left[\|\Delta_m^G\|_{\mathrm{op}}^2 \Bigr| \mathcal{F}_{m-1}\right] \right) \bm I_d.
        \end{align*}
        Let us bound an individual summand $\bbE\left[\|\Delta_m^G\|_{\mathrm{op}}^2 \Bigr| \mathcal{F}_{m-1}\right]$. Note that, using Cauchy-Schwarz inequality,
        \begin{align*}
            \bbE\left[\|\Delta_m^G\|_{\mathrm{op}}^2 \right] &= \bbE\left[\left\|\frac{1}{n}\sum_{i=1}^n \left(\bm x_i\bm v_i^\top - \bm x_i' \bm v_i'^\top\right)\right\|_{\mathrm{op}}^2 \right]\\
            &\leq \frac{1}{n^2}\left(2\bbE\left[\left\| \left(\bm x_m\bm v_m^\top - \bm x_m' \bm v_m'^T\right)\right\|_{\mathrm{op}}^2 \right]+2\bbE\left[\left\|\sum_{i\in S_m\setminus\{m\}} \left(\bm x_i\bm v_i^\top - \bm x_i' \bm v_i'^\top\right)\right\|_{\mathrm{op}}^2 \right]\right).
        \end{align*}
        Now, again using Cauchy-Schwarz inequality,
        \begin{align*}
            \bbE\left[\left\|\sum_{i\in S_m\setminus\{m\}} \left(\bm x_i\bm v_i^\top - \bm x_i' \bm v_i'^\top\right)\right\|_{\mathrm{op}}^2 \Bigr| S_m\setminus \{m\} \right] &\leq \bbE\left[|S_m\setminus\{m\}|\sum_{i\in S_m\setminus\{m\}}\left\| \left(\bm x_i\bm v_i^\top - \bm x_i' \bm v_i'^\top\right)\right\|_{\mathrm{op}}^2 \Bigr| S_m\setminus \{m\} \right]\\
            &\leq |S_m\setminus \{m\}|^2\max_{i\in S_m\setminus\{m\}} \bbE\left\| \left(\bm x_i\bm v_i^\top - \bm x_i' \bm v_i'^\top\right)\right\|_{\mathrm{op}}^2\\
            &= O\left(|S_m\setminus \{m\}|^2\cdot \frac{d^3}{n^{2\alpha}}\right).
        \end{align*}
        Taking a further expectation recalling that $S_m\setminus\{m\}\subseteq \bigcup_{k=1,2}\left(\Tilde{N}_i^{(k)}\cup \Tilde{N}_i^{(k)'}\right)$, it follows using Lemma~\ref{inverse wij count moments},
        $$\bbE\left[\left\|\sum_{i\in S_m\setminus\{m\}} \left(\bm x_i\bm v_i^\top - \bm x_i' \bm v_i'^\top\right)\right\|_{\mathrm{op}}^2 \right] \leq O\left(\bbE\left(|\Tilde{N}_i\cup \Tilde{N}_i'|^2\right)\cdot\frac{d^3}{n^{2\alpha}}\right) = O\left(\frac{d^3}{t_n^2}\right).$$
        Plugging this in, we get
        $$\bbE\left[\|\Delta_m^G\|_{\mathrm{op}}^2 \right]\leq \frac{1}{n^2}\left(2\bbE\left[\left\| \left(\bm x_m\bm v_m^\top - \bm x_m' \bm v_m'^\top\right)\right\|_{\mathrm{op}}^2 \right]+2\bbE\left[\left\|\sum_{i\in S_m\setminus\{m\}} \left(\bm x_i\bm v_i^\top - \bm x_i' \bm v_i'^\top\right)\right\|_{\mathrm{op}}^2 \right]\right)=O\left(\frac{d^3}{n^2t_n^2}\right).$$
        Summing over $m$, we get 
        $$\sum_{m=1}^n \bbE\left[\|\Delta_m^G\|_{\mathrm{op}}^2 \right]=O\left(\frac{d^3}{nt_n^2}\right).$$
        More specifically, let $C>0$ be a constant such that, using the above and Eq.~\ref{term R 2},
        $$\|\Delta_m^G\|_{\mathrm{op}} \leq C\frac{d^{3/2}}{nt_n} \quad  \forall m\in[n] \text{ inside event $\mathcal E_n$}, \qquad \sum_{m=1}^n\bbE\left[\|\Delta_m^G\|_{\mathrm{op}}^2 \right]\leq C\frac{d^3}{nt_n^2}.$$
        The martingale increments $\Delta_m^G$ are not necessarily symmetric, so construct
        $$\Tilde{\Delta}_m^G = \begin{bmatrix}
            \bm 0 & \Delta_m^G \\ (\Delta_m^G)^\top & \bm 0
        \end{bmatrix},$$
        and observe that 
        $$\|\Tilde{\Delta}_m^G\|_{\mathrm{op}} = \|{\Delta}_m^G\|_{\mathrm{op}}\leq C\frac{d^{3/2}}{nt_n} \quad  \forall m\in[n]\text{ inside event $\mathcal E_n$}, \qquad \sum_{m=1}^n\bbE\left[\|\Tilde{\Delta}_m^G\|_{\mathrm{op}}^2 \right]=\sum_{m=1}^n\bbE\left[\|\Delta_m^G\|_{\mathrm{op}}^2 \right]\leq C\frac{d^3}{nt_n^2}.$$
        Also, observe that
        $$\sum_{m=1}^n \Tilde{\Delta}_m^G = \begin{bmatrix}
            \bm 0 & \frac{1}{n}\bm X_n^\top \bm V_n - \frac{1}{n}\bbE\left[\bm X_n^\top \bm V_n\right]\\ \frac{1}{n}\bm V_n^\top \bm X_n - \frac{1}{n}\bbE\left[\bm V_n^\top \bm X_n\right] & \bm 0
        \end{bmatrix}\implies \left\| \sum_{m=1}^n \Tilde{\Delta}_m^G \right\|_{\mathrm{op}} = \left\|\frac{1}{n}\bm X_n^\top \bm V_n - \frac{1}{n}\bbE\left[\bm X_n^\top \bm V_n\right]\right\|_{\mathrm{op}}.$$
        
        Then, using Markov's inequality we get,
        $$\bbP\left( \sum_{m=1}^n \bbE\left[\|\Tilde{\Delta}_m^G\|_{\mathrm{op}}^2 \Bigr| \mathcal{F}_{m-1}\right] \geq CK\cdot \frac{d^3}{nt_n^2}\right) = \bbP\left( \sum_{m=1}^n \bbE\left[\|{\Delta}_m^G\|_{\mathrm{op}}^2 \Bigr| \mathcal{F}_{m-1}\right] \geq CK\cdot \frac{d^3}{nt_n^2}\right) \leq  \frac{\sum_{m=1}^n \bbE\left[\|\Delta_m^G\|_{\mathrm{op}}^2 \right]}{CK\cdot \frac{d^3}{nt_n^2}} \leq \frac{1}{K}.$$

        Then, by Theorem 1.2 in \cite{10.1214/ECP.v16-1624}, 
        \begin{align*}
            \bbP\left( \left\|\frac{1}{n}\bm X_n^\top \bm V_n - \bbE\left[\frac{1}{n}\bm X_n^\top \bm V_n\right]\right\|_{\mathrm{op}} \geq a \right) &=  \bbP\left( \left\{\left\|\sum_{m=1}^n \Tilde{\Delta}_m^G\right\|_{\mathrm{op}} \geq a\right\} \cap \mathcal E_n \right) + \bbP(\mathcal E_n^c)\\
            &\leq 2d\exp\left(\frac{-a^2/2}{CK\frac{d^3}{nt_n^2} + C\frac{ad^{3/2}}{3nt_n}}\right)+\bbP\left( \sum_{m=1}^n \bbE\left[\|\Tilde{\Delta}_m^G\|_{\mathrm{op}}^2 \Bigr| \mathcal{F}_{m-1}\right] \geq CK\cdot \frac{d^3}{nt_n^2}\right) +o(1).
        \end{align*}

        Choose $a = c\,\frac{d^{3/2}\sqrt{\log d}}{t_n\sqrt{n}}$ for some $c>0$. Then,
        $$CK\frac{d^3}{nt_n^2} + C\frac{ad^{2}}{3nt_n^2} = CK\frac{d^3}{nt_n^2} + cC\cdot \frac{d^{3/2}\sqrt{\log d}}{t_n\sqrt{n}}\cdot \frac{d^{3/2}}{3nt_n} \leq C'\cdot \frac{d^3}{nt_n^2},$$
        for a constant $C'>0$ independent of $c$, since the first summand clearly dominates the second eventually. Hence,
        \begin{align*}
            2d\exp\left(\frac{-a^2/2}{CK\frac{d^3}{nt_n^2} + C\frac{ad^{3/2}}{3nt_n}}\right) + \bbP\left( \sum_{m=1}^n \bbE\left[\|\Tilde{\Delta}_m^G\|_{\mathrm{op}}^2 \Bigr| \mathcal{F}_{m-1}\right] \geq CK\cdot \frac{d^3}{nt_n^2}\right) &\leq   2\exp\left( \log d - \frac{c^2d^3\log d/nt_n^2}{C'd^3/nt_n^2} \right)+\frac{1}{K}\\
            &\leq   2\exp\left(\log d - \frac{c^2\log d}{C'}\right)+\frac{1}{K}.
        \end{align*}
        Because the above can be made arbitrarily small by choosing $K, c$ appropriately, we get that
        $$\left\|\frac{1}{n}\bm X_n^\top\bm V_n - \bbE\left[\frac{1}{n}\bm X_n^\top\bm V_n\right]\right\|_{\mathrm{op}} = O_\bbP\left(\frac{d^{3/2}\sqrt{\log d}}{t_n\sqrt{n}}\right).$$

        A similar proof works for $\left\|\frac{1}{n}\bm V_n^\top\bm X_n - \bbE\left[\frac{1}{n}\bm V_n^\top\bm X_n\right]\right\|_{\mathrm{op}}$.
    \end{proof}
\end{lemma}

\begin{lemma}\label{lambda t u}
Assume that $n\gg d \gg (\log n)^3$. Then
$$\left\|\frac{1}{n}\bm V_n^\top \bm E_n\right\| = O_\bbP\left( \sqrt{\frac{d}{n}}\sqrt{{\sqrt{\frac{d^3\log d}{nt_n^2}}}+\text{err}_{n,d}^2} \right),$$
where
$$\text{err}_{n,d} = n^{\gamma+\frac{\sigma_x^4(\alpha-1)}{2(\sigma_x^2+\sigma_\eta^2)^2}-\alpha}+\frac{1}{t_n}+\frac{t_n^3}{\sqrt{d}}+\frac{\sqrt{d}}{\sqrt{t_nn^\alpha}}.$$
    \begin{proof}
        Note that $\bm V_n^\top \bm E_n | \bm V_n \sim N(\bm 0, \sigma_\varepsilon^2\bm V_n^\top \bm V_n)$. So, letting $\bm \zeta \sim N(\bm 0, \sigma_\varepsilon^2\bm I_d) \indep \bm V_n$,
        $$\bm V_n^\top \bm E_n \stackrel{d}{=} (\bm V_n^\top \bm V_n)^{1/2}\bm \zeta$$
        $$\implies \left\| \bm V_n^\top \bm E_n \right\| \stackrel{d}{=} \left\|(\bm V_n^\top \bm V_n)^{1/2}\bm \zeta \right\|.$$
        We shall bound the term on the right hand side. We have
        $$\left\|(\bm V_n^\top \bm V_n)^{1/2}\bm \zeta\right\| \leq \left\|(\bm V_n^\top \bm V_n)^{1/2}\right\|_{\mathrm{op}}\|\bm \zeta\|.$$
        By Lemma~\ref{norm x_i lemma}, we have $\|\bm \zeta \|=O_\bbP(\sqrt{d})$. On the other hand,
        \begin{align*}
            \left\|(\bm V_n^\top \bm V_n)^{1/2}\right\|_{\mathrm{op}} &= \sqrt{\left\|(\bm V_n^\top \bm V_n)\right\|_{\mathrm{op}}}\\
            &\leq \sqrt{\left\|\bm V_n^\top\bm V_n - \bbE(\bm V_n^\top \bm V_n)\right\|_{\mathrm{op}} + \left\| \bbE(\bm V_n^\top\bm V_n) \right\|_{\mathrm{op}}}\\
            &= \sqrt{O_\bbP\left(\frac{d^{3/2}\sqrt{n}\sqrt{\log d}}{t_n}+n\,\text{err}_{n,d}^2\right)},
        \end{align*}
        where it follows from Lemma~\ref{R diagonal},
        $$\text{err}_{n,d} = n^{\gamma+\frac{\alpha-1}{2(1+\sigma^2)^2}-\alpha}+\frac{1}{t_n}+\frac{t_n^3}{\sqrt{d}}+\frac{\sqrt{d}}{\sqrt{t_nn^\alpha}}.$$
        Thus, we have
        \begin{align*}
            \left\|\frac{1}{n}\bm V_n^\top \bm E_n\right\| = O_\bbP\left( \sqrt{\frac{d}{n}}\sqrt{{\sqrt{\frac{d^3\log d}{nt_n^2}}}+\text{err}_{n,d}^2} \right).
        \end{align*}
    \end{proof}
\end{lemma}

\section{Lemmas required for proof of Theorem~\ref{non-attn higher MSE}}

The main motivation for Theorem~\ref{non-attn higher MSE} is that when the average Erd\H{o}s--R\'enyi degree of a node is much higher than its geometric degree, i.e. $\alpha
<\gamma$, the averaged message from the neighbours in each layer of the non-attention-based network, i.e. the $M_1^{(\ell)}\frac{1}{|N_i|}\sum_{j\in N_i} \bm \xi_j^{(\ell)}$ term is insignificant. Consequently, the output $\bm \xi_i^{(L)}$ of the $L$-layer network is very close to a function only based on $\bm z_i$. Building on Lemma~\ref{mean of lipschitz of y is O(tn)} and Lemma~\ref{vanishing msg lemma}, that is exactly what Lemma~\ref{GNN output close to non-GNN} proves.

\begin{lemma}\label{mean of lipschitz of y is O(tn)}
Let $g:\bbR^d \mapsto \bbR^d$ be an odd, $C$-Lipschitz function where $C>0$ is a global constant. Then, we have uniformly in $i,j\in [n]$,
$$\left\|\bbE\left(g(\bm z_j)| \bm x_i^\top\bm x_j\geq \sigma_x^2 t_n\sqrt{d}\right)\right\| = O(t_n),$$
$$\bbE\left[\|g(\bm z_j)\|^2| \bm x_i^\top\bm x_j\geq \sigma_x^2t_n\sqrt{d}\right] = O\left(d+t_n^2\right).$$
    \begin{proof}
        Define the random unit vector $\bm u = \frac{\bm x_i}{\|\bm x_i\|}$. Then, we can write
        $$\bm z_j = S\bm u + R, \qquad S:=\frac{\bm x_i^\top \bm x_j}{\|\bm x_i\|}\Bigr|\bm x_i\sim N(0,\sigma_x^2),\quad R:=W+\bm \eta_j \indep S,$$
        where $W|\bm x_i \sim N(\bm 0, \bm \sigma_x^2(I_d - \bm u\bm u^\top)), \bm \eta_j\sim N(\bm 0, \sigma_\eta^2\bm I_d)$. Also let the random variable $S_t$ have the conditional distribution given $\bm x_i$,
        $$S_t\Bigr|\bm x_i \stackrel{d}{=} S\Bigr|\bm x_i,\left\{ S\geq \frac{\sigma_x^2t_n\sqrt{d}}{\|\bm x_i\|}\right\}.$$
        Then,
        $$\bm z_j| \{\bm x_i^\top\bm x_j\geq \sigma_x^2 t_n\sqrt{d}\}\,\stackrel{d}{=}\,S_t\bm u + R.$$
        Let $S' \sim N(0,\sigma_x^2)$, independently of $\bm x_i,R$. Then, by oddness of $g$,
        $$\bbE[g(S'u+R)]=0.$$
        Then, we get
        \begin{align*}
            \left\|\bbE\left(g(\bm z_j)| \bm x_i^\top \bm x_j\geq \sigma_x^2 t_n\sqrt{d}\right)\right\| &= \left\|\bbE\left(g(S_t\bm u+R)\right)-\bbE\left(g(S'\bm u+R)\right)\right\| \\
            &\leq \bbE\left\|g(S_t\bm u+R)-g(S'\bm u+R)\right\| \leq C\bbE\left|S_t-S'\right|,
        \end{align*}
        where the last step follows by using the $C$-Lipschitz property of $g$. Now, note that
        \begin{align*}
            \bbE\left|S_t-S'\right|\leq & \bbE|S_t| + \bbE|S'|
            \leq  \bbE\left|\frac{\bm x_i^\top\bm x_j}{\|\bm x_i\|}\Bigr| \frac{\bm x_i^\top\bm x_j}{\|\bm x_i\|}\geq \frac{\sigma_x^2 t_n\sqrt{d}}{\|\bm x_i\|}\right| + \sigma_x^2\sqrt{\frac{2}{\pi}}\\
            \leq & \bbE\left[\frac{\varphi\left(t_n\sqrt{d}/\|\bm x_i\|\right)}{\Phi^c\left(t_n\sqrt{d}/\|\bm x_i\|\right)}\right] + \sqrt{\frac{2}{\pi}}
            \leq  \bbE\left[\frac{t_n\sqrt{d}}{\|\bm x_i\|} + \frac{\|\bm x_i\|}{t_n\sqrt{d}}\right] + \sqrt{\frac{2}{\pi}}
            \leq  ct_n,
        \end{align*}
        for a uniform constant $c>0$, where the penultimate step follows by Mill's ratio bound and the final step follows using Lemma~\ref{q norm of chi sq}, Lemma~\ref{q norm of inverted chi sq} and that $t_n \longrightarrow \infty$ as $n\longrightarrow \infty$. Thus, we conclude that uniformly in $i,j\in[n]$,
        $$\left\|\bbE\left(g(\bm z_j)| \bm x_i^\top\bm x_j\geq \sigma_x^2 t_n\sqrt{d}\right)\right\| = O(t_n).$$

        \medskip
        
        For the second part, first note that
        \begin{align*}
            \bbE\left[\|\bm z_j\|^2| \bm x_i^\top \bm x_j\geq \sigma_x^2 t_n\sqrt{d}\right] = \bbE\left[S_t^2\right] + \bbE\|R\|^2 &= \operatorname{Var}(S_t) + \left(\bbE(S_t)\right)^2 + \bbE\|W\|^2 + \bbE\|\bm \eta_j\|^2\\
            &\leq \sigma_x^2 + c^2 t_n^2 + (d-1)\sigma_x^2 + \sigma_\eta^2 d.
        \end{align*}
        In the above calculation, cross terms vanish by orthogonality of $\bm u$ and $W$, and independence of $\bm \eta_j$ from all other terms. We also use that $\operatorname{Var}(S_t)\leq \operatorname{Var}(S) = \sigma_x^2$ and as proved above, $\bbE|S_t|\leq ct_n$. Then, recalling that $g$ is $C$-Lipschitz,
        $$\bbE\left[\|g(\bm z_j)\|^2| \bm x_i^\top\bm x_j\geq \sigma_x^2t_n\sqrt{d}\right] \leq C^2\bbE\left[\|\bm z_j\|^2| \bm x_i^\top\bm x_j\geq \sigma_x^2 t_n\sqrt{d}\right] = O\left(d+t_n^2\right).$$
    \end{proof}
\end{lemma}

\begin{lemma}\label{vanishing msg lemma}
    Assume that $\log n \ll d \ll n^\alpha$ and $\alpha<\gamma$. Let $g:\bbR^d \mapsto \bbR^d$ be any odd, $C$-Lipschitz function such that $g(\bm 0)=\bm 0$. Then, uniformly in $i\in [n]$,
    $$\bbE\left\|\frac{1}{|N_i|}\sum_{j\in N_i} g(\bm y_j)\right\|^2=O\left(n^{2(\alpha-\gamma)}+{d}n^{-\gamma}\right) .$$
    \begin{proof}
        Recall that $\bm z_j \sim N(\bm 0, (\sigma_x^2+\sigma_\eta^2)\bm I_d)$ and $g$ is an odd, $C$-Lipschitz function. So $\{g(\bm z_j):j\in N_i^{(ER)}\}$ are i.i.d. centered sub-gaussians with $\psi_2$ norm $O\left(C\sqrt{\sigma_x^2+\sigma_\eta^2}\right)$. So, leveraging independence of Erd\H{o}s--R\'enyi edges, and using Holder's inequality with Lemma~\ref{er ngbhd size} and Lemma~\ref{pure geo ngbhd size},
        \begin{equation}\label{er part}
            \bbE\left\|\frac{1}{|N_i|}\sum_{j\in N_i^{(ER)}}g(\bm z_j)\right\|^2 = O\left(\frac{C^2(\sigma_x^2+\sigma_\eta^2)d}{n^\gamma}\right).
        \end{equation}
        On the other hand,
        \begin{align*}
            &\frac{1}{|N_i|}\sum_{j\in N_i\setminus N_i^{(ER)}} g(\bm z_j) = \frac{|N_i\setminus N_i^{(ER)}|}{|N_i|}\cdot \frac{1}{|N_i\setminus N_i^{(ER)}|}\sum_{j\in N_i\setminus N_i^{(ER)}} g(\bm z_j),
        \end{align*}
        so conditionally given $N_i, N_i^{ER}$,
        \begin{align*}
            &\bbE\left[\left\|\frac{1}{|N_i|}\sum_{j\in N_i\setminus N_i^{(ER)}} g(\bm z_j)\right\|^2 \Bigr| N_i,N_i^{(ER)}\right] = \left(\frac{|N_i\setminus N_i^{(ER)}|}{|N_i|}\right)^2\bbE\left[\left\|\frac{1}{|N_i\setminus N_i^{(ER)}|}\sum_{j\in N_i\setminus N_i^{(ER)}} g(\bm z_j)\right\|^2 \Bigr| N_i,N_i^{(ER)}\right]\\
            =& \left(\frac{|N_i\setminus N_i^{(ER)}|}{|N_i|}\right)^2\left\|\bbE\left(g(\bm z_j)\Bigr|\bm x_i^\top\bm x_j\geq \sigma_x^2 t_n\sqrt{d}\right)\right\|^2 \\  &\qquad\qquad\quad+\left(\frac{|N_i\setminus N_i^{(ER)}|}{|N_i|}\right)^2\bbE\left(\left\|\frac{1}{|N_i\setminus N_i^{(ER)}|}\sum_{j\in N_i\setminus N_i^{(ER)}} g(\bm z_j)-\bbE\left(g(\bm z_j)\Bigr|\bm x_i^\top\bm x_j\geq \sigma_x^2t_n\sqrt{d}\right)\right\|^2\Bigr|N_i,N_i^{(ER)}\right).
        \end{align*}
        We have using Lemma~\ref{mean of lipschitz of y is O(tn)}, $\left\|\bbE\left(g(\bm y_j)\Bigr|\bm x_i^T\bm x_j\geq t_n\sqrt{d}\right)\right\|^2 = O(t_n^2)$, and
        \begin{align*}
            &\bbE\left(\left\|\frac{1}{|N_i\setminus N_i^{(ER)}|}\sum_{j\in N_i\setminus N_i^{(ER)}} g(\bm z_j)-\bbE\left(g(\bm z_j)\Bigr|\bm x_i^\top\bm x_j\geq \sigma_x^2 t_n\sqrt{d}\right)\right\|^2\Bigr|N_i,N_i^{(ER)}\right)\\
            =& \frac{1}{|N_i\setminus N_i^{(ER)}|}\,\bbE\left(\left\|g(\bm z_j) - \bbE\left(g(\bm z_j)\Bigr|\bm x_i^\top \bm x_j\geq \sigma_x^2 t_n\sqrt{d}\right)\right\|^2\Bigr| \bm x_i^\top\bm x_j\geq \sigma_x^2 t_n\sqrt{d}\right)\\
            \leq& \frac{2}{|N_i\setminus N_i^{(ER)}|}\left(\bbE\left[\|g(\bm z_j)\|^2\Bigr| \bm x_i^\top\bm x_j \geq \sigma_x^2 t_n\sqrt{d}\right] + \left\|\bbE\left[g(\bm z_j)\Bigr| \bm x_i^\top \bm x_j\geq \sigma_x^2 t_n\sqrt{d}\right]\right\|^2\right)=\frac{2}{|N_i\setminus N_i^{(ER)}|}O\left(d+t_n^2\right).
        \end{align*}
        Plugging it back in we get 
        \begin{align*}
            \bbE\left[\left\|\frac{1}{|N_i|}\sum_{j\in N_i\setminus N_i^{(ER)}} g(\bm z_j)\right\|^2 \Bigr| N_i,N_i^{(ER)}\right] &= \left(\frac{|N_i\setminus N_i^{(ER)}|}{|N_i|}\right)^2\left(O(t_n^2) + \frac{2}{|N_i\setminus N_i^{(ER)}|}O\left(d+t_n^2\right)\right).
        \end{align*}
        Using Lemma~\ref{er ngbhd size}, Lemma~\ref{pure geo ngbhd size} and Lemma~\ref{inverse wij count moments}, we get
        \begin{equation}\label{geo part}
            \bbE\left[\left\|\frac{1}{|N_i|}\sum_{j\in N_i\setminus N_i^{(ER)}} g(\bm z_j)\right\|^2 \right] = O\left(\frac{n^{2\alpha} }{t_n^2n^{2\gamma}}\cdot t_n^2 + \frac{n^\alpha}{t_nn^{2\gamma}} (d+t_n^2)\right) = O\left(n^{2(\alpha-\gamma)} + \frac{d}{t_n}n^{\alpha-2\gamma}\right) = O(n^{2(\alpha-\gamma)}),
        \end{equation}
        because by assumption $d\gg t_n^2$ and $n^\alpha\gg d$. Finally, combining Eq.~\ref{er part} and Eq.~\ref{geo part} with a Cauchy-Schwarz inequality,
        \begin{align*}
            \bbE\left\|\frac{1}{|N_i|}\sum_{j\in N_i} g(\bm z_j)\right\|^2 \leq 2\left(\bbE\left\|\frac{1}{|N_i|}\sum_{j\in N_i^{(ER)}}g(\bm z_j)\right\|^2+\bbE\left\|\frac{1}{|N_i|}\sum_{j\in N_i\setminus N_i^{(ER)}} g(\bm z_j)\right\|^2 \right) = O\left(dn^{-\gamma}+n^{2(\alpha-\gamma)} \right).
        \end{align*}
    \end{proof}
\end{lemma}

\begin{lemma}\label{GNN output close to non-GNN}
Assume that $(\psi_\ell)$ is a sequence of odd, $C_\psi$-Lipschitz $\bbR^d\mapsto \bbR^d$ activation functions and $(M_0^{(\ell)}, M_1^{(\ell)})$ is a sequence of $\bbR^{d\times d}$ matrices with their operator norm bounded above by $C_M$. Let a $L$-layer network be defined as
$$\bm \xi_i^{(\ell+1)} = \psi_{\ell}\left(M_0^{(\ell)}\bm \xi_i^{(\ell)}+M_1^{(\ell)}\frac{1}{|N_i|}\sum_{j\in N_i}\bm \xi_j^{(\ell)}\right), \qquad \ell=0,1,\ldots,L-1,$$
and $\bm \xi_i^{(0)} = \bm z_i$ for all $i\in [n]$. Also let a self-only recursion be defined as $\phi_0(\bm u) = \bm u$, and $\phi_{\ell+1}(\bm u) = \psi_{\ell}\left( M_0^{(\ell)}\phi_\ell(\bm u)\right)$. If $\log n \ll d \ll n^\alpha$ and $\alpha<\gamma$, then
$$\frac{1}{n}\sum_{i=1}^n \left\|\bm\xi_i^{(L)} - \phi_L(\bm z_i)\right\|^2 = O_\bbP(dn^{-\gamma}+n^{2(\alpha-\gamma)}).$$
    \begin{proof}
        Define the notation $\bm e_i^{(\ell)}:= \bm\xi_i^{(\ell)} - \phi_\ell(\bm z_i)$. The activation functions $\psi_\ell$ are assumed to be odd and $C_\psi$-Lipschitz, so $\phi_\ell$'s are also odd and $(C_\psi C_M)^\ell$-Lipschitz. In particular,
        $$\bbE[\phi_\ell(\bm z_j)]=0.$$
        \medskip

        Now, observe that
        \begin{align*}
            \bm e_i^{(\ell+1)} = \bm \xi_i^{(\ell+1)} - \phi_{\ell+1}(\bm z_i) &= \psi_{\ell}\left(M_0^{(\ell)}\bm \xi_i^{(\ell)}+M_1^{(\ell)}\frac{1}{|N_i|}\sum_{j\in N_i}\bm \xi_j^{(\ell)}\right) - \psi_\ell\left(M_0^\ell \phi_\ell(\bm z_i)\right),
        \end{align*}
        so using $C_\psi$-Lipschitz property of $\eta_\ell$ and triangle inequality,
        $$\|\bm e_i^{(\ell+1)}\| \leq C_\psi\left(\left\|M_0^{(\ell)}\left(\bm \xi_i^{(\ell)}-\phi_\ell(\bm z_i)\right)\right\|+\left\|M_1^{(\ell)}\frac{1}{|N_i|}\sum_{j\in N_i}\bm \xi_j^{(\ell)}\right\|\right)\leq C_\psi C_M\left(\|\bm e_i^{(\ell)}\| + \left\|\frac{1}{|N_i|}\sum_{j\in N_i}\bm \xi_j^{(\ell)}\right\| \right).$$
        Then, decompose
        $$\frac{1}{|N_i|}\sum_{j\in N_i}\bm \xi_j^{(\ell)} = \frac{1}{|N_i|}\sum_{j\in N_i}\phi_\ell(\bm z_j) + \frac{1}{|N_i|}\sum_{j\in N_i}\bm e_j^{(\ell)}.$$
        Plugging it back in the earlier inequality, we get by squaring and taking expectation, and applying Cauchy-Schwarz inequality once,
        \begin{equation}\label{recursion pre}
            \bbE\|\bm e_i^{(\ell+1)}\|^2 \leq K_\ell\left(\bbE\left\|\bm e_i^{(\ell)}\right\|^2 + \bbE\left\|\frac{1}{|N_i|}\sum_{j\in N_i}\phi_\ell(\bm z_j)\right\|^2 + \bbE\left\|\frac{1}{|N_i|}\sum_{j\in N_i}\bm e_j^{(\ell)}\right\|^2\right),
        \end{equation}
        where $K_\ell$ is a constant depending only on $C_\psi,C_M, \ell$. From Lemma~\ref{vanishing msg lemma}, we get
        $$\bbE\left\|\frac{1}{|N_i|}\sum_{j\in N_i}\phi_\ell(\bm y_j)\right\|^2 = O\left(dn^{-\gamma}+n^{2(\alpha-\gamma)}\right).$$
        For the third summand, using the fact that $\{\bm e_j^{(\ell)}: j\in N_i\}$ are all identically distributed,
        \begin{align*}
            \bbE\left\|\frac{1}{|N_i|}\sum_{j\in N_i}\bm e_j^{(\ell)}\right\|^2 &= \bbE\left(\bbE\left[\left\|\frac{1}{|N_i|}\sum_{j\in N_i}\bm e_j^{(\ell)}\right\|^2\Bigr| N_i\right]\right)\leq \bbE\left(\frac{1}{|N_i|^2}\sum_{j\in N_i}\bbE\|\bm e_j^{(\ell)}\|^2\right).
        \end{align*}
        Let $J_i\sim \operatorname{Unif}(N_i)$ be a uniformly random neighbour of $i$, then
        $$\bbE\left(\frac{1}{|N_i|^2}\sum_{j\in N_i}\bbE\|\bm e_j^{(\ell)}\|^2\right) = \bbE\left(\frac{1}{|N_i|}\bbE\left[\|\bm e_{J_i}^{(\ell)}\|^2\Bigr| N_i\right]\right) = \bbE\left(\frac{1}{|N_i|}\bbE\|\bm e_{J_i}^{(\ell)}\|^2\right).$$
        Observe that given the size $|N_i|=k$, $N_i$ is uniform over all $k$-size subsets of $[n]\setminus\{i\}$, and that $J_i$ is uniform over $N_i$ by assumption, so given $|N_i|=k$ for any $k$, $J_i$ is uniform over $[n]\setminus \{i\}$. Thus by exchangeability of node labels,
        $$\bbE\left[\|\bm e_{J_i}^{(\ell)}\|^2\Bigr| |N_i|=k \right] = \bbE\|\bm e_1^{(\ell)}\|^2 = \bbE\|\bm e_i^{(\ell)}\|^2 \qquad \forall k=1,\ldots,n-1.$$
        Plugging this back in, gives for the third summand,
        $$\bbE\left\|\frac{1}{|N_i|}\sum_{j\in N_i}\bm e_j^{(\ell)}\right\|^2\leq \bbE\left(\frac{1}{|N_i|}\right)\bbE\|\bm e_i^{(\ell)}\|^2 = O(n^{-\gamma})\bbE\|\bm e_i^{(\ell)}\|^2,$$
        where we use Lemma~\ref{er ngbhd size}, Lemma~\ref{pure geo ngbhd size} and Lemma~\ref{inverse wij count moments} for the order of $\bbE\left(\frac{1}{|N_i|}\right)$. Then from Eq.~\ref{recursion pre}, we obtain the recursion
        $$\bbE\|\bm e_i^{(\ell+1)}\|^2 \leq (K_\ell+O(n^{-\gamma}))\bbE\|\bm e_i^{(\ell)}\|^2 + O\left(dn^{-\gamma} + n^{2(\alpha-\gamma)}\right).$$
        Since $n^{-\gamma}\longrightarrow 0$ as $n\longrightarrow\infty$, iterating the above $L$ times yields,
        $$\bbE\|\bm e_i^{(L)}\|^2 = O\left(dn^{-\gamma}+n^{2(\alpha-\gamma)}\right), \qquad \text{uniformly for all $i\in [n]$},$$
        since the starting condition gives $\bm e_i^{(0)}=\bm 0$ for all $i\in [n]$. Then, 
        Markov's inequality yields that
        $$\frac{1}{n}\sum_{i=1}^n \|\bm e_i^{(L)}\|^2 = O_\bbP(dn^{-\gamma}+n^{2(\alpha-\gamma)}).$$
    \end{proof}
\end{lemma}

\section{Matrix preliminaries required for proof of Theorem~\ref{main theorem beta lambda} and Theorem~\ref{main theorem beta y}}

In this section, we prove the concentration of the covariance matrices $\frac{\bm X_n^\top \bm X_n}{n}$, $\frac{\bm N_n^\top \bm N_n}{n}$ and $\frac{\bm X_n^\top N_n}{n}$. The proofs in this section follow simply from concentration theorems of covariance matrix of i.i.d. random vectors.

\begin{proposition}[Theorem 4.7.1 and Exercise 4.7.3 in \cite{Vershynin_2018}]\label{covariance matrix LLN}
Let $\bm X_1,\ldots, \bm X_n$ are i.i.d. sub-gaussian random vector in $\bbR^d$. More precisely, assume that there exists $K\geq 1$ such that 
$$\|\langle \bm X_1, \bm x \rangle \|_{\psi_2} \leq K\|\langle \bm X_1, \bm x \rangle \|_{2} \quad \text{for any } \bm x \in \bbR^d. $$
Then, denoting $\bm \Sigma = \bbE \bm X_1 \bm X_1^\top$ and $\hat{\bm\Sigma}_n = \frac{1}{n}\sum_{i=1}^n \bm X_i \bm X_i^\top$, one has
    $$\bbE\|\hat{\bm \Sigma}_n - \bm \Sigma\|_{op} \leq CK^2 \left(\sqrt{\frac{d}{n}}+\frac{d}{n} \right)\|\bm \Sigma\|_{op},$$
    where $C>0$ is an absolute constant. Further for any $u>0$, one has
    $$\|\hat{\bm \Sigma}_n - \bm \Sigma\|_{op} \leq CK^2\left(\sqrt{\frac{d+u}{n}}+\frac{d+u}{n} \right)\|\bm \Sigma\|_{op},$$
    with probability at least $1-2e^{-u}$.
\end{proposition}

\begin{lemma}\label{Conv of XtX and EtE}
    Assume that $d\ll n$. Then, we have
    $$\left\|\frac{\bm X_n^\top \bm X_n}{n} - \sigma_x^2\bm I_d \right\|_{op} ,  \left\|\frac{\bm N_n^\top \bm N_n}{n} - \sigma_\eta^2\bm I_d \right\|_{op}, \left\|\frac{\bm X_n^\top \bm N_n}{n} \right\|_{op}, \left\|\frac{\bm N_n^\top \bm X_n}{n}  \right\|_{op} = O_\bbP\left(\sqrt{\frac{d}{n}}\right).$$
    \begin{proof}
        Write for each $i\in[n]$,
        $$\bm u_i := \begin{pmatrix}
            \frac{\bm x_i}{\sigma_x} \\[1.5mm] \frac{\bm \eta_i}{\sigma_\eta}
        \end{pmatrix}\in \bbR^{2d} \sim N(\bm 0, \bm I_{2d}).$$
        We stack the $\bm u_i$'s in a $n\times 2d$ matrix $\bm U_n$ such that
        $$\bm U_n = \begin{bmatrix}
            \frac{\bm X_n}{\sigma_x} & \frac{\bm N_n}{\sigma_\eta}
        \end{bmatrix}.$$
        We have
        $$\frac{\bm U_n^\top \bm U_n}{n} = \begin{bmatrix}
            \frac{\bm X_n^\top \bm X_n}{\sigma_x^2n} & \frac{\bm X_n^\top \bm N_n}{\sigma_x\sigma_\eta n} \\[2mm] \frac{\bm N_n^\top \bm X_n}{\sigma_x\sigma_\eta n} & \frac{\bm N_n^\top\bm N_n}{\sigma_\eta^2 n}
        \end{bmatrix} = \frac{1}{n}\sum_{i=1}^n \bm u_i\bm u_i^\top.$$
        Since, $\bm u_i \stackrel{iid}{\sim} N(\bm 0, \bm I_d)$, it fits into the framework of Proposition~\ref{covariance matrix LLN} with $K=1$ and $\bm \Sigma = \bm I_{2d}$. Choose $u=2d$ and we have
        \begin{align*}
            \bbP\left(\left\|\frac{\bm U_n^\top \bm U_n}{n} - \bm I_{2d}\right\|_{op} \geq C(\sqrt{4d/n}+4d/n)\|\bm I_{2d}\|_{op} \right) \leq 2e^{-2d}.
        \end{align*}
        Since $2e^{-2d} \longrightarrow 0$ as $d\longrightarrow \infty$, we have that
        $$\left\|\frac{\bm U_n^\top \bm U_n}{n} - \bm I_{2d} \right\|_{op} = O_\bbP\left(\sqrt{\frac{d}{n}}\right).$$
        Since the operator norm of any block is bounded by the operator norm of the whole matrix, we get that
        $$\left\|\frac{\bm X_n^\top \bm X_n}{n} - \sigma_x^2\bm I_d \right\|_{op} ,  \left\|\frac{\bm N_n^\top \bm N_n}{n} - \sigma_\eta^2\bm I_d \right\|_{op}, \left\|\frac{\bm X_n^\top \bm 
        N_n}{ n} \right\|_{op}, \left\|\frac{\bm N_n^\top \bm X_n}{ n}  \right\|_{op} = O_\bbP\left(\sqrt{\frac{d}{n}}\right).$$
        
    \end{proof}
\end{lemma}

    \begin{lemma}\label{conv of xtu}
        Assume that $d\ll n$. Then,
        $$\left\| \frac{\bm X_n^\top \bm E_n}{n} \right\| = O_\bbP\left(\sqrt{\frac{d}{n}}\right), \quad \text{and} \quad \left\| \frac{\bm N_n^\top \bm E_n}{n} \right\| = O_\bbP\left(\sqrt{\frac{d}{n}}\right).$$
        \begin{proof}
            Note that $\frac{\bm X_n^\top \bm E_n}{n} \in \bbR^d$, and its $j$-th coordinate is given by
            $$\left(\frac{\bm X_n^\top \bm E_n}{n}\right)_j = \frac{1}{n}\sum_{i=1}^n (\bm X_n)_{ij}(\bm E_n)_i.$$
            So, using underlying independence, we obtain that
            $$\left\{\left(\frac{\bm X_n^\top \bm E_n}{n}\right)_{j}: j\in [d]\right\} \Bigr| \bm E_n \stackrel{i.i.d.}{\sim} N\left(0, \frac{\sigma_x^2\|\bm E_n\|^2}{n^2} \right).$$
            Hence $\frac{\bm X_n^\top \bm E_n}{n} \stackrel{d}{=} \frac{\sigma_x\|\bm E_n\|}{n}\Tilde{g}_d,$ where $\Tilde{g}_d \sim N(\bm 0, \bm I_d) \indep \bm G_n$. Using Theorem 3.1.1 in \cite{Vershynin_2018}, there exists a global constant $c>0$ such that, 
            $$\bbP\left(\|\Tilde{g}_d\|\geq \sqrt{d}+t\right) \leq 2e^{-ct^2}, \qquad \bbP\left(\|\bm G_n\|\geq \sigma_\varepsilon(\sqrt{n}+t) \right) \leq 2e^{-ct^2}.$$
            Taking $t = \Theta(1)$ we get that, with probability at least $1-4e^{-ct}$,
            $$\left\|\frac{\bm X_n^\top \bm E_n}{n}\right\| \leq \frac{\sigma_x\sigma_\varepsilon(\sqrt{d}+t)(\sqrt{n}+t)}{n} = \sigma_x\sigma_\varepsilon\left( 
\sqrt{\frac{d}{n}} + \frac{t}{\sqrt{n}} + \frac{\sqrt{d}t}{n}+ \frac{t^2}{n} \right).$$
Recalling that $d\ll n$, we get that
$$\left\|\frac{\bm X_n^\top \bm E_n}{n}\right\| = O_\bbP\left( \sqrt{\frac{d}{n}} \right).$$
A similar argument works for $\left\|\frac{\bm N_n^T \bm E_n}{n}\right\|$.
        \end{proof}
    \end{lemma}

\section{Lemmas required for proof of Theorem~\ref{lambda close to x}}\label{lambda close to x appendix}

This section of the appendix is devoted to showing that the attention-based proxies $\bm \lambda_i$ are close to the latent node covariates $\bm x_i$ in $L_r$ norm $(r \geq 1)$. We proceed by showing that the conditional expectation of the node covariate of a geometric neighbour of $\bm x_i$ is closely aligned with $\bm x_i$ up to some scaling factor, and that the corresponding conditional expectation for an Erd\H{o}s--R\'enyi neighbour is $\bm 0$. We then show that the attention-filtered average message concentrates around its expectation and that the signal from the geometric neighbours are recoverable over the noise from the Erd\H{o}s--R\'enyi neighbours under certain regime of the parameters $n,d,\alpha,\gamma,\sigma_x^2,\sigma_\eta^2.$

First note that 
\begin{align*}
    \begin{bmatrix}
        \bm z_j^{(1)} \\[1.5mm] \bm x_i^\top\bm x_j \\[1.5mm] \bm z_i^{(2)\top}\bm z_j^{(2)}
    \end{bmatrix} \,\Bigr|\, \bm x_i, \bm \eta_i \sim N\left(\bm 0, \begin{bmatrix}
        (\sigma_x^2+\sigma_\eta^2)\bm I_{d/2} & \sigma_x^2\bm x_i^{(1)} & \bm 0 \\[1.5mm] \sigma_x^2\bm x_i^{(1)\top} & \sigma_x^2\|\bm x_i\|^2 & \sigma_x^2\bm x_i^{(2)\top} \bm z_i^{(2)} \\[1.5mm] \bm 0 & \sigma_x^2\bm x_i^{(2)\top} \bm z_i^{(2)} & (\sigma_x^2+\sigma_\eta^2)\|\bm z_i^{(2)}\|^2
    \end{bmatrix} \right).
\end{align*}
Then by conditional distribution property of multivariate normals, we obtain that
$$\bm z_j^{(1)} \,|\,  \bm x_i, \bm \eta_i, \bm x_i^\top \bm x_j, \bm z_i^{(2)\top}\bm z_j^{(2)} \sim N(\bar{\bm\mu}, \bar{\bm\Sigma}),$$
where
\begin{align*}
    \bar{\bm\mu} &= \begin{bmatrix}
        \bm x_i^{(1)} & \bm 0
    \end{bmatrix} \begin{bmatrix}
        \sigma_x^2\|\bm x_i\|^2 & \sigma_x^2\bm x_i^{(2)\top}\bm z_i^{(2)}\\[1.5mm] \sigma_x^2\bm x_i^{(2)\top} \bm z_i^{(2)} & (\sigma_x^2+\sigma_\eta^2)\|\bm z_i^{(2)}\|^2
    \end{bmatrix}^{-1} \begin{bmatrix}
        \bm x_i^\top \bm x_j \\[1.5mm] \bm z_i^{(2)\top} \bm z_j^{(2)}
    \end{bmatrix}\\
    &= \begin{bmatrix}
        \bm x_i^{(1)} & \bm 0
    \end{bmatrix} \frac{1}{\sigma_x^2(\sigma_x^2+\sigma_\eta^2)\|\bm x_i\|^2\|\bm z_i^{(2)}\|^2 - \sigma_x^4(\bm x_i^{(2)\top} \bm z_i^{(2)})^2}\begin{bmatrix}
        (\sigma_x+\sigma_\eta^2)\|\bm z_i^{(2)}\|^2 & -\sigma_x^2\bm x_i^{(2)\top}\bm z_i^{(2)}\\[1.5mm] -\sigma_x^2\bm x_i^{(2)\top} \bm z_i^{(2)} & \sigma_x^2\|\bm x_i\|^2
    \end{bmatrix} \begin{bmatrix}
        \bm x_i^\top \bm x_j \\[1.5mm] \bm y_i^{(2)\top} \bm y_j^{(2)}
    \end{bmatrix}\\
    &= \frac{(\sigma_x^2+\sigma_\eta^2)\|\bm z_i^{(2)}\|^2 \bm x_i^\top \bm x_j - \sigma_x^2(\bm x_i^{(2)\top} \bm z_i^{(2)})(\bm z_i^{(2)\top}\bm z_j^{(2)})}{\sigma_x^2(\sigma_x^2+\sigma_\eta^2)\|\bm x_i\|^2\|\bm z_i^{(2)}\|^2 - \sigma_x^4(\bm x_i^{(2)\top} \bm z_i^{(2)})^2}\sigma_x^2\bm x_i^{(1)} ,
\end{align*}
and
\begin{align*}
    \bar{\bm \Sigma} &= (\sigma_x^2+\sigma_\eta^2)\bm I_{d/2} - \begin{bmatrix}
        \sigma_x^2\bm x_i^{(1)} & \bm 0
    \end{bmatrix} \begin{bmatrix}
        \sigma_x^2\|\bm x_i\|^2 & \sigma_x^2\bm x_i^{(2)\top}\bm z_i^{(2)}\\[1.5mm] \sigma_x^2\bm x_i^{(2)\top} \bm z_i^{(2)} & (\sigma_x^2+\sigma_\eta^2)\|\bm z_i^{(2)}\|^2
    \end{bmatrix}^{-1} \begin{bmatrix}
        \sigma_x^2\bm x_i^{(1)\top} \\[1.5mm] \bm 0^T
    \end{bmatrix} \preceq \bm I_{d/2},
\end{align*}
since the product of the three matrices above is positive semidefinite.

\begin{lemma}\label{denom lemma}
Define
$$X_d = \left\{(\sigma_x^2+\sigma_\eta^2)\|\bm z_i^{(2)}\|^2 \bm x_i^\top \bm x_j - \sigma_x^2(\bm x_i^{(2)\top} \bm z_i^{(2)})(\bm z_i^{(2)\top}\bm z_j^{(2)})\right\}\sigma_x^2\bm x_i^{(1)}.$$
Define the sigma field $\mathcal{F}_{ij} = \sigma\{\bm x_i, \bm x_i^\top \bm x_j, \bm z_i^{(2)\top}\bm z_j^{(2)}\}$. We have for every $i\neq j$ that
{$$\left\| \left\|\bbE(\bm z_j^{(1)} | \mathcal{F}_{ij}) - \frac{4\bbE[X_d|\mathcal{F}_{ij}]}{\sigma_x^4(\sigma_x^4+4\sigma_x^2\sigma_\eta^2+2\sigma_\eta^4)d^2} \right\|\right\|_{L_q}=O(d^{-1/2}),$$}
for any fixed $q\geq 1$.
\begin{proof}
We shall use the notation
$$Y_d = \sigma_x^2(\sigma_x^2+\sigma_\eta^2)\|\bm x_i\|^2\|\bm z_i^{(2)}\|^2 - \sigma_x^4(\bm x_i^{(2)\top} \bm z_i^{(2)})^2.$$
    We start by writing
    $$\Delta_d = \bbE[\bm z_j^{(1)}|\mathcal{F}_{ij}] - \frac{4\bbE[X_d|\mathcal{F}_{ij}]}{\sigma_x^4(\sigma_x^4+4\sigma_x^2\sigma_\eta^2+2\sigma_\eta^4)d^2} = \bbE\left[\frac{X_d}{Y_d}\Bigr|\mathcal{F}_{ij}\right] - \frac{4\bbE[X_d|\mathcal{F}_{ij}]}{\sigma_x^4(\sigma_x^4+4\sigma_x^2\sigma_\eta^2+2\sigma_\eta^4)d^2} = \bbE\left[\frac{X_dR_d'}{d^2}\Bigr| \mathcal{F}_{ij}\right],$$
    where $R_d' = \frac{d^2}{Y_d} - \frac{4}{\sigma_x^4(\sigma_x^4+4\sigma_x^2\sigma_\eta^2+2\sigma_\eta^4)}$. It then suffices to prove that
    $$\left\| \left\|\frac{X_d}{d^2} \right\|\right\|_{L_{2q}} = O(1), \quad \|R_d'\|_{L_
    {2q}} = O(d^{-1/2}),$$
    since Holder's inequality will then yield the required statement.

    To that end, recall that
    $$X_d = \left[(\sigma_x^2+\sigma_\eta^2)\|\bm z_i^{(2)}\|^2 \bm x_i^\top \bm x_j - \sigma_x^2(\bm x_i^{(2)\top}\bm z_i^{(2)})(\bm z_i^{(2)\top} \bm z_j^{(2)})\right]\sigma_x^2\bm x_i^{(1)}.$$
    Then with \cref{q norms bounds} along with Holder's inequality, we get that $\|\|X_d\|\|_{L_{2q}} = O(d^2)$, or equivalently $\| \|{X_d}/{d^2}\| \|_{L_{2q}} = O(1)$. 

    \medskip

    Now observe
    \begin{align*}
        Y_d - \frac{d^2}{4}\sigma_x^4(\sigma_x^4+4\sigma_x^2\sigma_\eta^2+2\sigma_\eta^4) = \underbrace{\sigma_x^2(\sigma_x^2+\sigma_\eta^2)\left[\|\bm x_i\|^2\|\bm z_i^{(2)}\|^2 - \sigma_x^2(\sigma_x^2+\sigma_\eta^2)\frac{d^2}{2}\right]}_{=:T_1} - \underbrace{\sigma_x^4\left[(\bm x_i^{(2)\top}\bm z_i^{(2)})^2 - \sigma_x^4\frac{d^2}{4}\right]}_{=:T_2}.
    \end{align*}
    Write
    $$\delta_1 = \|\bm x_i\|^2 - \sigma_x^2d, \qquad \delta_2 = \|\bm x_i^{(2)}\|^2 - \sigma_x^2\frac{d}{2}, \qquad \delta_3 = \|\bm z_i^{(2)}\|^2 - (\sigma_x^2+\sigma_\eta^2)\frac{d}{2}.$$
    Then, by \cref{q norms bounds}, we know that $\|\delta_1\|_{L_q}, \|\delta_2\|_{L_q}, \|\delta_3\|_{L_q} = O(\sqrt{d})$ for any $q\geq 1$. Thus, we get that
    {\small\begin{align*}
        &T_1 = \sigma_x^2(\sigma_x^2+\sigma_\eta^2)\left[(\sigma_x^2 d+\delta_1)\left((\sigma_x^2+\sigma_\eta^2)\frac{d}{2} + \delta_3 \right) - \sigma_x^2 (\sigma_x^2+\sigma_\eta^2)\frac{d^2}{2} \right] = \sigma_x^2(\sigma_x^2+\sigma_\eta^2)\left[\delta_3(\sigma_x^2d+\delta_1) + \frac{\sigma_x^2+\sigma_\eta^2}{2}d\delta_1\right] \\
        \implies & \|T_1\|_{L_q} = O(d\sqrt{d}),
    \end{align*}}
    where the last step follows using Minkowski's inequality. Similarly,
    {\small\begin{align*}
        &T_2 = \sigma_x^4\left[(\bm x_i^{(2)\top} \bm x_i^{(2)} + \bm x_i^{(2)\top} \bm \eta_i^{(2)})^2 - \sigma_x^4\frac{d^2}{4}\right] = \sigma_x^4\left[\|\bm x_i^{(2)}\|^4 + 2\|\bm x_i^{(2)}\|^2(\bm x_i^{(2)\top} \bm \eta_i^{(2)}) + (\bm x_i^{(2)\top} \bm \eta_i^{(2)})^2 - \sigma_x^4\frac{d^2}{4}\right]\\
        \implies & T_2 = \sigma_x^4\left[\left(\sigma_x^2\frac{d}{2} + \delta_2 \right)^2 + 2\|\bm x_i^{(2)}\|^2(\bm x_i^{(2)\top} \bm \eta_i^{(2)}) + (\bm x_i^{(2)\top} \bm \eta_i^{(2)})^2 - \sigma_x^4\frac{d^2}{4}\right]\\
        \implies & T_2 = \sigma_x^6d\delta_2+\sigma_x^4\delta_2^2 + 2\|\bm x_i^{(2)}\|^2(\bm x_i^{(2)\top} \bm \eta_i^{(2)}) + (\bm x_i^{(2)\top} \bm \eta_i^{(2)})^2\implies \|T_2\|_{L_q} = O(d\sqrt{d}).
    \end{align*}}
    Then combining the two with Minkowski's inequality, we get that
    $$\left\| Y_d - \frac{d^2}{4}\sigma_x^4(\sigma_x^4+4\sigma_x^2\sigma_\eta^2+2\sigma_\eta^4) \right\|_{L_q} = O(d\sqrt{d}).$$

    On the other hand, also note that
    {\small\begin{align*}
        Y_d &= \sigma_x^2(\sigma_x^2+\sigma_\eta^2)\|\bm x_i\|^2\|\bm z_i^{(2)}\|^2 - \sigma_x^4(\bm x_i^{(2)\top} \bm z_i^{(2)})^2\\
        &\geq \sigma_x^2(\sigma_x^2+\sigma_\eta^2)\|\bm x_i^{(1)}\|^2\|\bm z_i^{(2)}\|^2 + \sigma_x^2(\sigma_x^2+\sigma_\eta^2)\|\bm x_i^{(2)}\|^2\|\bm z_i^{(2)}\|^2 - \sigma_x^4\|\bm x_i^{(2)}\|^2\|\bm z_i^{(2)}\|^2 \geq \sigma_x^2(\sigma_x^2+\sigma_\eta^2)\|\bm x_i^{(1)}\|^2\|\bm z_i^{(2)}\|^2,
    \end{align*}}
    where the first inequality above follows by expanding $\|\bm x_i\|^2=\|\bm x_i^{(1)}\|^2+\|\bm x_i^{(2)}\|^2$ and using Cauchy-Schwarz inequality on the $(\bm x_i^{(2)\top} \bm z_i^{(2)})^2$ term. It then follows that
    $$\frac{1}{Y_d} \leq \frac{1}{\sigma_x^2(\sigma_x^2+\sigma_\eta^2)\|\bm x_i^{(1)}\|^2\|\bm z_i^{(2)}\|^2} \implies \left\|\frac{1}{Y_d}\right\|_{L_q} = O\left(\left\| \frac{1}{\|\bm x_i^{(1)}\|^2} \right\|_{L_{2q}} \left\| \frac{1}{\|\bm y_i^{(2)}\|^2} \right\|_{L_{2q}} \right) = O\left(\frac{1}{d^2}\right).$$
    The last two steps above follows using Holder's inequality and Proposition~\ref{q norm of inverted chi sq}. We now combine the norm orders obtained above to obtain for any $q\geq 1$,
    \begin{align*}
        \|R_d'\|_{L_q} = \left\| \frac{d^2}{Y_d} - \frac{4}{\sigma_x^4(\sigma_x^4+4\sigma_x^2\sigma_\eta^2+2\sigma_\eta^4)}  \right\|_{L_q} = \left\| \frac{d^2\sigma_x^4(\sigma_x^4+4\sigma_x^2\sigma_\eta^2+2\sigma_\eta^4) - 4Y_d}{Y_d\cdot\sigma_x^4(\sigma_x^4+4\sigma_x^2\sigma_\eta^2+2\sigma_\eta^4)}\right\|_{L_q} = O\left(\frac{d\sqrt{d}}{d^2}\right) = O(d^{-1/2}).
    \end{align*}
    The last follows, again, by Holder's inequality, and this concludes the proof.
\end{proof}
\end{lemma}

\begin{lemma}
    Define $\mathcal{F}_{ij}' = \sigma\{\bm x_i, \bm x_i^\top\bm x_j, \bm z_i^{(1)\top}\bm z_j^{(1)}\}$. Then, one has for every $i\neq j$,
    $$\left\|\left\| \bbE(\bm z_j^{(2)}|\mathcal{F}_{ij}') - \frac{4\bbE\left[\left\{(\sigma_x^2+\sigma_\eta^2)\|\bm z_i^{(1)}\|^2 \bm x_i^\top \bm x_j - \sigma_x^2(\bm x_i^{(1)\top} \bm z_i^{(1)})(\bm z_i^{(1)\top}\bm z_j^{(1)})\right\}\sigma_x^2\bm x_i^{(2)}\Bigr| \mathcal{F}_{ij}'\right]}{\sigma_x^4(\sigma_x^4+4\sigma_x^2\sigma_\eta^2+2\sigma_\eta^4)d^2} \right\|\right\|_{L_q} = O(d^{-1/2}),$$
    for any fixed $q\geq 1$.
    \begin{proof}
        The proof is similar to that of Lemma~\ref{denom lemma}.
    \end{proof}
\end{lemma}

\begin{lemma}\label{cond mean lemma}
For any $q\geq 1$, we have for any $i\neq j$,
$$\left\|\left\| \bbE[\bm z_j^{(1)}|\mathcal{F}_{ij}] - \frac{4}{d\sigma_x^2(\sigma_x^4+4\sigma_x^2\sigma_\eta^2+2\sigma_\eta^4)}\left[ \frac{(\sigma_x^2+\sigma_\eta^2)^2}{2}\bm x_i^\top\bm x_j - \frac{\sigma_x^4}{2} \bm z_i^{(2)\top}\bm z_j^{(2)}\right] \bm  x_i^{(1)} \right\|\right\|_{L_q} = O(d^{-1/2}).$$

\begin{proof}
Recall that $\bbE[\bm z_j^{(1)}|\mathcal{F}_{ij}] = \bbE\left[\frac{X_d}{Y_d}\Bigr|\mathcal{F}_{ij}\right]$, where $X_d,Y_d,\mathcal{F}_{ij}$ are notations as defined in Lemma~\ref{denom lemma}. Now write,
\begin{align*}
    &X_d - d\left[ \frac{(\sigma_x^2+\sigma_\eta^2)^2}{2}\bm x_i^\top \bm x_j - \frac{\sigma_x^4}{2}\bm z_i^{(2)\top}\bm z_j^{(2)} \right]\sigma_x^2\bm x_i^{(1)}\\
    =& \left\{ (\sigma_x^2+\sigma_\eta^2)\|\bm z_i^{(2)}\|^2 - \frac{(\sigma_x^2+\sigma_\eta^2)^2d}{2} \right\}\bm x_i^\top \bm x_j \cdot \sigma_x^2\bm x_i^{(1)} - \left\{ \sigma_x^2\bm x_i^{(2)\top}\bm z_i^{(2)} - \frac{\sigma_x^4 d}{2} \right\}\bm z_i^{(2)\top}\bm z_j^{(2)}\cdot \sigma_x^2 \bm x_i^{(1)}\\
    =& (\sigma_x^2+\sigma_\eta^2)\underbrace{\left\{ \|\bm z_i^{(2)}\|^2 - \frac{(\sigma_x+\sigma_\eta^2)d}{2} \right\}}_{V_1}\underbrace{\bm x_i^\top\bm x_j}_{V_2} \cdot \underbrace{\sigma_x^2\bm x_i^{(1)}}_{V_3} - \left\{ \underbrace{\sigma_x^2\left[\bm x_i^{(2)\top}\bm x_i^{(2)} - \frac{\sigma_x^2d}{2}\right]}_{V_4} + \underbrace{\sigma_x^2\bm x_i^{(2)\top}\bm \eta_i^{(2)}}_{V_5} \right\}\underbrace{\bm z_i^{(2)\top}\bm z_j^{(2)}}_{V_6}\cdot \underbrace{\sigma_x^2\bm x_i^{(1)}}_{V_3}.
\end{align*}
We know from Lemma~\ref{q norms bounds} that the $L_q$ norms of all the $V_1,V_2,\ldots,V_6$ terms are $O(\sqrt{d})$. Hence, with Holder's and Minkowski's inequality, we get that
\begin{equation}\label{equation in cond mean lemma}
    \left\|\left\|X_d - d\left[ \frac{(\sigma_x^2+\sigma_\eta^2)^2}{2}\bm x_i^\top \bm x_j - \frac{\sigma_x^4}{2}\bm z_i^{(2)\top}\bm z_j^{(2)} \right]\sigma_x^2\bm x_i^{(1)} \right\|\right\|_{L_q} = O(d\sqrt{d}).
\end{equation}

Thus, we get that
{\small\begin{align*}
    &\bbE[\bm z_j^{(1)}|\mathcal{F}_{ij}] - \frac{4}{d\sigma_x^4(\sigma_x^4+4\sigma_x^2\sigma_\eta^2+2\sigma_\eta^4)}\left[ \frac{(\sigma_x^2+\sigma_\eta^2)^2}{2}\bm x_i^\top\bm x_j - \frac{\sigma_x^4}{2} \bm z_i^{(2)\top}\bm z_j^{(2)}\right]  \sigma_x^2 \bm x_i^{(1)}\\
    =& \left[ \bbE(\bm z_j^{(1)} | \mathcal{F}_{ij}) - \frac{4\bbE[X_d|\mathcal{F}_{ij}]}{\sigma_x^4(\sigma_x^4+4\sigma_x^2\sigma_\eta^2+2\sigma_\eta^4)d^2}  \right] + \frac{4}{\sigma_x^4(\sigma_x^4+4\sigma_x^2\sigma_\eta^2+2\sigma_\eta^4)d^2}\left[ \bbE[X_d|\mathcal{F}_{ij}] - d\left(\frac{(\sigma_x^2+\sigma_\eta^2)^2}{2}\bm x_i^\top\bm x_j - \frac{\sigma_x^4}{2}\bm z_i^{(2)\top}\bm z_j^{(2)}\right)\sigma_x^2\bm x_i^{(1)} \right],
\end{align*}}
and by Lemma~\ref{denom lemma}, 
$$\left\| \left\|\bbE(\bm z_j^{(1)} | \mathcal{F}_{ij}) - \frac{4\bbE[X_d|\mathcal{F}_{ij}]}{\sigma_x^4(\sigma_x^4+4\sigma_x^2\sigma_\eta^2+2\sigma_\eta^4)d^2} \right\|\right\|_{L_q} = O(d^{-1/2}).$$

Combining with Eq.~\ref{equation in cond mean lemma}, we get that
$$\left\|\left\| \bbE[\bm z_j^{(1)}|\mathcal{F}_{ij}] - \frac{4}{d\sigma_x^2(\sigma_x^4+4\sigma_x^2\sigma_\eta^2+2\sigma_\eta^4)}\left[ \frac{(\sigma_x^2+\sigma_\eta^2)^2}{2}\bm x_i^\top\bm x_j - \frac{\sigma_x^4}{2} \bm z_i^{(2)\top}\bm z_j^{(2)}\right] \bm  x_i^{(1)} \right\|\right\|_{L_q} = O(d^{-1/2}).$$
\end{proof}
\end{lemma}

\begin{lemma}
    For any fixed $q\geq 1$, we have for any $i\neq j$,
$$\left\|\left\| \bbE[\bm z_j^{(2)}|\mathcal{F}_{ij'}] - \frac{4}{d\sigma_x^2(\sigma_x^4+4\sigma_x^2\sigma_\eta^2+2\sigma_\eta^4)}\left[ \frac{(\sigma_x^2+\sigma_\eta^2)^2}{2}\bm x_i^\top\bm x_j - \frac{\sigma_x^4}{2} \bm z_i^{(1)\top}\bm z_j^{(1)}\right] \bm  x_i^{(2)}\right\|\right\|_{L_q} = O(d^{-1/2}).$$
\begin{proof}
    The proof is similar to that of Lemma~\ref{cond mean lemma}.
\end{proof}
\end{lemma}

\begin{lemma}\label{E u v lemma}
Assume $n\gg d\gg (\log n)^3$. Let $Z_1$ and $Z_2$ are standard normals with correlation $\frac{\sigma_x^2}{\sqrt{2}(\sigma_x^2+\sigma_\eta^2)}$. Then
$$\left|\frac{\bbE\left[\bm x_i^\top\bm x_j \indic\left\{\bm x_i^\top \bm x_j \geq \sigma_x^2 t_n\sqrt{d}, \bm z_i^{(2)\top}\bm z_j^{(2)} \geq \frac{\sigma_x^2 t_n\sqrt{d}}{2}\right\}\right]}{\sigma_x^2\sqrt{d}\bbE\left[Z_1\indic\left\{Z_1 \geq t_n, Z_2 \geq \frac{\sigma_x^2 t_n}{\sqrt{2}(\sigma_x^2+\sigma_\eta^2)}\right\}\right]} - 1\right| = O\left( \frac{t_n^3}{\sqrt{d}}\right).$$    
\begin{proof}
    Let us write 
    $$U = \frac{\bm x_i^\top \bm x_j}{\sigma_x^2\sqrt{d}}, \qquad V = \frac{\bm z_i^{(2)\top}\bm z_j^{(2)}}{\sqrt{\frac{d}{2}}(\sigma_x^2+\sigma_\eta^2)}, \qquad \rho = \frac{\sigma_x^2}{\sqrt{2}(\sigma_x^2+\sigma_\eta^2)}.$$
    Then, by writing
    $$f(s) = \bbP\left(U\geq s, V \geq \rho t_n\right), \qquad g(s) = \bbP\left(Z_1\geq s, Z_2 \geq \rho t_n\right),$$
    we obtain
    $$\frac{1}{\sigma_x^2\sqrt{d}}\bbE\left[\bm x_i^\top\bm x_j \indic\left\{\bm x_i^\top \bm x_j \geq \sigma_x^2t_n\sqrt{d}, \bm z_i^{(2)\top}\bm z_j^{(2)} \geq \frac{\sigma_x^2 t_n\sqrt{d}}{2}\right\}\right] = \bbE\left[U\indic\{U\geq t_n, V\geq \rho t_n\}\right] = t_nf(t_n)+\int_{t_n}^\infty f(s)ds,$$
    $$\bbE\left[Z_1\indic\left\{Z_1 \geq t_n, Z_2 \geq \rho t_n\right\}\right] = t_ng(t_n) + \int_{t_n}^\infty g(s)ds.$$
    \begin{itemize}
        \item[1.] When $s\in[t_n,2t_n)$.\\[3mm]
        By Lemma~\ref{P u v lemma}, in this range of $s$,
        $$|f(s) - g(s)| \leq C_1\left(g(s)\frac{s^3}{\sqrt{d}}\right) = C_1'\left(g(s)\frac{t_n^3}{\sqrt{d}}\right).$$
        Also,
        $$\int_{t_n}^{2t_n} g(s)ds = \int_{t_n}^{2t_n} \bbP(Z_1\geq s, Z_2\geq \rho t_n)ds \leq \int_{t_n}^{\infty} \bbP(Z_1\geq s)ds =\Phi^c(t_n) \leq \frac{\varphi(t_n)}{t_n}.$$
        Thus, plugging it in,
        \begin{equation}\label{range 1 of s}
            \left|t_nf(t_n) - t_ng(t_n) + \int_{t_n}^{2t_n}(f(s)-g(s))ds\right| \leq C_1'\frac{t_n^3}{\sqrt{d}}\left(t_ng(t_n)+\int_{t_n}^{2t_n}g(s)ds\right)\leq C_1'\frac{t_n^3}{\sqrt{d}}\left(t_ng(t_n)+\frac{\varphi(t_n)}{t_n}\right).
        \end{equation}

        \item [2.] When $s\in[2t_n,d^{1/6})$.\\[3mm]
        Again, by Lemma~\ref{P u v lemma}, in this range of $s$,
        $$\frac{f(s)}{g(s)} \leq 1 + C_2\frac{s^3}{\sqrt{d}}\leq C_2'.$$
        Hence,
        \begin{align*}
            \int_{2t_n}^{d^{1/6}} f(s)ds \leq C_2'\int_{2t_n}^{d^{1/6}} g(s)ds &= C_2'\int_{2t_n}^{d^{1/6}} \bbP\left(Z_1\geq s, Z_2 \geq \rho t_n\right) ds\\
            &\leq C_2'\int_{2t_n}^{\infty} \bbP\left(Z_1\geq s\right) ds = C_2'\Phi^c(2t_n).
        \end{align*}
        Then using Mill's ratio inequality, we get
        \begin{equation}\label{range 2 of s}
            \int_{2t_n}^{d^{1/6}} f(s)ds \leq C_2'\Phi^c(2t_n) \leq C_2'\frac{\varphi(2t_n)}{2t_n}.
        \end{equation}

        \item[3.] When $s\in[d^{1/6},\sqrt{d})$.\\[3mm]
        The moment generating function of $U$ at $t$ is given by
        $$\bbE(e^{tU}) = (1-t^2/d)^{-d/2}, \qquad |t|<\sqrt{d}.$$
        Then, Chernoff's bound gives
        $$\bbP(U\geq s)\leq \sup_{t\in(0,\sqrt{d})}\exp(-ts)\left( 1-t^2/d\right)^{-d/2} = \sup_{t\in(0,\sqrt{d})}\exp\left[ 
        -\left(ts + \frac{d}{2}\log(1-t^2/d)\right) \right].$$
        Recall the elementary inequality 
        $$\log(1-y) \geq -y-2y^2, \qquad 0\leq y\leq \frac{1}{2}.$$
        We choose $t = \frac{s}{2}$ so that $t^2/d\leq 1/2$, then applying the above inequality gives
        \begin{align*}
            \bbP(U\geq s)&\leq  \exp\left[-\frac{s^2}{2}-\frac{d}{2}\left( -\frac{t^2}{d} -\frac{2t^4}{d^2}\right)\right] = \exp\left[-\frac{3s^2}{8} + \frac{s^4}{16d}\right]= \exp\left[ -\frac{5s^2}{16} \right],
        \end{align*}
        Thus, we get
        \begin{equation}\label{range 3 of s}
            \begin{split}
                \int_{d^{1/6}}^{\sqrt{d}} f(s)ds\leq \int_{d^{1/6}}^{\sqrt{d}} \bbP(U\geq s)ds\leq \int_{d^{1/6}}^{\infty} e^{-5s^2/16}ds\leq \frac{8}{5d^{1/6}}e^{-5d^{1/3}/16}, 
            \end{split}
        \end{equation}
        where the last step follows using Mill's ratio inequality.
        
        \item[4.] When $s\in[\sqrt{d},\infty)$.\\[3mm]
        Recall that 
        $$\bbP(U\geq s)\leq \sup_{t\in(0,\sqrt{d})}\exp\left[ 
        -\left(ts + \frac{d}{2}\log(1-t^2/d)\right) \right].$$
        We choose $t = \frac{\sqrt{d}}{2}$ to get
        $$\bbP(U\geq s)\leq \exp\left[ 
        -\frac{s\sqrt{d}}{2} - \frac{d}{2}\log(3/4)\right].$$
        Also writing $s = \sqrt{d}+u$, for $u\geq 0$,
        $$\bbP(U\geq s)\leq \exp\left[-\frac{1}{2}(\sqrt{d}+u)\sqrt{d} - \frac{d}{2}\log(3/4)\right] = \exp\left[-c_1d-\frac{1}{2}\sqrt{d}u\right],$$
        where $c_1 = \frac{1}{2}(1+\log(3/4))>0$. Then, we get by integrating
        \begin{equation}\label{range 4 of s}
            \int_{\sqrt{d}}^\infty \bbP(U\geq s) ds \leq \int_0^\infty \exp\left[-c_1d-\frac{1}{2}\sqrt{d}u\right] du = e^{-c_1d}\cdot \frac{2}{\sqrt{d}}.
        \end{equation}
    \end{itemize}

    Now we shall combine the bounds for the different regimes. 
    Plugging in Eq.~\ref{range 1 of s}, \ref{range 2 of s}, \ref{range 3 of s} and \ref{range 4 of s}, we get
    \begin{align*}
        &\left|\bbE\left[U\indic\{U\geq t_n, V\geq \rho t_n\}\right] - \bbE\left[Z_1\indic\{Z_1\geq t_n, Z_2\geq \rho t_n\}\right] \right|\\
        \leq & \left|t_nf(t_n) - t_ng(t_n) + \int_{t_n}^{2t_n}(f(s)-g(s))ds\right|+\int_{2t_n}^\infty f(s)ds+\int_{2t_n}^\infty g(s)ds\\
        \leq & C_1'\frac{t_n^3}{\sqrt{d}}\left(t_ng(t_n)+\frac{\varphi(t_n)}{t_n}\right) + \left[C_2'\frac{\varphi(2t_n)}{2t_n}+\frac{8}{5d^{1/6}}e^{-5d^{1/3}/16}+\frac{2}{\sqrt{d}}e^{-c_1d}\right]+\int_{2t_n}^\infty \varphi(s)ds\\
        \leq&  C_1'\frac{t_n^3}{\sqrt{d}}\left(t_ng(t_n)+\frac{\varphi(t_n)}{t_n}\right) + \left[C_2''\frac{\varphi(2t_n)}{2t_n}+\frac{8}{5d^{1/6}}e^{-5d^{1/3}/16}+\frac{2}{\sqrt{d}}e^{-c_1d}\right].
    \end{align*}
    The last step is simply another application of the Mill's ratio inequality, with $C_2''=C_2'+1$, a new constant. Recall from Lemma~\ref{gaussian tail} that,
    $$g(t_n) = \bbP\left(Z_1\geq t_n, Z_2\geq \rho t_n\right) = \frac{\varphi(t_n)}{2t_n}\left(1+o(1)\right).$$
    Hence,
    \begin{align*}
        &\left|\bbE\left[U\indic\{U\geq t_n, V\geq \rho t_n\}\right] - \bbE\left[Z_1\indic\{Z_1\geq t_n, Z_2\geq \rho t_n\}\right] \right|\\
        \leq & \frac{\varphi(t_n)}{t_n}\left[C_1'' \frac{t_n^4}{\sqrt{d}}+ C_2''\frac{\varphi(2t_n)}{2\varphi(t_n)} + \frac{8t_n}{5d^{1/6}}\cdot\frac{e^{-5d^{1/3}/16}}{\varphi(t_n)} + \frac{2t_n}{\sqrt{d}}\cdot \frac{e^{-c_1d}}{\varphi(t_n)}\right] \leq \frac{\varphi(t_n)}{t_n}\cdot C_1'''\frac{t_n^4}{\sqrt{d}},
    \end{align*}
    for a constant $C_1'''>0$ because as we shall see, the $\frac{t_n^4}{\sqrt{d}}$ is the leading order term among the four summands inside the square bracket above. In fact,
    $$\frac{\varphi(2t_n)}{\varphi(t_n)}\cdot \frac{\sqrt{d}}{t_n^4} = \frac{1}{t_n^4}\cdot \exp\left(\frac{1}{2}\log d - \frac{3}{2}t_n^2\right)=o(1)\cdot \exp\left(\frac{1}{2}\log d - 3(1-\alpha)\log n\right) = o(1),$$
    $$\frac{t_n}{d^{1/6}}\cdot\frac{e^{-5d^{1/3}/16}}{\varphi(t_n)} \cdot \frac{\sqrt{d}}{t_n^4} = \frac{1}{t_n^3d^{1/6}} \cdot \exp\left( -\frac{5d^{1/3}}{16}+(1-\alpha)\log n+\frac{1}{2}\log d \right) = o(1),$$
    $$\frac{t_n}{\sqrt{d}}\cdot \frac{e^{-c_1d}}{\varphi(t_n)}\cdot\frac{\sqrt{d}}{t_n^4} = \frac{1}{t_n^3\sqrt{d}}\cdot\exp\left(-c_1d+(1-\alpha)\log n+\frac{1}{2}\log d\right) = o(1).$$
    Also, recall from Lemma~\ref{P u v lemma} that,
    $$\bbE[Z_1\indic\{Z_1\geq t_n, Z_2 \geq \rho t_n\}] = \frac{\varphi(t_n)}{2}(1+o(1)).$$
    So, it follows that
    \begin{align*}
        \left| \frac{\bbE[U\indic\{U\geq t_n, V \geq \rho t_n\}]}{\bbE[Z_1\indic\{Z_1\geq t_n, Z_2 \geq \rho t_n\}]} - 1 \right| &= \frac{\left|\bbE[U\indic\{U\geq t_n, V \geq \rho t_n\}] - \bbE[Z_1\indic\{Z_1\geq t_n, Z_2 \geq \rho t_n\}] \right|}{\bbE[Z_1\indic\{Z_1\geq t_n, Z_2 \geq \rho t_n\}]}\\
        &\leq \frac{C_1'''\frac{t_n^3\varphi(t_n)}{\sqrt{d}}}{\frac{\varphi(t_n)}{2}(1+o(1))} = O\left( \frac{t_n^3}{\sqrt{d}} \right).
    \end{align*}
    This concludes the proof.
\end{proof}
\end{lemma}

\begin{lemma}\label{the difficult errors}
    Assume that $n\gg d \gg (\log n)^3$. Then we have for any $i\neq j$,
    \begin{itemize}
        \item $\frac{\bbE[\bm x_i^\top \bm x_j|\bm x_i^\top \bm x_j\geq \sigma_x^2 t_n\sqrt{d}, \bm z_i^{(2)\top}\bm z_j^{(2)}\geq \sigma_x^2 t_n\sqrt{d}/2]}{\sigma_x^2 t_n\sqrt{d}}-1=O\left(\frac{1}{t_n^2}+\frac{t_n^3}{\sqrt{d}}\right)$
        \item $\frac{\bbE[\bm z_i^{(2)\top} \bm z_j^{(2)}|\bm x_i^\top \bm x_j\geq \sigma_x^2 t_n\sqrt{d}, \bm z_i^{(2)\top}\bm z_j^{(2)}\geq \sigma_x^2 t_n\sqrt{d}/2]}{\sigma_x^2 t_n\sqrt{d}/2}-1=O\left(\frac{1}{t_n}+\frac{t_n^3}{\sqrt{d}}\right)$
    \end{itemize}
    \begin{proof}
        Define standard normal random variables $Z_1,Z_2$ with correlation $\rho = \frac{\sigma_x^2}{\sqrt{2}(\sigma_x^2+\sigma_\eta^2)}$. Then, write $k=\frac{\rho}{\sqrt{1-\rho^2}}$ and we get from Lemma~\ref{gaussian tail},
        \begin{align*}
            \bbE[Z_1|{Z_1\geq t_n, Z_2 \geq \rho t_n}] &= \frac{\bbE[Z_1\indic\{Z_1\geq t_n, Z_2 \geq \rho t_n\}]}{\bbP({Z_1\geq t_n, Z_2 \geq \rho t_n})}\\
            &= \frac{\frac{1}{2}+\frac{k}{\sqrt{2\pi}t_n}+O\left(\frac{1}{t_n^3}\right)}{\frac{1}{2t_n}+\frac{k}{\sqrt{2\pi}t_n^2}-\frac{1}{2t_n^3}+O\left(\frac{1}{t_n^4}\right)}\\
            &= 2t_n\left(\frac{1}{2}+\frac{k}{\sqrt{2\pi}t_n}+O\left(\frac{1}{t_n^3}\right)\right)\left(1+\frac{2k}{\sqrt{2\pi}t_n}-\frac{1}{t_n^2}+O\left(\frac{1}{t_n^3}\right)\right)^{-1}\\
            &= 2t_n\left(\frac{1}{2}+\frac{k}{\sqrt{2\pi}t_n}+O\left(\frac{1}{t_n^3}\right)\right)\left(1 - \frac{2k}{\sqrt{2\pi}t_n} + \left(\left(\frac{2k}{\sqrt{2\pi}}\right)+1\right)^2\frac{1}{t_n^2} + O\left(\frac{1}{t_n^3}\right) \right)\\
            &= 2t_n\left[\frac{1}{2} + \left(-\frac{k}{\sqrt{2\pi}}+\frac{k}{\sqrt{2\pi}}\right)\frac{1}{t_n} + \left(\frac{1}{2}\left(\frac{2k}{\sqrt{2\pi}}\right)^2+\frac{1}{2} - \frac{2k^2}{2\pi}\right)\frac{1}{t_n^2}+O\left(\frac{1}{t_n^3}\right)\right]\\
            &= t_n + \frac{1}{t_n} + O\left(\frac{1}{t_n^2}\right).
        \end{align*}
        Then write 
        $$U = \frac{\bm x_i^\top \bm x_j}{\sigma_x^2\sqrt{d}}, \qquad V = \frac{\bm z_i^{(2)\top}\bm z_j^{(2)}}{\sqrt{\frac{d}{2}}(\sigma_x^2+\sigma_\eta^2)},$$
        so that
        \begin{align*}
            \frac{\bbE[\bm x_i^\top \bm x_j|\bm x_i^\top \bm x_j\geq \sigma_x^2 t_n\sqrt{d}, \bm z_i^{(2)\top}\bm z_j^{(2)}\geq \sigma_x^2 t_n\sqrt{d}/2]}{\sigma_x^2t_n\sqrt{d}}-1 &= \frac{\bbE[U|U\geq t_n, V\geq \rho t_n]}{t_n} - 1\\
            &= \frac{\bbE[U\indic\{U\geq t_n, V\geq \rho t_n\}]}{t_n\bbP(U\geq t_n, V\geq \rho t_n)} - 1\\
            (\text{from Lemmas~\ref{P u v lemma},~\ref{E u v lemma}})\,\,&= \frac{\bbE[Z_1\indic\{Z_1\geq t_n, Z_2\geq \rho t_n\}]\left(1+O\left(\frac{t_n^3}{\sqrt{d}}\right)\right)}{t_n\bbP(Z_1\geq t_n, Z_2\geq \rho t_n)\left(1+O\left(\frac{t_n^3}{\sqrt{d}}\right)\right)} - 1\\
            &= \left(1+\frac{1}{t_n^2}+O\left( \frac{1}{t_n^3}\right)\right)\left(1+O\left(\frac{t_n^3}{\sqrt{d}}\right)\right)-1\\
            &= \frac{1}{t_n^2} + O\left(\frac{t_n^3}{\sqrt{d}}\right).
        \end{align*}
        Similar proofs can be written for the other case by modifying the required lemmas accordingly.
    \end{proof}
\end{lemma}

\begin{lemma}\label{dot prod cond exp appx 1}
We have for all $i,j\in [n]$,
$$\bbE\left[\bm z_j^{(1)} \Bigr| \bm x_i, w_{ij,2}=1, j\in N_i^{(ER)}\right] = 0, \qquad \bbE\left[\bm z_j^{(2)} \Bigr| \bm x_i, w_{ij,1}=1, j\in N_i^{(ER)}\right] = 0.$$
    \begin{proof}
        Observe that all Erd\H{o}s--R\'enyi edges are independent of everything else. Also, the event $\{w_{ij,2}=1\} = \left\{\bm z_i^{(2)\top}\bm z_j^{(2)} \geq \frac{\sigma_x^2t_n\sqrt{d}}{2}\right\}$ is measurable with respect to $\bm z_i^{(2)},\bm z_j^{(2)}$. Since $\bm z_j^{(1)}$ is independent of $\bm x_i, \bm z_i$ and $\bm z_j^{(2)}$, we get that
        $$\bbE\left[\bm z_j^{(1)} \Bigr| \bm x_i, w_{ij,2}=1, j\in N_i^{(ER)}\right] = \bbE\left[\bm z_j^{(1)}\right]= 0.$$
        By similar argument, we also get that
        $$\bbE\left[\bm z_j^{(2)} \Bigr| \bm x_i, w_{ij,1}=1, j\in N_i^{(ER)}\right] = 0.$$
    \end{proof}
\end{lemma}

\begin{lemma}\label{dot prod cond exp appx 2}
Assume that $n\gg d \gg (\log n)^3$. We have for any fixed $r\geq 1$,
$$\left\|\left\|\bbE[\bm z_j^{(1)}| \bm x_i, w_{ij,2}=1, j\in N_i\setminus N_i^{(ER)}] - \frac{\bm x_i^{(1)}}{d}\cdot t_n\sqrt{d} \right\|\right\|_{L_r} = O\left(1+\frac{t_n^4}{\sqrt{d}}\right),$$
$$\left\|\left\|\bbE[\bm z_j^{(2)}| \bm x_i, w_{ij,1}=1, j\in N_i\setminus N_i^{(ER)}] - \frac{\bm x_i^{(2)}}{d}\cdot t_n\sqrt{d} \right\|\right\|_{L_r} = O\left(1+\frac{t_n^4}{\sqrt{d}}\right).$$

    \begin{proof}

        For any fixed $r\geq 1$, with Holder's inequality, using Lemma~\ref{the difficult errors},
        \begin{align*}
            &\left\|\left\|\frac{\bm x_i^{(1)}}{d}\bbE\left(\bm x_i^\top\bm x_j| \bm x_i^\top \bm x_j \geq \sigma_x^2 t_n\sqrt{d},  \bm z_i^{(2)\top}\bm z_j^{(2)}\geq \frac{\sigma_x^2t_n\sqrt{d}}{2}\right)- \frac{\bm x_i^{(1)}}{d} \cdot \sigma_x^2 t_n\sqrt{d} \right\|\right\|_{L_r} \\
            \leq& \sigma_x^2t_n\sqrt{d}\left\|\left\|\frac{\bm x_i^{(1)}}{d}\right\|\right\|_{L_{2r}}\left\| \frac{1}{\sigma_x^2 t_n\sqrt{d}}\bbE\left(\bm x_i^\top\bm x_j| \bm x_i^\top\bm x_j\geq \sigma_x^2 t_n\sqrt{d},  \bm z_i^{(2)\top}\bm z_j^{(2)}\geq \frac{\sigma_x^2 t_n\sqrt{d}}{2}\right)- 1\right\|_{L_{2r}}\\
            =& \sigma_x^2 t_n\sqrt{d} \times O(d^{-1/2}) \times O\left(\frac{1}{t_n^2}+\frac{t_n^3}{\sqrt{d}}\right) = O\left(\frac{1}{t_n}+\frac{t_n^4}{\sqrt{d}}\right).
        \end{align*}

        Similarly, we also get
        \begin{align*}
            &\left\|\left\|\frac{\bm x_i^{(1)}}{d}\bbE\left(\bm z_i^{(2)\top}\bm z_j^{(2)}| \bm x_i^\top \bm x_j \geq \sigma_x^2 t_n\sqrt{d},  \bm z_i^{(2)\top}\bm z_j^{(2)}\geq \frac{\sigma_x^2 t_n\sqrt{d}}{2}\right)- \frac{\bm x_i^{(1)}}{d} \cdot \frac{\sigma_x^2 t_n\sqrt{d}}{2} \right\|\right\|_{L_r}\\
            &= O\left(\frac{t_n\sqrt{d}}{\sqrt{d}}\left(\frac{1}{t_n}+\frac{t_n^3}{\sqrt{d}}\right)\right)=O\left(1+\frac{t_n^4}{\sqrt{d}}\right).
        \end{align*}
        
        So using Lemma~\ref{cond mean lemma} along with an application of Minkowski's and Holder's inequality, we have obtained for any fixed $r\geq 1$,
        {\small\begin{align*}
            &\left\|\left\|\bbE\left[\bm z_j^{(1)}| \bm x_i, w_{ij,2}=1, j\in N_i\setminus N_i^{(ER)}\right] - \frac{\bm x_i^{(1)}}{d}\cdot t_n\sqrt{d}\right\|\right\|_{L_r} \\
            \leq & \left\|\left\|\bbE[\bm z_j^{(1)}| \bm x_i, w_{ij,2}=1, j\in N_i \setminus N_i^{(ER)}] - \frac{4\bm x_i^{(1)}}{d\sigma_x^2 (\sigma_x^4+4\sigma_x^2\sigma_\eta^2+2\sigma_\eta^4)}\bbE\left[\frac{(\sigma_x^2+\sigma_\eta^2)^2}{2}\bm x_i^\top \bm x_j - \frac{\sigma_x^4}{2}\bm z_i^{(2)\top}\bm z_j^{(2)}\Bigr| w_{ij,2}=1, j\in N_i\setminus N_i^{(ER)}\right]\right\|\right\|_{L_r} \\
            &\quad\quad\quad\quad\quad\quad + \frac{2(\sigma_x^2+\sigma_\eta^2)^2}{\sigma_x^2(\sigma_x^4+4\sigma_x^2\sigma_\eta^2+2\sigma_\eta^4)}\left\|\left\|\frac{\bm x_i^{(1)}}{d}\bbE\left(\bm x_i^\top\bm x_j| \bm x_i^\top\bm x_j \geq \sigma_x^2t_n\sqrt{d}, \bm z_i^{(2)\top}\bm z_j^{(2)}\geq \frac{\sigma_x^2 t_n\sqrt{d}}{2}\right)- \frac{\bm x_i^{(1)}}{d}\cdot \sigma_x^2 t_n\sqrt{d}\right\|\right\|_{L_r}\\
            &\quad\quad\quad\quad\quad\quad+ \frac{2\sigma_x^4}{\sigma_x^2(\sigma_x^4+4\sigma_x^2\sigma_\eta^2+2\sigma_\eta^4)}\left\|\left\|\frac{\bm x_i^{(1)}}{d}\bbE\left(\bm z_i^{(2)\top}\bm z_j^{(2)}| \bm x_i^\top \bm x_j \geq \sigma_x^2t_n\sqrt{d}, \bm z_i^{(2)\top}\bm z_j^{(2)}\geq \frac{\sigma_x^2t_n\sqrt{d}}{2}\right)- \frac{\bm x_i^{(1)}}{d}\cdot \frac{\sigma_x^2 t_n\sqrt{d}}{2}\right\|\right\|_{L_r}\\
            =& O(d^{-1/2}) + O\left(1+\frac{t_n^4}{\sqrt{d}}\right)=O\left(1+\frac{t_n^4}{\sqrt{d}}\right).
        \end{align*}}
        A similar proof works for $\left\|\left\|\bbE[\bm z_j^{(2)}| \bm x_i, w_{ij,1}=1, j\in N_i\setminus N_i^{(ER)}] - \frac{\bm x_i^{(2)}}{d}\cdot t_n\sqrt{d} \right\|\right\|_{L_r}$.
    \end{proof}
\end{lemma}

\begin{lemma}\label{x_i from cond exp prelim}
Assume that $n\gg d \gg (\log n)^3$ and $\alpha>\gamma + \frac{\sigma_x^4(\alpha-1)}{2(\sigma_x^2+\sigma_\eta^2)^2}$. Then, we have for any $i\in [n]$ and fixed $r\geq 1$,
$$\left\|\left\|\bbE[\bm \lambda_i^{(1)}|\bm x_i]-\bm x_i^{(1)}\right\|\right\|_{L_r} = O\left(\sqrt{d}\left(n^{\gamma+\frac{\sigma_x^4(\alpha-1)}{2(\sigma_x^2+\sigma_\eta^2)^2}-\alpha}+\frac{1}{t_n}+\frac{t_n^3}{\sqrt{d}}\right)\right),$$
$$\left\|\left\|\bbE[\bm \lambda_i^{(2)}|\bm x_i]-\bm x_i^{(2)}\right\|\right\|_{L_r} = O\left(\sqrt{d}\left(n^{\gamma+\frac{\sigma_x^4(\alpha-1)}{2(\sigma_x^2+\sigma_\eta^2)^2}-\alpha}+\frac{1}{t_n}+\frac{t_n^3}{\sqrt{d}}\right)\right).$$
    \begin{proof}
    First note that
        \begin{equation*}
    \begin{split}
        &\bbE\left[\frac{t_n}{\sqrt{d}}\bm \lambda_i^{(1)}\Bigr|\bm x_i, N_i, N_i^{(ER)}\right] = \bbE\left[\frac{\sum_{j\in N_i} w_{ij,2} \bm z_j^{(1)}}{\sum_{j\in N_i} w_{ij,2}} \Bigr| \bm x_i,N_i, N_i^{(ER)}\right]\\
    =& \bbE\bbE\left[\frac{\sum_{j\in N_i} w_{ij,2} \bm z_j^{(1)}}{\sum_{j\in N_i} w_{ij,2}} \Bigr| \bm x_i,N_i, N_i^{(ER)}, \{w_{ij,2}\}_{j\in N_i}\right]\\
    =& \bbE\left[ \frac{\sum_{j\in N_i^{(ER)}} \bbE\left(w_{ij,2}\bm z_j^{(1)} | \bm x_i, \{w_{ij,2}\}_{j\in N_i} \right) + \sum_{j\in N_i\setminus N_i^{(ER)}} \bbE\left(w_{ij,2}\bm z_j^{(1)} | \bm x_i, \{w_{ij,2}\}_{j\in N_i} \right)}{\sum_{j\in N_i} w_{ij,2}}\Bigr| \{w_{ij,2}\}_{j\in N_i} \right].
    \end{split}
\end{equation*}
Now, we observe
\begin{align*}
    \bbE(w_{ij,2}\bm z_j^{(1)}| \bm x_i, w_{ij,2}, j\in N_i^{(ER)}) &= w_{ij,2}\bbE(\bm z_j^{(1)}| \bm x_i, w_{ij,2}=1, j\in N_i^{(ER)}),\\
    \bbE(w_{ij,2}\bm z_j^{(1)}| \bm x_i, w_{ij,2}, j\in N_i \setminus N_i^{(ER)}) &= w_{ij,2}\bbE(\bm z_j^{(1)}| \bm x_i, w_{ij,2}=1, j\in N_i \setminus N_i^{(ER)}).
\end{align*}

Thus,
\begin{equation}\label{eq msg pass}
    \begin{split}
        &\bbE\left[\frac{t_n}{\sqrt{d}}\bm \lambda_i^{(1)} \Bigr|\bm x_i, N_i, N_i^{(ER)}\right] = \bbE\left[\frac{\sum_{j\in N_i} w_{ij,2} \bm z_j^{(1)}}{\sum_{j\in N_i} w_{ij,2}} \Bigr| \bm x_i,N_i, N_i^{(ER)}\right]\\
    =& \bbE\left[\frac{\sum_{j\in N_i^{(ER)}} w_{ij,2}}{\sum_{j\in N_i} w_{ij,2}}\Bigr|\bm x_i,N_i, N_i^{(ER)}\right] \bbE(\bm z_j^{(1)}| \bm x_i, w_{ij,2}=1, j\in N_i^{(ER)}) \\
    &\quad\quad\quad\quad\quad+ \bbE\left[\frac{\sum_{j\in N_i\setminus N_i^{(ER)}} w_{ij,2}}{\sum_{j\in N_i} w_{ij,2}}\Bigr|\bm x_i,N_i, N_i^{(ER)} \right] \bbE(\bm z_j^{(1)}| \bm x_i, w_{ij,2}=1, j\in N_i\setminus N_i^{(ER)})\\
    =& \mathcal{T}_{1,i}\bbE(\bm z_j^{(1)}| \bm x_i, w_{ij,2}=1, j\in N_i^{(ER)})+(1-\mathcal{T}_{1,i})\bbE(\bm z_j^{(1)}| \bm x_i, w_{ij,2}=1, j\in N_i\setminus N_i^{(ER)})\\
    =& \mathcal{T}_{1,i}\cdot 0 +(1-\mathcal{T}_{1,i})\left(\frac{\bm x_i^{(1)}}{d}\cdot t_n\sqrt{d} + \nu_{i}\right),
    \end{split}
\end{equation}
by using Lemma~\ref{dot prod cond exp appx 1} and defining the quantities
\begin{align*}
    \mathcal{T}_{1,i} &= \bbE\left[\frac{\sum_{j\in N_i^{(ER)}} w_{ij,2}}{\sum_{j\in N_i} w_{ij,2}}\Bigr|\bm x_i,N_i, N_i^{(ER)}\right], \\
    \nu_{i} &= \bbE[\bm z_j^{(1)}| \bm x_i, w_{ij,2}=1, j\in N_i\setminus N_i^{(ER)}] - \frac{\bm x_i^{(1)}}{d}\cdot t_n\sqrt{d}.
\end{align*}
Then for any fixed $r\geq 1$, using the results of Lemma~\ref{wij count er} and Lemma~\ref{wij count} along with Lemma~\ref{inverse wij count moments} and Holder's inequality,
$$\|\bbE[\mathcal{T}_{1,i}|\bm x_i]\|_{L_r} \leq \|\mathcal{T}_{1,i}\|_{L_r} = \left\|\bbE\left[\frac{\sum_{j\in N_i^{(ER)}} w_{ij,2}}{\sum_{j\in N_i} w_{ij,2}}\Bigr|\bm x_i,N_i, N_i^{(ER)}\right] \right\|_{L_r} \leq \left\| \frac{\sum_{j\in N_i^{(ER)}} w_{ij,2}}{\sum_{j\in N_i} w_{ij,2}}\right\|_{L_r} = O\left(n^{\gamma+\frac{\sigma_x^4(\alpha-1)}{2(\sigma_x^2+\sigma_\eta^2)^2}-\alpha}\right).$$

Then, continuing from \cref{eq msg pass}, we write
\begin{align*}
    \bbE[\bm \lambda_i^{(1)}|\bm x_i] - \bm x_i^{(1)} 
    &= \bbE\left[ -\mathcal{T}_{1,i}\bm x_i^{(1)}  +  \frac{\sqrt{d}(1-\mathcal{T}_{1,i})\nu_{i}}{t_n}\Bigr| \bm x_i  \right].
\end{align*}
Taking the $L_r$ norm of the above, it follows using Jensen's and Holder's inequality, and Lemma~\ref{dot prod cond exp appx 1} and Lemma~\ref{dot prod cond exp appx 2} that, 
\begin{align*}
    \left\|\bbE[\bm \lambda_i^{(1)}|\bm x_i] - \bm x_i^{(1)} \right\|_{L_r} &\leq  \|\|\bbE[\mathcal{T}_{1,i}|\bm x_i]\bm x_i^{(1)}\|\|_{L_r}+\frac{\sqrt{d}}{t_n}\|\|\bbE[(1-\mathcal{T}_{1,i})\nu_{i}|\bm x_i]\|\|_{L_r}\\
    &\leq \|\mathcal{T}_{1,i}\|_{L_{2r}}\|\|\bm x_i^{(1)}\|\|_{L_{2r}} + \frac{\sqrt{d}}{t_n}\|1-\mathcal{T}_{1,i}\|_{L_{2r}}\|\|\nu_{i}\|\|_{L_{2r}}\\
    &= O\left(\sqrt{d}\cdot n^{\gamma+\frac{\sigma_x^4(\alpha-1)}{2(\sigma_x^4+\sigma_\eta^2)^2}-\alpha}\right)+O\left(\frac{\sqrt{d}}{t_n}\cdot \left(1+\frac{t_n^4}{\sqrt{d}}\right)\right)\\
    &= O\left(\sqrt{d}\left(n^{\gamma+\frac{\sigma_x^4(\alpha-1)}{2(\sigma_x^2+\sigma_\eta^2)^2}-\alpha}+\frac{1}{t_n}+\frac{t_n^3}{\sqrt{d}}\right)\right).
\end{align*}
A similar proof works for $\left\|\left\|\bbE[\bm \lambda_i^{(2)}|\bm x_i]-\bm x_i^{(2)}\right\|\right\|_{L_r}$.
    \end{proof}
\end{lemma}

\begin{corollary}\label{x_i from cond exp}
Assume that $n\gg d \gg (\log n)^3$ and $\alpha>\gamma + \frac{\sigma_x^4(\alpha-1)}{2(\sigma_x^2+\sigma_\eta^2)^2}$. Then, we have for any $i\in [n]$ and fixed $r\geq 1$,
$$\left\|\left\|\bbE[\bm \lambda_i|\bm x_i]-\bm x_i\right\|\right\|_{L_r} = O\left(\sqrt{d}\left(n^{\gamma+\frac{\sigma_x^4(\alpha-1)}{2(\sigma_x^2+\sigma_\eta^2)^2}-\alpha}+\frac{1}{t_n}+\frac{t_n^3}{\sqrt{d}}\right)\right).$$    
\begin{proof}
    The proof follows immediately by combining the two statements from Lemma~\ref{x_i from cond exp prelim}.
\end{proof}
\end{corollary}

\begin{lemma}\label{T_i from cond exp prelim}
Assume that $n \gg d \gg (\log n)^3$. Then, for any fixed $r\geq 1$,
$$\left\|\left\|\bm \lambda_i^{(1)}-\bbE[\bm \lambda_i^{(1)}|\bm x_i]\right\|\right\|_{L_r}=O\left( \frac{d}{\sqrt{t_n n^\alpha}} \right),$$
$$\left\|\left\|\bm \lambda_i^{(2)}-\bbE[\bm \lambda_i^{(2)}|\bm x_i]\right\|\right\|_{L_r}=O\left( \frac{d}{\sqrt{t_n n^\alpha}} \right).$$
    \begin{proof}
        Start by defining the notations
        $$\mu_i = \bbE[\bm z_j^{(1)}|\bm x_i, j\in \Tilde{N}_i^{(1)}],$$
        $$\nu_{i} = \bbE[\bm z_j^{(1)}| \bm x_i, w_{ij,2}=1, j\in N_i\setminus N_i^{(ER)}] - \frac{\bm x_i^{(1)}}{d}\cdot t_n\sqrt{d}.$$
        
        Then it follows that
        \begin{align*}
            \bbE\left[\frac{t_n}{\sqrt{d}}\bm \lambda_i^{(1)}|\bm x_i\right] = \bbE\left[\bbE\left[ \frac{1}{|\Tilde{N}_i^{(1)}|}\sum_{j\in \Tilde{N}_i^{(1)}} \bm z_j^{(1)} \Bigr| \bm x_i,\Tilde{N}_i^{(1)} \right]\Bigr| \bm x_i \right]
            = \bbE[\mu_i|\bm x_i].
        \end{align*}
        We also have from Lemma~\ref{dot prod cond exp appx 1} and Lemma~\ref{dot prod cond exp appx 2},
        \begin{align*}
            \mu_i &= \bbE[\bm z_j^{(1)}|\bm x_i, j\in \Tilde{N}_i^{(ER,1)}]\bbP\left(j\in \Tilde{N}_i^{(ER,1)}|\bm x_i, j\in \Tilde{N}_i^{(1)}\right)\\
            &\quad\quad\quad+ \bbE[\bm z_j^{(1)}|\bm x_i, j\in \Tilde{N}_i^{(1)}\setminus \Tilde{N}_i^{(ER,1)}]\bbP\left(j\in \Tilde{N}_i^{(1)}\setminus \Tilde{N}_i^{(ER,1)}|\bm x_i, j\in \Tilde{N}_i^{(1)}\right)\\
            \implies \|\|\mu_i\|\|_{L_r} &\leq 0 + \left\|\left\| \bbE[\bm z_j^{(1)}|\bm x_i, j\in \Tilde{N}_i^{(1)}\setminus \Tilde{N}_i^{(ER,1)}]\right\|\right\|_{L_r}\\
            &=  \left\|\left\| \frac{\bm x_i^{(1)}}{d}\cdot t_n\sqrt{d}+\nu_{i} \right\|\right\|_{L_r}\\
            &\leq \frac{t_n}{\sqrt{d}}\|\|\bm x_i^{(1)}\|\|_{L_r}+\|\|\nu_{i}\|\|_{L_r} = O\left(t_n\right)+O\left(\frac{1}{t_n}+\frac{t_n^4}{\sqrt{d}}\right)=O(t_n).
        \end{align*}

        Then, writing $R_{ij} = \bm z_j^{(1)} - \bbE[\mu_i|\bm x_i]$, we have
        \begin{align*}
         \|\|R_{ij}\|\|_{L_r} = \left\|\left\|\bm z_j^{(1)}-\bbE[\mu_i|\bm x_i]\right\|\right\|_{L_r} = \|\|\bm z_j^{(1)}\|\|_{L_r} + \|\|\mu_i\|\|_{L_r} = O(\sqrt{d}+t_n) = O(\sqrt{d}).
        \end{align*}
        
        With $\{R_{ij}: j\in S_i\}$ being conditionally i.i.d. given $\Tilde{N}_i^{(1)}$, we get from Lemma~\ref{inverse wij count moments} and Lemma~\ref{random sum moment order}, 
        $$\left\|\left\|\bm \lambda_i^{(1)}-\bbE[\bm \lambda_i^{(1)}|\bm x_i]\right\|\right\|_{L_r}=\frac{\sqrt{d}}{t_n}\left\|\frac{1}{|\Tilde{N}_i^{(1)}|}\sum_{j\in \Tilde{N}_i^{(1)}} R_{ij}\right\|_{L_r} = O\left( \frac{\sqrt{d}}{t_n}\cdot \frac{\sqrt{dt_n}}{\sqrt{n^\alpha}} \right) = O\left( \frac{d}{\sqrt{t_n n^\alpha}} \right).$$
    
    A similar proof works for $\left\|\left\|\bm \lambda_i^{(2)}-\bbE[\bm \lambda_i^{(2)}|\bm x_i]\right\|\right\|_{L_r}$.
    \end{proof}
\end{lemma}

\begin{corollary}\label{T_i from cond exp}
    Assume that $n \gg d \gg (\log n)^3$. Then, for any fixed $r\geq 1$,
$$\left\|\left\|\bm \lambda_i-\bbE[\bm \lambda_i|\bm x_i]\right\|\right\|_{L_r}=O\left( \frac{d}{\sqrt{t_n n^\alpha}} \right).$$
\begin{proof}
    The proof follows immediately by combining the two statements from Lemma~\ref{T_i from cond exp prelim}.
\end{proof}
\end{corollary}

\begin{lemma}
    Assume that $n^\alpha\gg d \gg (\log n)^3$ and $\alpha>\gamma+\frac{\sigma_x^4(\alpha-1)}{2(\sigma_x^2+\sigma_\eta^2)^2}$. Then for any fixed $r\geq 1$,
    $$\left\|\bm \lambda_i - \bm x_i\right\|_{L_r} = O\left(\sqrt{d}\left(n^{\gamma+\frac{\sigma_x^4(\alpha-1)}{2(\sigma_x^2+\sigma_\eta^2)^2}-\alpha}+\frac{1}{t_n}+\frac{t_n^3}{\sqrt{d}}+\frac{\sqrt{d}}{\sqrt{t_nn^\alpha}}\right)\right)=o(\sqrt{d}).$$
    \begin{proof}
        Writing
        $$\bm \lambda_i - \bm x_i = \left\{\bm \lambda_i - \bbE[\bm \lambda_i|\bm x_i]\right\} + \left\{\bbE[\bm \lambda_i|\bm x_i] - \bm x_i\right\},$$
        the claim follows immediately from an application of Minkowski's inequality, and Corollary~\ref{x_i from cond exp}, Corollary~\ref{T_i from cond exp}.
    \end{proof}
\end{lemma}

\section{Lemmas on concentration of neighbourhood sizes}\label{ngbhd size appdx}

We prove in this section that both the geometric and Erd\H{o}s--R\'enyi neighbourhood sizes concentrate for all nodes in the graph under our described model. We also show that the attention-filtered sub-neighbourhood sizes  concentrate.

\begin{lemma}\label{er ngbhd size}
For each $i\in [n]$, the Erd\H{o}s--R\'enyi neighbourhood size $|N_i^{(ER)}|$ follows
$$\bbE|N_i^{(ER)}| = n^\gamma \text{ and } |N_i^{(ER)}|=\Theta_\bbP(n^\gamma).$$

    \begin{proof}
        For each $i\in [n]$, the Erd\H{o}s--R\'enyi neighbourhood size has distribution
        $$|N_i^{(ER)}| \sim \operatorname{Bin}(n, n^{\gamma-1}). 
        $$
        Hence, $\bbE|N_i^{(ER)}| = n^\gamma $, and further by Chernoff's inequality,
        \begin{align*}
            \bbP\left(\frac{n^{\gamma}}{2} \leq |N_i^{(ER)}| \leq 2n^{\gamma}\right) \geq 1 - \left(\exp\left\{-\frac{n^{\gamma}}{3}\right\} + \exp\left\{-\frac{n^\gamma}{8}\right\}\right).
        \end{align*}
        By union bound,
        $$\bbP\left(\frac{n^{\gamma}}{2} \leq |N_i^{(ER)}| \leq 2n^{\gamma} \text{ for all $i\in[n]$}\right) \geq 1 - n\left(\exp\left\{-\frac{n^{\gamma}}{3}\right\} + \exp\left\{-\frac{n^\gamma}{8}\right\}\right).$$
        For any $0<\gamma<1$, note that $n\left(\exp\left\{-\frac{n^{\gamma}}{3}\right\} + \exp\left\{-\frac{n^\gamma}{8}\right\}\right) \longrightarrow 0$ as $n\longrightarrow \infty$, and that concludes the proof.
    \end{proof}
\end{lemma}

\begin{lemma}\label{pure geo ngbhd size}
Assume that $d\gg (\log n)^3$. Then, for all $i\in [n]$, the pure geometric sub-neighbourhood size $|N_i\setminus N_i^{(ER)}|$ follows
$$\bbE|N_i\setminus N_i^{(ER)}| = {\Theta}\left(\frac{n^\alpha}{t_n}\right), \text{ and } |N_i\setminus N_i^{(ER)}| = {\Theta}_\bbP\left(\frac{n^\alpha}{t_n}\right).$$

    \begin{proof}
        Conditionally given $N_i^{(ER)}$, the pure geometric sub-neighbourhood size of node $i$ follows the distribution $\left|N_i
        \setminus N_i^{(ER)}\right| \, \bigr|, N_i^{(ER)} \sim \operatorname{Bin}\left(n-\left|N_i^{(ER)}\right|, q_n\right)$, where
        \begin{align*}
            q_n = \bbE\left[\bbP\left(\bm{x}_i^\top\bm{x}_j \geq \sigma_x^2 t_n\sqrt{d}|\bm x_i\right)\right]=\bbE\left[\bbP\left(\bm{x}_i^\top \bm{x}_j \geq \sigma_x^2 t_n\sqrt{d}|\bm x_i \right)\right] = \bbE\left[\Phi^c\left( \frac{t_n\sqrt{d}}{\|\bm x_i/\sigma_x\|} \right)\right].
        \end{align*}
        We have from Proposition~\ref{expectation of phi^c general form} and Lemma~\ref{er ngbhd size} that,
        $$\bbE\left[\Phi^c\left( \frac{t_n\sqrt{d}}{\|\bm x_i/\sigma_x\|} \right)\right] = {\Theta}\left(\frac{n^{\alpha-1}}{t_n}\right), \qquad \bbE[n - |N_i^{(ER)}|]=n-n^\gamma = \Theta(n).$$
        Combining the above two gives,
        $$\bbE\left|N_i\setminus N_i^{(ER)}\right|={\Theta}\left(\frac{n^\alpha}{t_n}\right).$$
        
        Also by Chernoff's inequality,
        \begin{align*}
            &\bbP\left(\left(n-\left|N_i^{(ER)}\right|\right)\frac{q_n}{2} \leq \left|N_i \setminus N_i^{(ER)}\right| \leq 2\left(n-\left|N_i^{(ER)}\right|\right)q_n\,\bigr| \, N_i^{(ER)}\right)\\[2mm]
            \geq& 1 - \left(\exp\left\{-\left(n-\left|N_i^{(ER)}\right|\right)\frac{q_n}{3}\right\} + \exp\left\{-\left(n-\left|N_i^{(ER)}\right|\right)\frac{q_n}{8}\right\}\right).
        \end{align*}
        Since from Lemma~\ref{er ngbhd size}, $n-\left|N_i^{(ER)}\right| = \Theta_\bbP(n)$, we get
        $$\left(n-\left|N_i^{(ER)}\right|\right)q_n = {\Theta}_\bbP\left(\frac{n^\alpha}{t_n}\right).$$
        Because $0<\alpha<1$, one has $1-n\exp(-{\Theta}(n^\alpha/t_n)) \longrightarrow 1$, as $n\longrightarrow\infty$, hence by a union bound, for all $i\in [n]$, 
        $$\left|N_i\setminus N_i^{(ER)}\right| = {\Theta}_\bbP\left(\frac{n^\alpha}{t_n}\right).$$
        
    \end{proof}
\end{lemma}

\begin{lemma}\label{wij count er}
    Assume that $d\gg (\log n)^3$ and $\gamma+\frac{\sigma_x^4(\alpha-1)}{2(\sigma_x^2+\sigma_\eta^2)^2}>0$. Then for all $i\in [n]$ and $k=1,2$, we have that 
    $$\bbE\left|\Tilde{N}_i^{(ER,k)}\right| = {\Theta}\left(\frac{n^{\gamma+\frac{\sigma_x^4(\alpha-1)}{2(\sigma_x^2+\sigma_\eta^2)^2}}}{t_n}\right), \qquad \left|\Tilde{N}_i^{(ER,k)}\right|= {\Theta}_\bbP\left(\frac{n^{\gamma+\frac{\sigma_x^4(\alpha-1)}{2(\sigma_x^2+\sigma_\eta^2)^2}}}{t_n}\right).$$
    \begin{proof}
        Conditionally given $\bm x_i, \bm \varepsilon_i$ and $N_i^{(ER)}$,
        $$\left|\Tilde{N}_i^{(ER,1)}\right| \,\Bigr|\, \bm x_i, \bm \eta_i, N_i^{(ER)} \sim \operatorname{Bin}\left(\left|N_i^{(ER)}\right|, p^{(ER,1)}_n\right),$$
        where
        \begin{align*}
            p^{(ER,1)}_n &= \bbP\left( w_{ij,2}=1 | \bm x_i, \bm \eta_i, j\in N_i^{(ER)} \right)= \bbP\left(\bm z_i^{(2)\top}\bm z_j^{(2)} \geq \frac{\sigma_x^2 t_n\sqrt{d}}{2} \Bigr| \bm x_i, \bm \eta_i \right)\\
            &= \Phi^c\left( \frac{ \sigma_x^2 t_n\sqrt{d}}{2\sqrt{\sigma_x^2+\sigma_\eta^2}\|\bm z_i^{(2)}\|} \right) = \Phi^c\left( \frac{ t_n\sqrt{d/2}}{\frac{\sqrt{2}(\sigma_x^2+\sigma_\eta^2)}{\sigma_x^2}\cdot \left\|\bm z_i^{(2)}/\sqrt{\sigma_x^2+\sigma_\eta^2}\right\|} \right).
        \end{align*}
        We have from Proposition~\ref{expectation of phi^c general form} and Lemma~\ref{er ngbhd size} that,
        $$\bbE\left[\Phi^c\left( \frac{ t_n\sqrt{d/2}}{\frac{\sqrt{2}(\sigma_x^2+\sigma_\eta^2)}{\sigma_x^2}\cdot \left\|\bm z_i^{(2)}/\sqrt{\sigma_x^2+\sigma_\eta^2}\right\|} \right) \right] = {\Theta}\left(\frac{1}{t_n}\cdot {n^{\frac{
        \sigma_x^4(\alpha-1)}{2(\sigma_x^2+\sigma_\eta
        ^2)^2}}}\right), \qquad \bbE\left[\left|N_i^{(ER)}\right|\right]=n^\gamma .$$
        Combining the above two gives,
        $$\bbE\left|\Tilde{N}_i^{(ER,1)}\right|={\Theta}\left(\frac{n^{\gamma+\frac{
        \sigma_x^4(\alpha-1)}{2(\sigma_x^2+\sigma_\eta
        ^2)^2}}}{t_n} \right).$$
        
        Also by Chernoff's inequality,
        \begin{align*}
            &\bbP\left(\left|N_i^{(ER)}\right|\frac{p_n^{(ER,1)}}{2} \leq \left|\Tilde{N}_i^{(ER,1)}\right| \leq 2\left|N_i^{(ER)}\right|p_n^{(ER,1)} \,\bigr|\, \bm x_i, \bm \varepsilon_i, N_i^{(ER)}\right)\\
            \geq& 1 - \left(\exp\left\{-\left|N_i^{(ER)}\right|\frac{p_n^{(ER,1)}}{3}\right\} + \exp\left\{-\left|N_i^{(ER)}\right|\frac{p_n^{(ER,1)}}{8}\right\}\right).
        \end{align*}
        Since from Lemma~\ref{er ngbhd size}, $\left|N_i^{(ER)}\right| = \Theta_\bbP(n^\gamma)$, 
        $$\left|N_i^{(ER)}\right|p_n^{(ER,1)} = {\Theta}_\bbP\left(\frac{n^{\gamma+\frac{\sigma_x^4(\alpha-1)}{2(\sigma_x^2+\sigma_\eta^2)^2}}}{t_n}\right).$$
        Because $0<\gamma+\frac{\sigma_x^4(\alpha-1)}{2(\sigma_x^2+\sigma_\eta^2)^2}<1$, one has $1-n\exp\left(-{\Theta}\left(n^{\gamma+\frac{\sigma_x^4(\alpha-1)}{2(\sigma_x^2+\sigma_\eta^2)^2}}/t_n\right)\right) \longrightarrow 1$, hence by a union bound, for all $i\in [n]$, 
        $$\left|\Tilde{N}_i^{(ER,1)}\right| = {\Theta}_\bbP\left(\frac{n^{\gamma+\frac{\sigma_x^4(\alpha-1)}{2(\sigma_x^2+\sigma_\eta^2)^2}}}{t_n}\right).$$
        
        A similar argument works to show the same orders for $|\Tilde{N}_i^{(ER,2)}|$.
    \end{proof}
\end{lemma}

\begin{lemma}\label{wij count geo}
    Assume that $d\gg (\log n)^3$. For all $i\in [n]$ and $k=1,2$, we have that
    $$\bbE\left|\Tilde{N}_i^{(k)} \setminus \Tilde{N}_i^{(ER,k)}\right| = {\Theta}(n^\alpha/t_n), \text{ and } \left|\Tilde{N}_i^{(k)} \setminus \Tilde{N}_i^{(ER,k)}\right| = {\Theta}_\bbP(n^\alpha/t_n).$$
    \begin{proof}
        Conditionally given $N_i\setminus N_i^{(ER)}$,
        $$|\Tilde{N}_i^{(1)}\setminus \Tilde{N}_i^{(ER,1)}| \,\Bigr|\, N_i \setminus N_i^{(ER)} \sim \operatorname{Bin}(|N_i \setminus N_i^{(ER)}|, p^{(geo,1)}_n),$$
        where
        \begin{align*}
            p^{(geo,1)}_n &= \bbP\left( w_{ij,2}=1 |  j\in N_i \setminus N_i^{(ER)} \right)= \bbP\left(\bm z_i^{(2)\top}\bm z_j^{(2)} \geq \frac{\sigma_x^2 t_n\sqrt{d}}{2} \Bigr| \bm x_i^\top \bm x_j \geq \sigma_x^2 t_n\sqrt{d}\right).
        \end{align*}
        We have from Lemma~\ref{gaussian tail}, Proposition~\ref{expectation of phi^c general form} and Lemma~\ref{P u v lemma} that,
        \begin{align*}
            \bbP\left(\bm z_i^{(2)\top}\bm z_j^{(2)} \geq \frac{\sigma_x^2t_n\sqrt{d}}{2} \Bigr| \bm x_i^\top \bm x_j \geq \sigma_x^2t_n\sqrt{d}\right) &= \frac{\bbP\left(\bm z_i^{(2)\top}\bm z_j^{(2)} \geq \frac{\sigma_x^2 t_n\sqrt{d}}{2} , \bm x_i^\top \bm x_j \geq \sigma_x^2 t_n\sqrt{d}\right)}{\bbP\left( \bm x_i^\top \bm x_j \geq \sigma_x^2 t_n\sqrt{d}\right)}\\
            &= \frac{\frac{\varphi(t_n)}{2t_n}(1+o(1))\left(1+O\left(\frac{t_n^3}{\sqrt{d}}\right)\right)}{\bbE\left[\Phi^c\left(\frac{t_n\sqrt{d}}{\|\bm x_i/\sigma_x\|}\right)\right]}\\
            &= \frac{\frac{n^{\alpha-1}}{2\sqrt{2\pi}t_n}}{\frac{1}{\sqrt{2\pi}t_n}n^{\alpha-1}}\frac{(1+o(1))\left(1+O\left(\frac{t_n^3}{\sqrt{d}}\right)\right)}{1+o(1)} = \frac{1}{2}(1+o(1)).
        \end{align*}
        Also recall that, from Lemma~\ref{pure geo ngbhd size}, $\bbE(|N_i \setminus N_i^{(ER)}|) = {\Theta}(n^\alpha/t_n)$. Combining it with the above,
        $$\bbE\left|\Tilde{N}_i^{(1)} \setminus \Tilde{N}_i^{(ER,1)}\right|={\Theta}(n^\alpha/t_n).$$
        
        Also by Chernoff's inequality,
        \begin{align*}
            &\bbP\left(\left|N_i\setminus N_i^{(ER)}\right|\frac{p_n^{(ER,1)}}{2} \leq \left|\Tilde{N}_i^{(1)}\setminus \Tilde{N}_i^{(ER,1)}\right| \leq 2\left|N_i\setminus N_i^{(ER)}\right|p_n^{(ER,1)} \,\bigr|\, N_i\setminus N_i^{(ER)}\right)\\
            \geq& 1 - \left(\exp\left\{-\left|N_i\setminus N_i^{(ER)}\right|\frac{p_n^{(geo,1)}}{3}\right\} + \exp\left\{-\left|N_i\setminus N_i^{(ER)}\right|\frac{p_n^{(geo,1)}}{8}\right\}\right).
        \end{align*}
        Since from Lemma~\ref{pure geo ngbhd size}, $|N_i\setminus N_i^{(ER)}| = {\Theta}_\bbP(n^\alpha/t_n)$, 
        and $0<\alpha<1$, one has $1-n\exp\left(-{\Theta}\left(n^\alpha/t_n\right)\right) \longrightarrow 1$, as $n\longrightarrow \infty$, hence by a union bound, for all $i\in [n]$, 
        $$|\Tilde{N}_i^{(1)}\setminus \Tilde{N}^{(ER,1)}| = {\Theta}_\bbP\left(n^{\alpha}/t_n\right).$$
        
        A similar argument works to show the same orders for $|\Tilde{N}_i\setminus \Tilde{N}_i^{(ER,2)}|$.
    \end{proof}
\end{lemma}

\begin{lemma}\label{wij count}
Assume that $d\gg (\log n)^3$ and $\alpha > \gamma + \frac{\sigma_x^4(\alpha-1)}{2(\sigma_x^2+\sigma_\eta^2)^2}$. For any $i\in [n]$ and $k=1,2$, we have that
$$\bbE|\Tilde{N}_i^{(k)}| = {\Theta}\left(n^{\alpha}/t_n\right), \text{ and } |\Tilde{N}_i^{(k)}| = {\Theta}_\bbP\left(n^{\alpha}/t_n\right).$$
    \begin{proof}
    Observe the decomposition
    $$|\Tilde{N}_i^{(k)}| = |\Tilde{N}_i^{(ER,k)}| + |\Tilde{N}_i^{(k)} \setminus \Tilde{N}_i^{(ER, k)}|.$$
    Then the conclusion follows immediately from Lemma~\ref{wij count er} and Lemma~\ref{wij count geo}, recalling the assumption $\alpha > \frac{\sigma_x^4(\alpha-1)}{2(\sigma_x^2+\sigma_\eta^2)^2}$.
    \end{proof}
\end{lemma}

\section{Repeatedly used toolkits}

\begin{proposition}\label{norm x_i lemma}
Fix any $0<u<\sqrt{d}$. For all $1 \leq i \leq n$, we have
$$\sqrt{d} - u \leq \|\bm{x}_i\| \leq \sqrt{d} + u$$
with probability at least $1-2n\exp(u^2/4)$.
    \begin{proof}
    
        We know that $\|\bm{x}_i\|^2 = \sum_{k\in [d]} \bm{x}_i(k)^2$ follows a $\chi^2$ distribution with $d$ degrees of freedom, marginally. Then by Lemma 1 in \cite{10.1214/aos/1015957395}, we first have
        \begin{align*}
            \bbP\left(\|\bm{x}_i\|^2  \geq d + u\sqrt{d} + \frac{u^2}{2}\right) \leq \exp(-u^2/4).
        \end{align*}
        Similarly using the other concentration inequality from Lemma 1 in \cite{10.1214/aos/1015957395},
        \begin{align*}
            \bbP\left(\|\bm{x}_i\|^2  \leq d - u\sqrt{d} \right) \leq \exp(-u^2/4).
        \end{align*}
        Then note that
        \begin{align*}
            \bbP\left(\left|\|\bm{x}_i\|-\sqrt{d} \right| \geq u \right) &\leq \bbP(\|\bm{x}_i\| \geq \sqrt{d}+u) + \bbP(\|\bm{x}_i\| \leq \sqrt{d}-u)\\
            &= \bbP(\|\bm{x}_i\|^2 \geq d + 2u\sqrt{d}+u^2) + \bbP(\|\bm{x}_i\|^2 \leq d - 2u\sqrt{d}+u^2)\\
            &{\leq} \bbP\left(\|\bm{x}_i\|^2 \geq d+u\sqrt{d}+ \frac{u^2}{2}\right) + \bbP\left(\|\bm{x}_i\|^2 \leq d-u\sqrt{d}\right) \leq 2\exp(-u^2/4),
        \end{align*}
        where the penultimate step follows by noting that $d + 2u\sqrt{d}+u^2 \geq d + u\sqrt{d}+u^2/2$ because $u>0$ and $d - 2u\sqrt{d}+u^2 \leq d-u\sqrt{d}$ because $u<\sqrt{d}$. The final conclusion follows by taking an union bound over all $i$'s.
    \end{proof}
\end{proposition}

\begin{lemma}\label{orders lemma}
Let $c>0$ be a constant. We have that
$$\left|\Phi^c\left(\frac{t_n}{c}\right) - \Phi^c\left(\frac{t_n\sqrt{d}}{c\|\bm x_i\|}\right)\right| \leq \left(\frac{t_n u}{c(\sqrt{d}-u)}\right)\varphi\left(\frac{t_n\sqrt{d}}{c(\sqrt{d}+u)}\right)$$
holds for each $i\in [n]$, with probability at least $1-2ne^{-u^2/4}$.

    \begin{proof}
        For some $0<u<\sqrt{d}$, define the event
        $$E = \{\sqrt{d}-u \leq \|\bm{x}_i\| \leq \sqrt{d}+u \text{ for all } i\in[n]\},$$
        so that by Proposition~\ref{norm x_i lemma} and a union bound, $\bbP(E^c) \leq 2ne^{-u^2/4}$.
        Conditional on $E$ happening,
        \begin{equation}\label{norm feature order}
            \frac{\sqrt{d}}{\sqrt{d}+u} \leq \frac{\sqrt{d}}{\|\bm{x}_i\|} \leq \frac{\sqrt{d}}{\sqrt{d}-u}  \implies \varphi\left(\frac{t_n\sqrt{d}}{c(\sqrt{d}-u)}\right) \leq \varphi\left(\frac{t_n\sqrt{d}}{c\|\bm{x}_i\|}\right) \leq \varphi\left(\frac{t_n\sqrt{d}}{c(\sqrt{d}+u)}\right).
        \end{equation}
        We also have that, since $\varphi(\cdot)$ is strictly decreasing in $\bbR^+$, 
        $$\varphi\left(\frac{t_n}{c}\right) \leq \varphi\left(\frac{t_n\sqrt{d}}{c(\sqrt{d}+u)}\right).$$
        
        Further, by mean value theorem, there exists $\xi$ between $\frac{t_n}{c}$ and $\frac{t_n\sqrt{d}}{c\|\bm x_i\|}$ such that
        \begin{align*}
            \left|\Phi^c\left(\frac{t_n}{c}\right) - \Phi^c\left(\frac{t_n\sqrt{d}}{c\|\bm x_i\|}\right)\right| 
            &\leq \frac{\varphi(\xi)}{c}\left|t_n - \frac{t_n\sqrt{d}}{\|\bm{x}_i\|}\right| \\
            &\leq  \left(\frac{t_n u}{c(\sqrt{d}-u)}\right)\varphi\left(\frac{t_n\sqrt{d}}{c(\sqrt{d}+u)}\right),
        \end{align*}
        since $\varphi(\cdot)$ is monotonic between $t_n$ and $\frac{t_n\sqrt{d}}{\|\bm x_i\|}$.
    \end{proof}
\end{lemma}

\begin{proposition}\label{q norm of chi sq}
For every $r\ge 1$ there exists a finite constant $C_{r}\in(0,\infty)$,
depending only on $r$, such that
\[
    \sup_{d\ge 1}\;
    \mathbb E\left[ \left(\tfrac{\|\bm x_{i}\|}{\sqrt d}\right)^{r}\right]
    \le C_{r},
\]
and so $\left\| \frac{\|\bm x_i\|}{\sqrt{d}} \right\|_{L_r}$ is bounded above uniformly in $n,d$.

\begin{proof}
Denote
\[
    X_{d}\;:=\;\frac{\|\bm x_{i}\|}{\sqrt d},
\]
and it is known that $\|\bm x_i\|^2 \sim \chi_d^2$.
Using the well-known form of the $s$-th moment of a $\chi^{2}_{d}$ variable,
\[
    \mathbb E\bigl[(\|\bm x_i\|^2)^{s}\bigr]
    \;=\;2^{s}\,\frac{\Gamma\!\bigl(\frac{d}{2}+s\bigr)}{\Gamma\!\bigl(\frac{d}{2}\bigr)}.
\]
Choosing $s=r/2$ and dividing by $d^{r/2}$ gives
\[
    \mathbb E[X_{d}^{\,r}]
    \;=\;d^{-\,r/2}\,2^{r/2}\,
        \frac{\Gamma\!\bigl(\frac{d+r}{2}\bigr)}{\Gamma\!\bigl(\frac{d}{2}\bigr)}.
\]

Using the inequality $\frac{\Gamma(x+\alpha)}{\Gamma(x)}
    \;\le\;x^{\alpha}$ for all $x\ge 1$ and $\alpha\ge 0$,
 and plugging in $x=\tfrac d2\ge 1$ and
$\alpha=\tfrac r2$, we obtain
\[
    \mathbb E[X_{d}^{\,r}]
    \;\le\;d^{-\,r/2}\,2^{r/2}\left(\tfrac d2\right)^{r/2}
    \;=\;1,
    \qquad d\ge 2.
\]

For the only remaining case $d=1$,
\[
    \mathbb E[X_{1}^{\,r}]
    \;=\;2^{\,r/2}\,
        \frac{\Gamma\!\left(\tfrac{r+1}{2}\right)}{\sqrt\pi}.
\]
Then, setting
\[
    {\;
        C_{r}:=\max\left\{1,\;
        2^{r/2}\frac{\Gamma\left(\tfrac{r+1}{2}\right)}{\sqrt\pi}\right\}},
\]
we have concluded the proof.
\end{proof}
\end{proposition}

\begin{proposition}\label{q norm of inverted chi sq}
    For any $1\leq r \leq \frac{d}{2}$,
    $$\left\|\frac{\sqrt{d}}{\|\bm x_i\|}\right\|_{L_r} \leq \sqrt{2}.$$
    \begin{proof}
        Note that, since $\|\bm x_i\|^2 \sim \chi^2(d)$,
        \begin{align*}
            \bbE[\|\bm x_i\|^{-r}] = \bbE[(\|\bm x_i\|^2)^{-r/2}] &= \int_0^\infty z^{-r/2}\cdot \frac{1}{2^{d/2}\Gamma(d/2)}z^{d/2-1}e^{-z/2}dz\\
            &= 2^{-r/2}\frac{\Gamma((d-r)/2)}{\Gamma(d/2)}\\
            &\leq 2^{-r/2}\left( \frac{d-r}{2}\right)^{-r/2}\\
            &= (d-r)^{-r/2}.
        \end{align*}
        In the last inequality step above, we used that for any $u>v>0$, one has $\Gamma(u-v)/\Gamma(v)\leq (u-v)^{-v}$. Then we get, since $r \leq \frac{d}{2}$,
        $$\left\|\frac{\sqrt{d}}{\|\bm x_i\|}\right\|_{L_r}  = \sqrt{d}\,\bbE[\|\bm x_i\|^{-r}]^{1/r} \leq \sqrt{\frac{d}{d-r}}\leq \sqrt{2}.$$
    \end{proof}
\end{proposition}

\begin{lemma}\label{q norms bounds}
    Fix any $q\geq 1$. Then we have
    \begin{itemize}
        \item $\|\|\bm x_i\|^2\|_{L_q} = O(d), \|\|\bm z_i^{(2)}\|^2\|_{L_q} = O(d), \|\|\bm x_i^{(2)}\|^2\|_{L_q} = O(d), \|\bm x_i^{(2)\top} \bm z_i^{(2)}\|_{L_q} = O(d)$
        \item  $\|\|\bm x_i\|^2 - \sigma_x^2d \|_{L_q} = O(d^{1/2}), \|\|\bm z_i^{(2)}\|^2 - (\sigma_x^2+\sigma_\eta^2)d/2\|_{L_q} = O(d^{1/2}), \|\|\bm x_i^{(2)}\|^2 - \sigma_x^2d/2\|_{L_q} = O(d^{1/2})$
        \item $\|\bm x_i^{(2)\top}\bm \eta_i^{(2)}\|_{L_q} = O(d^{1/2}), \|\bm x_i^\top \bm x_j\|_{L_q} = O(d^{1/2}), \|\bm z_i^{(2)\top} \bm z_j^{(2)}\|_{L_q} = O(d^{1/2})$
    \end{itemize}
    \begin{proof}
        Note that 
        $$ \|\bm x_i\|^2 = \sum_{k=1}^d \bm x_i(k)^2,$$
        where $\bm x_i(k)^2 \sim \sigma_x^2\cdot \chi^2_1$. Hence by Minkowski's inequality and Proposition 2.7.1 in \cite{Vershynin_2018},
        $$\|\|\bm x_i\|^2\|_{L_q} \leq \sum_{k=1}^d \|\bm x_i(k)^2\|_{L_q} = O(d).$$
        Similar argument shows that $\|\|\bm z_i^{(2)}\|^2\|_{L_q} = O(d)$ and $\|\|\bm x_i^{(2)}\|^2\|_{L_q} = O(d)$.

        Next, note that
        $$
            \|\bm x_i\|^2 - \sigma_x^2d = \sum_{k=1}^d  (\bm x_i(k)^2 - \sigma_x^2),\quad
            \|\bm z_i^{(2)}\|^2 - (\sigma_x^2+\sigma_eta^2)d/2 = \sum_{k=d/2}^d (\bm z_i(k)^2-(\sigma_x^2+\sigma_\eta^2)), \quad \|\bm x_i^{(2)}\|^2 - \sigma_x^2 d/2 = \sum_{k=d/2}^d (\bm x_i(k)^2-1),
            $$
        $$\bm x_i^{(2)\top}\bm \eta_i^{(2)} = \sum_{k=d/2}^{d} \bm x_i(k)\bm \eta_i(k), \quad \bm x_i^\top \bm x_j = \sum_{k=1}^d \bm x_i(k) \bm x_j(k),\quad
            \bm z_i^{(1)\top}\bm z_j^{(1)} = \sum_{k=1}^d \bm z_i^{(1)}(k)\bm z_j^{(1)}(k).$$
        Each of the above are sums of i.i.d. centered sub-exponential random variables, so it follows by Proposition~\ref{sum of centered subexpo} that all of them have $O(d^{1/2})$ $L_q$ norms. Finally, noting that $\bm x_i^{(2)\top} \bm z_i^{(2)} = \|\bm x_i^{(2)}\|^2 + \bm x_i^{(2)\top}\bm \eta_i^{(2)},$ we have that
        $$\|\bm x_i^{(2)\top} \bm z_i^{(2)}\|_{L_q} = O(d).$$
    \end{proof}
\end{lemma}

\begin{proposition}\label{expectation of phi^c general form}
    Let $\bm Z \sim N(\bm 0, \bm I_d)$, and $c>0$ be a fixed constant. Then,
    $$\sqrt{2\pi}\frac{t_n}{c}n^{(1-\alpha)/c^2}\bbE\left[\Phi^c\left(\frac{t_n\sqrt{d}}{c\|\bm Z\|}\right)\right] \longrightarrow 1,$$
    as $n,d\longrightarrow \infty$ and $d \gg (\log n)^2$.
    \begin{proof}
        For any $x>0$, by Mills ratio inequality,
        $$\frac{\varphi(x)}{x+\frac{1}{x}} \leq \Phi^c(x) \leq \frac{\varphi(x)}{x}.$$
        For the proof, we shall temporarily use notations $S = \frac{\|\bm Z\|^2}{d}, R = \frac{\sqrt{d}}{\|\bm Z\|} = \frac{1}{\sqrt{S}}$, and $u_n = \frac{t_nR}{c}$. In the above inequality, put $x = u_n$ and multiply all sides by $\sqrt{2\pi}\frac{t_n}{c}n^{(1-\alpha)/c^2}$,
        \begin{align*}
             \frac{t_n}{c}e^{-u_n^2/2}\frac{u_n}{1+u_n^2}n^{(1-\alpha)/c^2} \leq \sqrt{2\pi}\frac{t_n}{c}n^{(1-\alpha)/c^2}\Phi^c\left(u_n\right) \leq \frac{t_n}{c}\frac{e^{-u_n^2/2}}{u_n}n^{(1-\alpha)/c^2}.
        \end{align*}
        The upper bound simplifies to 
        $$\frac{t_n}{c}\frac{e^{-u_n^2/2}}{u_n}n^{(1-\alpha)/c^2}=\frac{1}{R}n^{(1-\alpha)/c^2}e^{-t_n^2R^2/c^2}=\frac{1}{R}n^{\frac{(1-\alpha)}{c^2}(1-R^2)},$$
        and the lower bound simplifies to
        $$\frac{t_n}{c}e^{-u_n^2/2}\frac{u_n}{1+u_n^2}n^{(1-\alpha)/c^2} = \frac{t_n^2 R}{c^2+t_n^2R^2}\frac{1}{R}n^{\frac{(1-\alpha)}{c^2}(1-R^2)} = \frac{t_n^2R^2}{c^2 + t_n^2R^2}\frac{1}{R}n^{\frac{(1-\alpha)}{c^2}(1-R^2)}.$$
        Taking expectation, we get
        \begin{equation}\label{bounds on E g_n}
            \bbE\left[\frac{t_n^2R^2}{c^2 + t_n^2R^2}g_n\right] \leq \sqrt{2\pi}\frac{t_n}{c}n^{(1-\alpha)/c^2}\bbE[\Phi^c(u_n)] \leq \bbE[g_n],
        \end{equation}
        where $g_n = \frac{1}{R}n^{\frac{(1-\alpha)}{c^2}(1-R^2)} = \sqrt{S}\exp\left(\frac{1-\alpha}{c^2}\log n \cdot \frac{S-1}{S}\right)$. We shall prove that $\bbE[g_n]\longrightarrow 1$ and the gap between the lower and the upper bound tends to $0$.

        \begin{itemize}
            \item[(i)] $\bbE[g_n] \longrightarrow 1.$\\[3mm]
            By law of large numbers, $S \stackrel{\bbP}{\longrightarrow}1$. By Chebyshev's inequality, $|S-1| = O_\bbP(1/\sqrt{d})$; and on $|S-1|\leq M/\sqrt{d}$,
            $$\left| \frac{1-\alpha}{c^2}\log n \cdot \frac{S-1}{S} \right|\leq M'\frac{1-\alpha}{c^2}\cdot \frac{\log n}{\sqrt{d}} \longrightarrow 0,$$
            since $d\gg (\log n)^2$. Reusing $S \stackrel{\bbP}{\longrightarrow}1$,
            $$g_n = \sqrt{S}\exp\left(\frac{1-\alpha}{c^2}\log n \cdot \frac{S-1}{S}\right)\stackrel{\bbP}{\longrightarrow}1.$$

            We, next, show that $g_n$ is uniformly integrable via a uniform $L^2$ bound. We have
            $$g_n^2 = Sn^{\frac{2(1-\alpha)}{c^2}(1-1/S)}.$$
            If $S \leq 1$, we get
            $$1-\frac{1}{S} \leq 0 \implies g_n^2 \leq S \leq 1 \implies \bbE[g_n^2] \leq 1.$$
            On the other hand, if $S \geq 1$, we get
            $$1-\frac{1}{S} \leq S-1 \implies g_n^2 \leq S\exp\left(\lambda_n(S-1)\right),$$
            where $\lambda_n = \frac{2(1-\alpha)}{c^2}\log n$. Observe that $\frac{\lambda_n}{d} = O(\log n/d)\leq \frac{1}{2}$, for large enough $n,d$. Hence, noting that $dS \sim \chi^2_d$, we use the MGF of $\chi^2_d$ distribution at $\lambda_n/d$ and differentiate it to get,
            \begin{align*}
                \bbE\left[ S\exp(\lambda_n(S-1)) \right] = \frac{e^{-\lambda_n}}{d}\bbE\left[ dSe^{-\frac{\lambda_n}{d}\cdot dS} \right]= \frac{e^{-\lambda_n}}{d} \cdot d\left( 1 - \frac{2\lambda_n}{d} \right)^{-\frac{d}{2}-1}.
            \end{align*}
            It follows from the expansion $\log(1-x)=-x-\frac{x^2}{2}+O(x^3)$, that
            \begin{align*}
                &\log\bbE\left[ S\exp(\lambda_n(S-1)) \right] = -\lambda_n -\left(\frac{d}{2}+1\right)\log\left(1-\frac{2\lambda_n}{d}\right)\\
                \implies & \log\bbE\left[ S\exp(\lambda_n(S-1)) \right] = -\lambda_n + \left(\lambda_n+\frac{2\lambda_n}{d}\right)+\left(\frac{\lambda_n^2}{d}+\frac{2\lambda_n^2}{d^2} \right)+O\left(\frac{\lambda_n^3}{d^2}\right)\\
                \implies & \bbE\left[ S\exp(\lambda_n(S-1)) \right] = \exp\left(\frac{2\lambda_n}{d}+\frac{\lambda_n^2}{d}+O\left(\frac{\lambda_n^3}{d^2}\right)\right).
            \end{align*}
            Observing that the leading order in the summation on the right hand side of the above equation is $\frac{\lambda_n^2}{d}$, which converges to $0$ since $d\gg (\log n)^2$, we conclude that 
            $$\sup_n \bbE[S\exp(\lambda_n(S-1))]<\infty.$$
            This shows that $\sup_n \bbE[g_n^2]<\infty$, and combining this with $g_n \stackrel{\bbP}{\longrightarrow}1$, we get that $\bbE[g_n] \longrightarrow 1$.

            \item[(ii)] The gap between the lower and upper bound vanishes.\\[3mm]
            We have from Eq.~\ref{bounds on E g_n},
            $$0\leq \bbE[g_n] - \sqrt{2\pi}\frac{t_n}{c}n^{\frac{(1-\alpha)}{c^2}}\bbE[\Phi^c(u_n)] \leq \bbE\left[g_n\frac{c^2}{c^2+t_n^2R^2}\right]\leq \frac{c^2}{t_n^2} \bbE\left[ \frac{g_n}{R^2} \right]=\frac{c^2}{t_n^2} \bbE\left[ Sg_n \right].$$
            It suffices to show that $\frac{c^2}{t_n^2} \bbE\left[ Sg_n \right]\longrightarrow 0$. By Cauchy-Schwarz inequality,
            $$\sup_n\bbE\left[ Sg_n \right] \leq \sup_n \sqrt{E[S^2]E[g_n^2]} < \infty,$$
            by Proposition~\ref{q norm of chi sq}, and as shown above, $\sup_n\bbE[g_n^2]<\infty$. Because $c$ is a constant, while $t_n\longrightarrow\infty$, we conclude that
            $$\frac{c^2}{t_n^2} \bbE\left[ Sg_n \right]\longrightarrow 0.$$
        \end{itemize}
    \end{proof}

\end{proposition}

\begin{proposition}\label{sum of centered subexpo}
    Let $Z_1,\ldots,Z_d$ be $d$ i.i.d. sub-exponential random variables with $\bbE[Z_1] = 0$ and $\|Z_1\|_{\psi_1} = O(1)$ for all $k\in [d]$. Then for any fixed $q>1$,
    $$\left\| \sum_{k=1}^d Z_k \right\|_{q}  = O(d^{1/2}).$$
    \begin{proof}
        By Bernstein's inequality (Theorem 2.8.1 in \cite{Vershynin_2018}),
        $$\bbP\left(\left| \sum_{k=1}^d Z_k\right|\geq t \right) \leq 2\exp\left[ -c\min\left\{ \frac{t^2}{\|Z_1\|_{\psi_1}^2d},\frac{t}{\|Z_1\|_{\psi_1}} \right\} \right].$$
        Then we get
        \begin{align*}
            \bbE\left| \sum_{k=1}^d Z_k\right|^q &= q\int_{0}^\infty t^{q-1}\bbP\left(\left| \sum_{k=1}^d Z_k\right|\geq t  \right) dt\\
            &= \underbrace{q\int_{0}^d t^{q-1}\bbP\left(\left| \sum_{k=1}^d Z_k\right|\geq t  \right) dt}_{I_1} + \underbrace{q\int_{d}^\infty t^{q-1}\bbP\left(\left| \sum_{k=1}^d Z_k\right|\geq t  \right) dt}_{I_2}.
        \end{align*}
        For the term $I_1$,
        \begin{align*}
            I_1 &= q\int_{0}^d t^{q-1}\bbP\left(\left| \sum_{k=1}^d Z_k\right|\geq t  \right) dt\\
            &= 2q\int_{0}^d t^{q-1}\exp\left[-\frac{ct^2}{\|Z_1\|_{\psi_1}^2d} \right] dt\\
            &= 2q\int_0^{\frac{cd}{\|Z_1\|_{\psi_1}^2}} \left(\frac{\|Z_1\|_{\psi_1}^2 d}{c}\right)^{\frac{q-1}{2}} u^{\frac{q-1}{2}}e^{-u}\frac{\sqrt{\|Z_1\|_{\psi_1}^2 d}}{2\sqrt{c}}\frac{du}{\sqrt{u}}\\
            &\leq q\left(\frac{\sqrt{\|Z_1\|_{\psi_1}^2 d}}{c} \right)^{q}\int_{0}^\infty u^{\frac{q}{2}-1}e^{-u}du\\
            &= O(1)\Gamma(q/2)\times d^{q/2} = O(d^{q/2}).
        \end{align*}
        Moving on to the term $I_2$,
        \begin{align*}
            I_2 &= q\int_{d}^\infty t^{q-1}\bbP\left(\left| \sum_{k=1}^d Z_k\right|\geq t  \right) dt\\
            &= 2q\int_d^\infty t^{q-1} \exp\left[-\frac{ct}{\|Z_1\|_{\psi_1}} \right]dt = O(1),
        \end{align*}
        since the above is bounded by the full gamma integral. Hence, we get that 
        $$\left\| \sum_{k=1}^d Z_k \right\|_{q} \leq (I_1 + I_2)^{1/q} = O(d^{1/2}).$$
    \end{proof}
\end{proposition}

\begin{proposition}\label{random sum moment order}
Suppose that $N$ is a random subset of $[n]$, satisfying $\left\|\frac{1}{|N|}\right\|_{L_r}=\Theta(1/b_n)$ for some sequence $(b_n)_{n\in\bbN}$, and a fixed $r\geq 1$. A collection of random vectors $\{X_k: k\in N\}$ are i.i.d. given $N$, and $\|\|X_k\|\|_{L_r}=\Theta(a_n)$ for some sequence $(a_n)_{n\in\bbN}$. Then,
$$\left\|\left\| \frac{1}{|N|}\sum_{k\in N} X_k \right\|\right\|_{L_r} = O\left(\frac{a_n}{\sqrt{b_n}}\right).$$
\begin{proof}
    Write $Y_n = \frac{1}{|N|}\sum_{k\in N}X_k$ and $m=|N|$. Then,
    $$\bbE\|\|Y_n\|\|^r = \bbE\left[ \frac{1}{m^r}\bbE\left[\left\|\sum_{k\in N} X_k\right\|^r \Bigr|N \right] \right].$$
    Then using triangle inequality with Rosenthal's inequality to bound the $r$-th moment of sum i.i.d. random variables,
    $$\bbE\left[\left\|\left\|\sum_{k\in N} X_k\right\|\right\|^r \Bigr|N \right] \leq O(m^{r/2}a_n^r).$$
    Plugging this in,
    \begin{align*}
        \bbE\|\|Y_n\|\|^r &= O(a_n^r)\bbE\left[m^{-r/2} \right]\\
        &= O\left(\frac{a_n^r}{b_n^{r/2}}\right),
    \end{align*}
    recalling the assumption about $L_r$ moments of $\frac{1}{|N|}$. This concludes the proof.
\end{proof}
\end{proposition}

\begin{lemma}\label{subexpo wasserstein}
Let $W = K^{-1/2}\sum_{i=1}^K X_i\in \bbR^m$, where $\{X_1,\ldots,X_K\}$ are independent, $\bbE(X_i)=\bm 0$ for all $i$, and $\operatorname{Cov}(W)=\Sigma$. Suppose $\|X_i\|_{\psi_1}\leq b$ for all $1\leq i\leq K$. Let $Z \sim N(\bm 0, \Sigma)$. Then, for any $p\geq 2$, we have
$$\mathcal{W}_p(W,Z) \leq C\left(\frac{pm^{1/4}}{\sqrt{K}}+ \frac{p^{5/2}}{K}\right)\|\Sigma^{1/2}\|_{op}\|\Sigma^{-1/2}\|_{op}^2b^2. $$
    \begin{proof}
        Since $\operatorname{Cov}(W) = \Sigma$, we can define $Y := \Sigma^{-1/2}W$ with $\operatorname{Cov}(Y)=\bm I_m$. Also this operation does not change the mean, hence
        $$\bbE(Y) = \bm 0, \qquad \operatorname{Cov}(Y) = \bm I_m, \qquad \|Y\|_{\psi_1} \leq \|\Sigma^{-1/2}\|_{op}b.$$
        Define $Z \sim N(\bm 0, \bm I_m)$. Then, by Theorem~4.1 in \cite{10.1214/23-EJP976}, 
        $$\mathcal{W}_p(Y,Z) \leq C\left(\frac{pm^{1/4}}{\sqrt{K}}+ \frac{p^{5/2}}{K}\right)\|\Sigma^{-1/2}\|_{op}^2b^2. $$

        Then we immediately obtain
        $$\mathcal{W}_p(W, \Sigma^{1/2}Z) = \mathcal{W}_p(\Sigma^{1/2}Y, \Sigma^{1/2}Z)\leq \|\Sigma^{1/2}\|_{op}\mathcal{W}_p(Y,Z) \leq C\left(\frac{pm^{1/4}}{\sqrt{K}}+ \frac{p^{5/2}}{K}\right)\|\Sigma^{1/2}\|_{op}\|\Sigma^{-1/2}\|_{op}^2b^2.$$
    \end{proof}
\end{lemma}

\begin{proposition}\label{the nongaussian theorem}
    Let $W$ be a $m$-dimensional random vector, and $Z \sim N(\bm 0, \bm \Sigma)$. Let $A\subset \bbR^m$ be a nonempty convex Borel set, and for $\epsilon>0$, define its outer and inner parallel sets 
    $$A_\epsilon = \{\bm x\in \bbR^m: \operatorname{dist}(x, A)\leq \epsilon\}, \qquad A_{-\epsilon} = \{\bm x\in \bbR^m: \operatorname{dist}(x, A^c)> \epsilon\}.$$ Suppose there exist constants $\alpha> 0, A_0>0, p_0\geq 1,$ and $\Delta\in (0,1)$ such that
    $$\mathcal{W}_p(W,Z) \leq A_0 p^\alpha\Delta \quad\text{ for all } 1\leq p \leq p_0,$$
    and let 
    $$s_A = \log\left(\frac{1}{\bbP(Z \notin A)}\right)\in (0,\infty].$$
    Suppose that, there exists $\epsilon_0>0$ such that the following admissibility and boundary-strip condition hold:
    \begin{itemize}
        \item $|\log \Delta| + s_A\leq p_0$, and with 
        $$\epsilon_* := eA_0(|\log \Delta|+s_A)^\alpha\Delta,$$
        we have $\epsilon_*\leq \epsilon_0$.
        \item There is a non-decreasing function $r_A:(0,\epsilon_0] \mapsto [0,\infty)$ such that for all $0<\epsilon\leq \epsilon_0$,
        $$\frac{\bbP(Z\in A_\epsilon\setminus A_{-\epsilon})}{\bbP(Z \notin A)}\leq r_A(\epsilon).$$
    \end{itemize}
    Then,
    $$\left| \frac{\bbP(W \notin A)}{\bbP(Z\notin A)} -1\right|\leq r_A(\epsilon_*)+\Delta.$$
    \begin{proof}
        Write $E = \{W\notin A\}, G = \{Z\notin A\}$, and $H = \{\|W-Z\|\leq \epsilon\}$. First,
        $$|\bbP(E)-\bbP(G)|\leq \bbP(E\triangle G),$$
        so it suffices to bound the right hand side to bound the left hand side. The claim is that $E\triangle G \subset \{Z \in A_\epsilon \setminus A_{-\epsilon}\} \cup H^c$. Indeed, take any outcome with $\|W-Z\|\leq \epsilon$, i.e. in $H$. If $W\notin A$ while $Z \in A$, then $\operatorname{dist}(Z,A^c)\leq \epsilon$, hence $Z\notin A_{-\epsilon}$; so $Z \in A\setminus A_{-\epsilon}\subset A_\epsilon \setminus A_{-\epsilon}$. The other possibility is $W\in A$, while $Z \notin A$, in which case $\operatorname{dist}(Z,A)\leq \epsilon$ so that $Z \in A_\epsilon$; so $Z \in A_\epsilon\setminus A \subset A_\epsilon \setminus A_{-\epsilon}$. Therefore, on $H$, the symmetric difference $E\triangle G$ implies $Z$ lies in the boundary strip, thus proving $$|\bbP(E)-\bbP(G)|\leq \bbP(E\triangle G)\leq \bbP(Z \in A_\epsilon \setminus A_{-\epsilon}) + \bbP(H^c).$$ 
        This translates to
        \begin{equation}\label{symm diff equation}
            |\bbP(W\notin A) - \bbP(Z \notin A)| \leq \bbP(Z \in A_{\epsilon} \setminus A_{-\epsilon}) + \bbP(\|Z-W\|>\epsilon).
        \end{equation}
        Now, set $p = |\log \Delta| + s_A$ and hence $\epsilon_* = eA_0p^\alpha\Delta$. By admissibility assumptions, we have $p\leq p_0$ and $\epsilon_* \leq \epsilon_0$. From the upper bound on $\mathcal{W}_p(W,Z)$, we can couple $(W,Z)$ so that $\bbE[\|W-Z\|^p]^{1/p} \leq A_0p^\alpha\Delta$. Then, by Markov's inequality,
        \begin{align*}
            \bbP\left(\|W-Z\|>\epsilon_*\right)\leq  \frac{\bbE[\|W-Z\|^p]}{\epsilon_*^p} \leq \frac{(A_0p^\alpha\Delta)^p}{\epsilon_*^p} =  \frac{1}{e^p}.
        \end{align*}
        Now, with $\epsilon=\epsilon_*$, divide Eq.~\ref{symm diff equation} by $P(Z\notin A)$, and use the boundary-strip condition to get
        $$\left|\frac{\bbP(W\notin A)}{\bbP(Z\notin A)}-1\right|\leq \frac{\bbP(Z\in A_{\epsilon_*}\setminus A_{-{\epsilon_*}})}{\bbP(Z\notin A)} + \frac{\bbP(\|Z-W\|>\epsilon_*)}{\bbP(Z\notin A)}\leq r_A(\epsilon_*) + \frac{e^{-p}}{\bbP(Z\notin A)}.$$
        Finally, plugging in $p$ and $s_A$,
        $$\left|\frac{\bbP(W\notin A)}{\bbP(Z\notin A)}-1\right|\leq r_A(\epsilon_*) + \frac{e^{-|\log \Delta|+\log(\bbP(Z\notin A))}}{\bbP(Z\notin A)} = r_A(\epsilon_*) + \Delta.$$
        The last step followed since we assumed $\Delta\in(0,1)$.
    \end{proof}    
\end{proposition}

\begin{proposition}\label{Laplace at 0}
    If $F \in \mathcal{C}^{J+1}[0,\infty)$ and $F^{(J+1)}$ is bounded on $[0, \infty)$, then
    $$\int_{0}^\infty e^{-tw}F(w)dw = \sum_{j=0}^{J} \frac{F^{(j)}(0)}{t^{j+1}} + R_{J+1}, \text{ where } |R_{J+1}|\leq \frac{\|F^{(J+1)}\|_\infty}{t^{J+2}}. $$
\end{proposition}

\begin{lemma}\label{gaussian tail}
Let $Z_1, Z_2$ be two standard normal random variables with correlation $\rho\in(0,1)$. Then
$$\bbP\left(Z_1 \geq t_n, Z_2 \geq \rho t_n\right) = \varphi(t_n)\left[\frac{1}{2t_n} + \frac{\rho}{\sqrt{2\pi(1-\rho^2)}t_n^2} - \frac{1}{2t_n^3} \right] + O\left(\frac{\varphi(t_n)}{t_n^4}\right),$$
$$\bbE\left[Z_1\indic\{Z_1\geq t_n, Z_2 \geq \rho t_n\}\right] = \varphi(t_n)\left[\frac{1}{2}+\frac{\rho}{\sqrt{2\pi(1-\rho^2)}t_n} \right]+O\left(\frac{\varphi(t_n)}{t_n^3}\right).$$
\begin{proof}
    Note that
    $$\bbP(Z_2 \geq \rho t_n|Z_1 =u) = \Phi\left( \frac{\rho}{\sqrt{1-\rho^2}}(u-t_n) \right).$$
    Hence,
    \begin{align*}
        \bbP\left(Z_1 \geq t_n, Z_2 \geq \rho t_n\right) = \int_{t_n}^\infty \varphi(u)\, \bbP(Z_2 \geq \rho t_n|Z_1 =u) du &= \int_{t_n}^\infty \varphi(u)\Phi\left( \frac{\rho}{\sqrt{1-\rho^2}}(u-t_n) \right)du\\
        &= \varphi(t_n)\int_0^\infty e^{-t_n w} F(w) dw,
    \end{align*}
    where $F(w) = e^{-w^2/2}\Phi\left(\frac{\rho}{\sqrt{1-\rho^2}}w\right)$. Let's write $k=\frac{\rho}{\sqrt{1-\rho^2}}, A(w)=e^{-w^2/2}$ and $B(w)=\Phi\left(kw\right)$; and note that
    \begin{align*}
        A'(w) = -wA(w), \quad A''(w) = (w^2-1)A(w), \quad A'''(w) = -(w^3-3w)A(w),\\
        B'(w) = k\varphi(kw), \quad B''(w) = -k^3w\varphi(kw), \quad B'''(w) = k^3(k^2w^2-1)\varphi(kw).
    \end{align*}
    Consequently, we get
    $$F(0) = \frac{1}{2}, \quad F'(0) = \frac{k}{\sqrt{2\pi}}, \quad F''(0) = -\frac{1}{2},$$
    while $$F'''(w) = O(\operatorname{poly}(w)e^{-w^2/2}(\varphi(kw)+\Phi(kw))).$$
    We know that $\varphi(kw),\Phi(kw)$ are both bounded uniformly, and although the $\operatorname{polynom}(w)$ has coefficients depending on $k (k\leq 1)$, the $e^{-w^2/2}$ make the entire term uniformly bounded in $w, w\geq 0$. This sets up the premise to apply Proposition~\ref{Laplace at 0} and get
    \begin{align*}
        \bbP\left(Z_1 \geq t_n, Z_2 \geq \rho t_n\right) &= \varphi(t_n)\int_0^\infty e^{-t_n w} F(w) dw\\
        &= \varphi(t_n)\left[ \frac{F(0)}{t_n} + \frac{F'(0)}{t_n^2} + \frac{F''(0)}{t_n^3} \right] + O\left(\frac{\varphi(t_n)}{t_n^4}\right)\\
        &= \varphi(t_n)\left[\frac{1}{2t_n} + \frac{k}{\sqrt{2\pi}t_n^2} - \frac{1}{2t_n^3} \right] + O\left(\frac{\varphi(t_n)}{t_n^4}\right).
    \end{align*}
    On the other hand, 
    \begin{align*}
        \bbE\left[Z_1\indic\{Z_1\geq t_n, Z_2 \geq \rho t_n\}\right] = &\int_{t_n}^\infty u\varphi(u)\Phi(k(u-t_n))du\\
        =& \varphi(t_n)\int_{0}^\infty e^{-t_nw}(t_n+w)F(w)dw\\
        =& \varphi(t_n)\left[\underbrace{t_n\int_{0}^\infty e^{-t_nw}F(w)dw}_{\frac{t_n\bbP(Z_1\geq t_n, Z_2 \geq \rho t_n)}{\varphi(t_n)}}    +    \underbrace{\int_{0}^\infty we^{-t_nw}F(w)dw}_{:=R(t_n)} \right].
    \end{align*}
    Now, for the second summand $R(t_n)$, observe that
    \begin{align*}
        R(t_n) &= -\frac{d}{dt_n}\int_{0}^\infty e^{-t_n w}F(w)dw\\
        &= -\frac{d}{dt_n}\left[ \frac{F(0)}{t_n}+\frac{F'(0)}{t_n^2} + \frac{F''(0)}{t_n^3}+O\left(\frac{1}{t_n^4}\right) \right]\\
        &= \frac{F(0)}{t_n^2}+\frac{2F'(0)}{t_n^3} + \frac{3F''(0)}{t_n^4}+O\left(\frac{1}{t_n^5}\right) = \frac{1}{2t_n^2} + O\left(\frac{1}{t_n^3}\right).
    \end{align*}
    Plugging it in, we get
    \begin{align*}
        \bbE\left[Z_1\indic\{Z_1\geq t_n, Z_2 \geq \rho t_n\}\right] = & \varphi(t_n)\left[\frac{1}{2} + \frac{k}{\sqrt{2\pi}t_n} - \frac{1}{2t_n^2} + \frac{1}{2t_n^2}\right] + O\left(\frac{\varphi(t_n)}{t_n^3}\right)\\
        =& \varphi(t_n)\left[\frac{1}{2}+\frac{k}{\sqrt{2\pi}t_n}\right] + O\left(\frac{\varphi(t_n)}{t_n^3}\right).
    \end{align*}
    This concludes the proof.
\end{proof}
    
\end{lemma}

\begin{lemma}\label{P u v lemma}
Assume $d\gg (\log n)^3$. Let $Z_1$ and $Z_2$ are standard normals with correlation $\frac{\sigma_x^2}{\sqrt{2}(\sigma_x^2+\sigma_\eta^2)}$. Then for any $s$ satisfying $t_n\leq  s\leq d^{1/6}$,
$$\left|\frac{\bbP\left(\bm x_i^\top \bm x_j \geq s\sigma_x^2\sqrt{d}, \bm z_i^{(2)\top}\bm z_j^{(2)} \geq \frac{\sigma_x^2t_n\sqrt{d}}{2}\right)}{\bbP\left(Z_1 \geq s, Z_2 \geq \frac{\sigma_x^2t_n}{\sqrt{2}(\sigma_x^2+\sigma_\eta^2)}\right)} - 1\right| = O\left( \frac{s^3}{\sqrt{d}}\right).$$
    \begin{proof}
        Let 
        $$U = \frac{\bm x_i^\top \bm x_j}{\sqrt{d}\sigma_x^2}, \qquad V = \frac{\bm z_i^{(2)\top}\bm z_j^{(2)}}{\sqrt{\frac{d}{2}}(\sigma_x^2+\sigma_\eta^2)}.$$
        Then, 
        $$\bbE\begin{pmatrix}
            U\\ V
        \end{pmatrix}=\bm 0, \qquad \operatorname{Cov}\begin{pmatrix}
            U\\ V
        \end{pmatrix} = \begin{bmatrix}
            1 & \rho \\ \rho & 1
        \end{bmatrix}=:\Sigma_\rho,$$
        where $\rho = \frac{\sigma_x^2}{\sqrt{2}(\sigma_x^2+\sigma_\eta^2)}$. Also,  $(Z_1,Z_2)\sim N(\bm 0, \Sigma_\rho)$. Define the set $A = \{x\geq s, y \geq \rho t_n\}^c \in \bbR^2$; and in the notation of Proposition~\ref{the nongaussian theorem}, 
        $$A_{\epsilon}\setminus A_{-\epsilon} = \{s - \epsilon \leq x <s+\epsilon, y\geq \rho t_n-\epsilon\} \cup \{x\geq s+\epsilon, \rho t_n - \epsilon\leq y<\rho t_n+\epsilon\}.$$
        Define
        \begin{align*}
            \psi_{t_n,-\epsilon}(u) &= \bbP\left(Z_2\geq \rho t_n - \epsilon| Z_1=u\right)=\Phi\left(\frac{\rho}{\sqrt{1-\rho^2}}(u-t_n)+\frac{\epsilon}{\sqrt{1-\rho^2}}\right),\\
            \psi_{t_n,\epsilon}(u) &= \bbP\left(Z_2\geq \rho t_n + \epsilon| Z_1=u\right)=\Phi\left(\frac{\rho}{\sqrt{1-\rho^2}}(u-t_n)-\frac{\epsilon}{\sqrt{1-\rho^2}}\right).
        \end{align*}
        Then we get for one of the strips
        \begin{equation}\label{strip1}
            \begin{split}
                \bbP\left((Z_1,Z_2) \in \{s - \epsilon \leq x <s+\epsilon, y\geq \rho t_n-\epsilon\}\right)=
                 \int_{s-\epsilon}^{s+\epsilon} \varphi(u)\,\psi_{t_n,-\epsilon} (u)du
                \leq \int_{s-\epsilon}^{s+\epsilon} \varphi(u)du
                \leq 2\epsilon\varphi(s-\epsilon),
            \end{split}
        \end{equation}
        and for the other strip
        \begin{equation}\label{strip2}
        \begin{split}
            \bbP\left((Z_1,Z_2) \in \{ x \geq s+\epsilon, \rho t_n-\epsilon \leq y < \rho t_n+\epsilon\}\right) &= \int_{s+\epsilon}^\infty \varphi(u) \, (\psi_{t_n,-\epsilon}(u)-\psi_{t_n,\epsilon}(u)) du\\
            &\leq C'\epsilon\int_{s+\epsilon}^\infty \varphi(u)du \leq C'\epsilon\int_{s}^\infty \varphi(u)du\\
            &= C'\epsilon\Phi^c(s) \leq C'\epsilon\frac{\varphi(s)}{s},
        \end{split}
        \end{equation}
        where the first inequality step follows because $\psi_{t_n,-\epsilon}(u)-\psi(t_n,\epsilon)(u)$ is bounded linearly in $\epsilon$ since standard normal density is uniformly bounded, and the last inequality step follows by Mill's ratio inequality. Then, combining Eq.~\ref{strip1} and Eq.~\ref{strip2}, we get that for an appropriate constant $C''>0$,
        \begin{equation}\label{strip control}
            \bbP((Z_1,Z_2)\in A_{\epsilon}\setminus A_{-\epsilon}) \leq  C''\epsilon\left( \varphi(s-\epsilon)+\frac{\varphi(s)}{s} \right).
        \end{equation}
        Because the function $u\mapsto \Phi\left(\frac{\rho}{\sqrt{1-\rho^2}}(u-t_n)\right)$ is increasing,
        \begin{equation}\label{P A^c}
            \begin{split}
                \bbP((Z_1,Z_2)\in A^c) &= \int_s^\infty \varphi(u)\bbP\left(Z_2\geq \rho t_n|Z_1=u \right)du\\
                &= \int_s^\infty \varphi(u)\Phi\left(\frac{\rho}{\sqrt{1-\rho^2}}(u-t_n)\right)du \\
            &\geq \Phi\left(\frac{\rho}{\sqrt{1-\rho^2}}(s-t_n)\right)\int_s^\infty \varphi(u)du \geq \frac{1}{2}\Phi^c(s)\geq \frac{\varphi(s)}{2(s+1/s)},
            \end{split}
        \end{equation}
        where the last step follows using Mill's ratio inequality. Hence,
        $$\frac{\bbP((Z_1,Z_2)\in A_{\epsilon}\setminus A_{-\epsilon})}{\bbP((Z_1,Z_2) \in A^c)} \leq C''\cdot\frac{\epsilon\left( \varphi(s-\epsilon)+\frac{\varphi(s)}{s} \right)}{\varphi(s)/2(s+1/s)}.$$
        Write $U$ and $V$ each as a scaled sum of $\frac{d}{2}$ independent, subexponential random variables,
        $$U = \frac{1}{\sqrt{d}}\sum_{k=1}^{d/2} \left[\frac{\bm x_i(k)}{\sigma_x}\frac{\bm x_j(k)}{\sigma_x} + \frac{\bm x_i\left(\frac{d}{2}+k\right)}{\sigma_x}\frac{\bm x_j\left(\frac{d}{2}+k\right)}{\sigma_x}\right], \qquad V = \frac{1}{\sqrt{d/2}}\sum_{k=1}^{d/2} \frac{\bm z_i\left(\frac{d}{2}+k\right)}{\sqrt{\sigma_x^2+\sigma_\eta^2}}\frac{\bm z_j\left(\frac{d}{2}+k\right)}{\sqrt{\sigma_x^2+\sigma_\eta^2}}.$$
        Then the $\psi_1$ norms of $U$ and $V$ are bounded uniformly in $d$, and also both $\|\Sigma_\rho^{1/2}\|_{op}$ and $\|\Sigma_\rho^{-1/2}\|_{op}$ are bounded, so we get from Lemma~\ref{subexpo wasserstein},
        $$\mathcal{W}_p((U,V),(Z_1,Z_2)) \leq C\left(\frac{p}{\sqrt{d}} + \frac{p^{5/2}}{d}\right).$$
        Take $\alpha=1$ and fix $p_0 \geq 2$. Then, for all $1\leq p \leq p_0$,
        $$\frac{p^{5/2}}{d}\leq \frac{pp_0^{3/2}}{d} \implies \mathcal{W}_p((U,V),(Z_1,Z_2)) \leq Cp\Delta,$$
        where $\Delta = \frac{1}{\sqrt{d}}+ \frac{p_0^{3/2}}{d}$. Then, we get from Eq.~\ref{P A^c},
        $$s_A = \log\left(\frac{1}{\bbP((Z_1,Z_2)\notin A)}\right) \leq \log\left(\frac{2(s+1/s)}{\varphi(s)}\right)\leq cs^2,$$
        for a constant $c>0$. So if we pick  $p_0=C's^2$ for a constant $C'>0$,
        $$\Delta = \frac{1}{\sqrt{d}} + \frac{C^{'3/2}s^3}{\sqrt{d}\cdot\sqrt{d}} \geq \frac{1}{\sqrt{d}} \implies |\log \Delta|\leq  \frac{1}{2}\log d.$$
        It can also be observed that  $\frac{s^3}{d}\leq \frac{1}{\sqrt{d}}$ since $s\leq d^{1/6}$, yielding $\Delta = O(1/\sqrt{d})$.
        We choose the constant $C'$ in $p_0=C's^2$ such that
        $$p_0 =C's^2 \geq |\log \Delta|+s_A.$$
        Finally, define
        $$\epsilon_* = eC(|\log \Delta|+s_A)\Delta = O\left(s^2 \cdot \frac{1}{\sqrt{d}}\right) = O\left(d^{-1/6}\right),$$
        where we use that $s\leq d^{1/6}$ and $\Delta = O(1/\sqrt{d})$.
        An observation here is that 
        $$s\leq d^{1/6} \implies \frac{1}{s}\geq \frac{1}{d^{1/6}} \implies \epsilon_* \leq \frac{k}{s},$$
        for some constant $k>0$. Now, being totally in the setup of Proposition~\ref{the nongaussian theorem}, using $s+\frac{1}{s}=s(1+o(1))$ since $s\gg 1$,
        \begin{align*}
            \left|\frac{\bbP((U,V) \in A^c)}{\bbP((Z_1,Z_2) \in A^c)}-1\right| &\leq C''\cdot\frac{\epsilon_*\left( \varphi(s-\epsilon_*)+\frac{\varphi(s)}{s} \right)}{\varphi(s)/2(s+1/s)} +\Delta\\
            &\leq \frac{C''}{\frac{1}{2}(1+o(1))}\frac{\epsilon_*\left(e^{\epsilon_*s}\varphi(s)+\frac{\varphi(s)}{s} \right)}{\varphi(s)/s}+\Delta\\
            &\leq O\left(\epsilon_*s + \frac{1}{\sqrt{d}}\right) \quad(\text{using $\epsilon_*\leq \frac{k}{s}$ eventually})\\
            &= O\left(\frac{s^3}{\sqrt{d}}\right).
        \end{align*}
    \end{proof}
\end{lemma}

\begin{lemma}\label{inverse wij count moments}
Let $B_n \sim \operatorname{Binomial}(n,s_n)$ be a sequence of random variables and $ns_n\longrightarrow \infty$. Then, for any fixed $r\geq 1$,
$$\|B_n\|_{L_r}=O(ns_n), \qquad \left\|\frac{1}{B_n}\right\|_{L_r} = O((ns_n)^{-1}).$$
    \begin{proof}
        
        Start by writing $\|B_n\|_{L_r} \leq ns_n+\left\|\sum_{j=1}^n (Y_j-s_n)\right\|_{L_r}$, where $Y_j$ are i.i.d. $\operatorname{Bernoulli}(s_n)$ random variables.
        \begin{itemize}
            \item If $1\leq r\leq 2$, we have
            $$\left\|\sum_{j=1}^n (Y_j-s_n)\right\|_{L_r} \leq \left\|\sum_{j=1}^n (Y_j-s_n)\right\|_{L_2} = \sqrt{\operatorname{Var}(B_n)}=\sqrt{ns_n(1-s_n)}\leq \sqrt{ns_n}.$$

            \item If $r\geq 2$, we have by Rosenthal's inequality for sum of centered, independent $Y_j$'s,
            $$\bbE\left|\sum_{j=1}^n (Y_j-s_n)\right|^r \leq C_r\left( \sum_{j=1}^n \bbE|Y_j-s_n|^r + \left(\sum_{j=1}^n \operatorname{Var}(Y_j)\right)^{r/2} \right).$$
            Note that, for $r\geq 2$
            $$\bbE|Y_j-s_n|^r = s_n(1-s_n)^r + (1-s_n)s_n^r\leq 2s_n,$$
            so we have
            $$\bbE\left|\sum_{j=1}^n (Y_j-s_n)\right|^r \leq C_r(2ns_n + (ns_n)^{r/2}).$$
        \end{itemize}
        Combining we get for any $r\geq 1$,
        $$\|B_n\|_{L_r} \leq ns_n + (C_r \vee 
        1)(2ns_n + (ns_n)^{r/2}) = O(ns_n).$$

        On the other hand, using Chernoff's inequality for Binomials,
        $$\bbP\left(B_n \leq \frac{1}{2}\bbE(B_n)\right) \leq \exp\left(\frac{-\bbE(B_n)}{8} \right).$$

        So we get
        \begin{align*}
            \bbE\left( \frac{1}{B_n^r} \right) &= \bbE\left( \frac{1}{B_n^r}\indic\left\{ B_n> \frac{\bbE B_n}{2}\right\}\right) + \bbE\left( \frac{1}{B_n^r}\indic\left\{ B_n\leq \frac{\bbE B_n}{2}\right\}\right)\\
            &\leq \left( \frac{2}{\bbE B_n} \right)^r + \bbP\left( B_n \leq \frac{\bbE B_n}{2}\right)\\
            &\leq \left( \frac{2}{\bbE B_n} \right)^r + \exp\left(\frac{-\bbE B_n}{8} \right).
        \end{align*}
        Since $\bbE B_n$ diverges to infinity, the term $\left( \frac{2}{\bbE B_n} \right)^r$ is of the leading order in the above sum. Hence,
        $$\bbE\left( \frac{1}{B_n^r} \right) = O\left( \frac{1}{\bbE B_n^r}\right) \implies \left\| \frac{1}{B_n}\right\|_{L_r} = {O}((ns_n)^{-1}).$$
    \end{proof}
\end{lemma}

\begin{proposition}\label{orthogonal matrix linalg fact}
    Let $\mathcal{O}\left(m\right)$ be the set of all real-valued $m\times m$ orthogonal matrices, and define the set $\mathcal M$ of matrices as 
    $$\mathcal M:= \left\{\begin{bmatrix}
        \bm Q_1 & \bm 0 \\ \bm 0 & \bm Q_2
    \end{bmatrix}: \bm Q_1, \bm Q_2 \in \mathcal{O}\left(\frac{d}{2}\right)\right\}\subset \mathcal{O}(d).$$
    If $\bm R\in \bbR^{d\times d}$ is such that $\bm R \bm M= \bm M \bm R$ for all $\bm M \in \mathcal{M}$, then
    $$\|\bm R\|_{\mathrm{op}} \leq \frac{2\sum_{j=1}^d |\bm R_{jj}|}{d}.$$
    \begin{proof}
        Start by writing 
        $$\bm R = \begin{bmatrix}
            \bm A & \bm B \\ \bm C & \bm D
        \end{bmatrix},$$
        where $\bm A, \bm B, \bm C, \bm D \in \bbR^{\frac{d}{2}\times \frac{d}{2}}$. Since $\bm R$ commutes with $\bm M$ for all $\bm M\in \mathcal{M}$, we get that
        \begin{equation}
            \begin{split}
                \bm Q_1 \bm A = \bm A \bm Q_1, \qquad \forall \bm Q_1 \in \mathcal{O}\left(\frac{d}{2}\right),\\
            \bm Q_1 \bm B = \bm B \bm Q_2,\qquad \forall \bm Q_1, \bm Q_2 \in \mathcal{O}\left(\frac{d}{2}\right),\\
            \bm Q_2 \bm C = \bm C \bm Q_1, \qquad \forall \bm Q_1, \bm Q_2 \in \mathcal{O}\left(\frac{d}{2}\right),\\
            \bm Q_2 \bm D = \bm D \bm Q_2 \qquad \forall \bm Q_2 \in \mathcal{O}\left(\frac{d}{2}\right).
            \end{split}
        \end{equation}
        It follows immediately from the above equation that $\bm B = \bm C = \bm 0$, and $\bm A = a\bm I_{d/2}, \bm D = b\bm I_{d/2}$ for some scalars $a,b\in \bbR$. Thus, we get that $\bm R$ is diagonal and as a consequence
        $$\|\bm R\|_{\mathrm{op}} = \max\{|a|, |b|\} \leq |a| + |b| = \frac{2\sum_{j=1}^d |\bm R_{jj}|}{d}.$$
    \end{proof}
\end{proposition}

\begin{proposition}\label{Neumann series}
    Let $(M_n)_{n\in\bbN}$ and $(B_n)_{n\in \bbN}$ be two sequences of random matrices such that
    $$\|M_n - B_n\|_{op} = O_\bbP(a_n),$$
    where $a_n \longrightarrow 0$, as $n\longrightarrow \infty$, and $\sup_n \|B_n^{-1}\|_{op} \leq C < \infty$. Then, $M_n$ is invertible with probability tending to $1$ and we have 
    $$\|M_n^{-1} - B_n^{-1}\|_{op} = O_\bbP(a_n).$$
    \begin{proof}
        Let $E_n:=M_n-B_n$ and $F_n:=B_n^{-1}E_n$. Since
$\|F_n\|_{op}\le \|B_n^{-1}\|_{op}\|E_n\|_{op}\le C\|E_n\|_{op}$
and $\|E_n\|_{op}=O_{\mathbb P}(a_n)$, we have $\|F_n\|_{op}=O_{\mathbb P}(a_n)$.
Hence $\mathbb P(\|F_n\|_{op}<1/2)\to1$. On this high probability event, $(I+F_n)$ is invertible and so
\[
M_n=B_n(I+F_n)\quad\Rightarrow\quad M_n^{-1}=(I+F_n)^{-1}B_n^{-1}.
\]

From $(I+F_n)(I+F_n)^{-1}=I$, expand and rearrange to obtain
\[
(I+F_n)^{-1}-I =  -F_n(I+F_n)^{-1}.
\]
Taking operator norms and using submultiplicativity gives
\[
\|(I+F_n)^{-1}-I\|_{op}
\le \|(I+F_n)^{-1}\|_{op}\;\|F_n\|_{op}.
\]
When $\|F_n\|_{op}<1$, the Neumann series
$(I+F_n)^{-1}=\sum_{k=0}^\infty (-F_n)^k$ converges and hence
\[
\|(I+F_n)^{-1}\|_{op} \le \sum_{k=0}^\infty \|F_n\|_{op}^{\,k}
= \frac{1}{1-\|F_n\|_{op}}.
\]
Therefore, on $\{\|F_n\|_{op}\le \tfrac12\}$,
\[
\|(I+F_n)^{-1}-I\|_{op}
\le \frac{\|F_n\|_{op}}{1-\|F_n\|_{op}}
\le 2\,\|F_n\|_{op}.
\]

Finally,
\[
\|M_n^{-1}-B_n^{-1}\|_{op}
=\|(I+F_n)^{-1}-I\|_{op}\;\|B_n^{-1}\|_{op}
\le 2C\,\|F_n\|_{op}
\le 2C^2\,\|E_n\|_{op}.
\]
Since $\|E_n\|_{op}=O_{\mathbb P}(a_n)$, it follows that
$\|M_n^{-1}-B_n^{-1}\|_{op}=O_{\mathbb P}(a_n)$.
    \end{proof}
\end{proposition}

\section{Real data experiment details}\label{experiment appdx}

\subsection{Real-graph datasets and experimental pipeline}

We evaluate on three real graphs: \textbf{OGBN-Products}, \textbf{OGBN-MAG (paper)} and \textbf{PyG-Reddit}. The pipeline is the same across datasets:
\begin{enumerate}
\item \textbf{Dot-product embedding fit (graph alignment).} We first preprocess existing node covariates $\bm v_i$ so that dot-products of learned covariates $\bm x_i = f(\bm v_i)$ explain the observed adjacency as well as possible, and treat the learned covariates as latent covariates $\bm x_i$.
\item \textbf{Controlled regression task.} From the learned latents, we form noisy covariates and responses
\[
\bm z_i \;=\; \bm x_i + \bm \eta_i, \qquad y_i \;=\; \bm x_i^\top \bm \beta + \bm \varepsilon_i,
\]
and sweep $\sigma_\eta^2$. This keeps the real graph structure while controlling signal and noise.
\item \textbf{ER edge augmentation (varying neighbourhood size / adding non-geometric edges).} To probe robustness to additional non-geometric structure and to vary effective neighbourhood sizes, we further augment the observed graph by adding Erd\H{o}s--R\'enyi random edges at controlled rates (details below).
\end{enumerate}

\subsection{Embedding model and training (Stage 1)}

\paragraph{Model.} A 3-layer MLP encoder maps raw node features $\bm v_i \in \bbR^{d_{\mathrm{in}}}$ to an embedding $\bm x_i \in \mathbb{R}^{d_{\mathrm{out}}}$:
$$\text{Linear($d_{\mathrm{in}},256$) $\to$ ReLU $\to$ Linear($256,256$) $\to$ ReLU $\to$ Linear($256,d_{\mathrm{out}}$).}$$
$d_{\mathrm{in}}, d_{\mathrm{out}}$ for different datasets are listed in Table~\ref{data-params}. Edges are scored by a dot-product decoder with a learned scalar bias: $s_{ij}= \bm x_i^{\top}\bm x_j + b$.

\paragraph{Objective.} Binary cross-entropy link prediction loss using observed edges with 1:1 negative sampling (uniform random node pairs).

\paragraph{Optimization / schedule.} Adam (lr $10^{-3}$, weight decay $10^{-5}$), trained for 5 epochs with positive-edge minibatches of 250K edges.

\begin{table}[tb]
\caption{Feature and Embedding dimensions of node covariates across datasets}
\label{data-params}
\vskip 0.15in
\begin{center}
\begin{small}
\begin{sc}
\begin{tabular}{lccccc}
\toprule
Data set & $d_{\mathrm{in}}$ & $d_{\mathrm{out}}$ & Avg. degree before & Avg. degree after & Degree cutoff\\
\midrule
OGBN-Products    & 100 & 100 & 50 & 99 & 25\\
OGBN-MAG (paper) & 128  & 32 & 15 & 33 & 20\\
PyG-Reddit    & 602  & 602 & 492 & 958 & 200\\
\bottomrule
\end{tabular}
\end{sc}
\end{small}
\end{center}
\vskip -0.1in
\end{table}

\subsection{ER augmentation on real graphs}

The numbers of ER edges we add in the graph are approximately the same as the number of edges existing in the corresponding dataset. Concretely, we set an augmentation rate $p_n$, we sample random node pairs, each with success probability $p_n$ and add them as edges, then coalesce duplicates. Average degrees before and after adding ER edges are in Table~\ref{data-params}. 

\subsection{Predictors, training protocol, and early stopping}

\textbf{Evaluation protocol.} For each seed we sample evaluation nodes uniformly without replacement; all remaining nodes are used for training. Results are averaged over 10 seeds. 

\textbf{OLS.} Regress $y$ on $\bm z$, ignoring the graph.

\textbf{GCN baseline.} A full-batch mean-aggregation GCN regressor implemented with chunked edge processing:
\begin{itemize}
\item 2 layers, hidden width 64, dropout 0.1, ReLU activations, linear readout.
\item AdamW (lr $5 \times 10^{-3}$, weight decay $5 \times 10^{-5}$).
\item Mixed precision enabled on GPU; neighbour aggregation streamed in edge chunks.
\end{itemize}

\textbf{GAT (trained comparator).} A residual GAT regressor designed for full-batch training on large augmented graphs:
\begin{itemize}
\item Each attention layer uses scaled dot-product attention with learned query/key/value projections: scores are computed as a dot-product between queries and keys and then softmax-normalized over neighbours; messages use the value projections.
\item Input projection to width $H$, then 2 attention layers with residual connections and LayerNorm, followed by an MLP readout (Linear($H,H$)--ReLU--Dropout--Linear($H,1$)).
\item Heads = 1 and hidden-per-head = 64 (so $H = 64$); dropout 0.1.
\item AdamW (lr $3 \times 10^{-3}$, weight decay $10^{-5}$), gradient clipping; attention computed in streamed edge chunks, with intra-layer gradient checkpointing to control activation memory.
\end{itemize}

\textbf{Early stopping.} For both GCN and GAT we use a plateau-based criterion on training MSE: after a minimum of 10 epochs, stop if relative improvement is below 0.5\% for 5 consecutive epochs, with a hard cap at 100 epochs.

\subsection{Proxy-regression on real graphs: high-degree restriction}

A key practical issue on real graphs is presence of lower degree nodes. Our denoising relies on averaging over screened neighbourhoods and is therefore most stable for nodes with sufficiently large neighbourhoods. Accordingly:

\textbf{High-degree-only estimation.} We estimate the regression coefficient according to Algorithm~\ref{alg:attn-proxy-regression} \textbf{only on a high-degree subset} (degree computed on the base, unaugmented graph). The degree cutoffs of each dataset are in Table~\ref{data-params}.

\textbf{Hybrid predictions.} The ``proxy-regression'' curve reported on the full evaluation set corresponds to a hybrid predictor:
\begin{itemize}
\item use predictions according to Algorithm~\ref{alg:attn-proxy-prediction} on high-degree nodes;
\item use GCN predictions on low-degree nodes.
\end{itemize}
Thus, full-population MSE for ``proxy-regression'' should be interpreted as ``proxy-regression on high-degree nodes + GCN elsewhere.'' To isolate the denoising effect, we additionally report MSE restricted to high-degree evaluation nodes, where ``proxy-regression'' corresponds to the genuine Algorithm~\ref{alg:attn-proxy-prediction} prediction without fallback.

\subsection{Compute}

All real-graph experiments were run on an H200 GPU cluster using NVIDIA H200 GPUs (one GPU per job), with SLURM-managed resources (64 CPU cores and 50 GB RAM).

\end{document}